\documentclass{article}

\usepackage[preprint]{neurips_2026}

\usepackage[utf8]{inputenc}
\usepackage[T1]{fontenc}
\usepackage{hyperref}
\usepackage{url}
\usepackage{booktabs}
\usepackage{amsfonts}
\usepackage{microtype}
\usepackage{xcolor}
\usepackage{graphicx}
\usepackage{subcaption}
\usepackage{placeins}
\graphicspath{{figures/}}
\usepackage{amsmath,amssymb,amsthm}
\usepackage[ruled,vlined]{algorithm2e}

\newcommand{\calD}{\mathcal{D}}
\newcommand{\calX}{\mathcal{X}}
\newcommand{\calW}{\mathcal{W}}
\newcommand{\calL}{\mathcal{L}}
\newcommand{\calO}{\mathcal{O}}
\newcommand{\calF}{\mathcal{F}}
\newcommand{\bbe}[2][]{\mathbb{E}_{#1}\left[ #2 \right]}
\DeclareMathOperator{\Var}{Var}
\newcommand{\R}{\mathbb{R}}
\newcommand{\E}{\mathbb{E}}
\newcommand{\RNE}{\mathcal{R}}
\newcommand{\Vcost}{V^{\mathrm{cost}}}
\DeclareMathOperator*{\argmin}{arg\,min}
\newtheorem{theorem}{Theorem}[section]
\newtheorem{lemma}[theorem]{Lemma}
\newtheorem{proposition}[theorem]{Proposition}
\newtheorem{corollary}[theorem]{Corollary}

\newtheorem{assumption}{Assumption}
\newtheorem{remark}{Remark}[section]

\title{Decision-Focused On-Policy Learning for Contextual Linear Optimization with Partial Feedback}

\author{%
  Wyame Benslimane$^{*,1}$,\enspace
  Tinghan Ye$^{*,2}$,\enspace
  Pascal Van Hentenryck$^{2}$,\enspace
  Paul Grigas$^{1}$ \\[0.6em]
  $^{1}$Department of Industrial Engineering and Operations Research \\
  University of California, Berkeley \\
  \texttt{\{wyame.benslimane,\,pgrigas\}@berkeley.edu} \\[0.3em]
  $^{2}$H.\ Milton Stewart School of Industrial and Systems Engineering \\
  Georgia Institute of Technology \\
  \texttt{\{joe.ye,\,pvh\}@gatech.edu} \\[0.4em]
  $^{*}$Equal contribution.
}

\begin{document}

\maketitle

\begin{abstract}
Decision-focused learning (DFL) trains predictive models by optimizing downstream decision quality rather than standalone prediction accuracy.
For contextual linear optimization, most existing DFL methods assume offline data and full observations of the objective cost vector.
We develop an on-policy learning method for sequential contextual linear optimization under partial feedback, generalizing the standard bandit feedback setting.
Our method learns a stochastic predict-then-optimize policy that samples a cost-vector prediction from a conditional distribution and solves the resulting downstream linear optimization problem.
To update this distributional model, we introduce a two-component hybrid gradient estimator.
The first component is a score function estimator, which provides an unbiased but potentially high-variance policy gradient estimate.
The second is a decision-focused plug-in component that uses an auxiliary nuisance estimate of the latent cost vector to exploit the downstream optimization structure, becoming more informative as the estimate improves.
We prove an $\mathcal{O}(T^{-1/2})$ bound on the average squared policy-gradient norm, matching the standard non-convex SGD rate.
Experiments on top-$k$ selection, shortest path, combinatorial pricing, and a real-data energy-scheduling benchmark show that, under bandit feedback, the hybrid gradient approach achieves lower mean cumulative regret than the contextual-bandit baselines on all four benchmarks and than linear Thompson sampling on three, and that it also works with richer conditional generative cost models.
Code is available at \url{https://github.com/Joeyetinghan/on-policy-bandit-dfl}.
\end{abstract}

\section{Introduction}

Many operational systems make sequential decisions by solving constrained optimization problems given observed contextual information, with applications spanning vehicle routing, dynamic pricing, power scheduling, product allocation, order fulfillment, and recommendation \citep{van2006online,powell2022designing,elmachtoub2022smart,li2010contextual,zhao2026diffusion,sadana2025survey,ye2026learning}. In these settings, the feasible set and operational constraints are known, while the relevant cost or reward parameters are uncertain at decision time and instead must be learned. When such systems are deployed repeatedly, the learner is not merely trained on a fixed historical dataset; its own decisions become part of the data-generating process, determining which outcomes are observed and hence what information is available for future learning. Furthermore, the collected data rarely include full information about problem parameters: a selected route reveals only the realized travel time, a posted price reveals only demand at that price, a recommended subset reveals only response on the chosen items. The learner thus repeatedly sees a context, chooses a feasible decision from the current policy, and observes only the feedback generated by that decision.
This paper studies sequential contextual linear optimization (CLO), a widely applicable class of constrained optimization problems, in this on-policy decision-making environment with partial feedback.

The combination of constrained optimization and online decision-making with partial feedback is not yet captured systematically by any one existing framework.
Decision-focused learning (DFL) trains predictions through the decisions they induce \citep{Mandi_2024, elmachtoub2022smart}, but most existing methods assume batch data and full observations of the latent cost vector in CLO. Contextual bandits and online learning study sequential learning from partial feedback, including large and structured action spaces \citep{li2010contextual, abbasi2011improved, krishnamurthy2016contextual, zhu2022contextual}, but typically model rewards directly over context-action pairs rather than preserving and strongly exploiting the downstream linear optimization model. Given the surge of development in powerful DFL methods for offline problems with full data~\citep{Mandi_2024}, a natural question emerges:  can DFL methods be adapted to on-policy, sequential decision-making with partial feedback?

We answer affirmatively by introducing a new class of on-policy stochastic gradient estimators that enables strong plug-in/predict-then-optimize policies. We learn a structured conditional distribution over cost vectors and induce decisions through linear optimization, rather than a direct context-to-decision map.
We update the policy parameters by mixing two gradient signals with complementary trade-offs.
A score function estimator uses the on-policy realized cost directly, making it unbiased for the induced policy objective and naturally tied to exploration through sampling. However, the score function approach treats the downstream optimizer as a black-box action generator and therefore does not strongly use the structure of the feasible set or linear objective in its gradient. We therefore propose a decision-focused plug-in gradient estimator that uses this structure: it pairs the known feasible set with a nuisance estimate of the conditional mean cost vector, and extends the plug-in approach that \citet{hu2024contextual, hu2025contextual} developed for offline, off-policy data to on-policy learning. Once the nuisance estimate is plugged-in to a cost approximation, the resulting problem is equivalent to a full information contextual stochastic linear optimization problem and can be tackled with all of the machinery, such as surrogate loss functions, developed in recent years~\citep{Mandi_2024, sadana2025survey}. We bound the bias of this nuisance estimator approach and show its advantages grow as data accumulates.

\paragraph{Contributions.}
\textit{First}, we develop an on-policy decision-focused learning method for online contextual linear optimization under a generic partial-feedback model that subsumes bandit, semi-bandit, and full-information settings. The policy update is a hybrid gradient combining an unbiased score function term with a decision-focused plug-in term from a learned auxiliary nuisance cost-vector model. We extend the decision-focused plug-in term to accommodate a broad class of DFL surrogates (e.g., the Smart Predict-then-Optimize surrogate (SPO+), Perturbation Gradient (PG), and Perturbed Fenchel--Young loss (PFYL)). \textit{Second}, we exploit a universal-minimizer property of the nuisance least-squares objective to analyze the nuisance update as standalone online convex optimization against a fixed comparator, independently of the policy iterates. We then convert the resulting regret into a bias bound on the plug-in gradient via a one-step policy descent inequality, yielding $\calO(T^{-1/2})$ stationarity. The same argument extends to the surrogate variant. \textit{Third}, numerical experiments on top-$k$ selection, shortest path, combinatorial pricing, and a real-data energy-scheduling benchmark show, under bandit feedback, lower mean cumulative regret than contextual-bandit baselines on all benchmarks and than linear Thompson sampling on three of four.

\subsection{Literature review}

\paragraph{Decision-focused and predict-then-optimize learning.}
Decision-focused learning trains predictive models through the downstream optimization problems that consume their outputs \citep{Mandi_2024}. For linear objectives, Smart Predict-then-Optimize (SPO) and SPO$+$ give a task-aware alternative to prediction error \citep{elmachtoub2022smart}. Later methods develop differentiable perturbation-based optimizers and directional-gradient estimators \citep{berthet2020learning, gupta2024decision}. Score function estimators extend DFL to cases where direct differentiation through the task loss or optimization layer is difficult \citep{silvestri2026score}, and distributional DFL methods model uncertainty in optimization parameters using generative predictors such as conditional normalizing flows and diffusion models \citep{wang2025gendfl,zhao2026diffusion}. Our proposed method inherits the use of distributional predictions and score function gradients, but the observation model differs: the learning signal is online, with policy-generated partial feedback rather than observed cost vectors. Sequential DFL methods include online predict-then-optimize \citep{liu2022online} and multi-stage Predict+Optimize with sequential parameter revelation and interleaved decisions \citep{hu2024multi}; the closest theoretical analog is online DFL for linear optimization over polytopes \citep{capitaine2025online}. All three share our sequential viewpoint but assume full-information feedback rather than the partial feedback induced by the implemented decision.

\paragraph{Contextual optimization with partial feedback.}
The closest contextual optimization precursor is \citet{hu2024contextual}, who study offline contextual linear optimization under partial feedback (pure bandit and semi-bandit). Their work develops, from logged partial-feedback data, induced empirical risk minimization and surrogate approaches using least-squares nuisance estimation. Other work studies partial information from different angles. \citet{ye2025contextual} use distributional predictions for mixed-integer and nonlinear contextual stochastic optimization. \citet{bennouna2025data,bennouna2026informativeness} study data informativeness for linear optimization, characterizing when partial observations contain enough information to identify an optimal decision. \citet{bouvier2025primal} learn randomized combinatorial-optimization policies offline, from full observations of the uncertainty, through a primal-dual Fenchel-Young surrogate. This paper studies the online counterpart. At each iteration, the current policy and optimization oracle choose a feasible decision. Feedback is observed only for that decision, from aggregate loss in the pure bandit case to selected component observations in semi-bandit variants. The same stream of partial feedback is used both to estimate the cost model and to update the policy.

\paragraph{Contextual bandits and policy optimization.}
Contextual bandits provide the standard framework for sequential learning from partial feedback \citep{li2010contextual, abbasi2011improved, bubeck2012regretanalysisstochasticnonstochastic}. Work on large or structured action spaces includes contextual semi-bandits and oracle-based methods \citep{krishnamurthy2016contextual, zhu2022contextual, zierahn2023nonstochastic}. Structured prediction with bandit feedback is also related: a learner outputs a structured object and observes only its task loss \citep{sokolov2016stochastic, kreutzer2017bandit}. Our score function estimator follows the likelihood-ratio principle of REINFORCE and policy-gradient methods \citep{williams1992simple, sutton2000policy, konda2000actor}. Unlike standard contextual-bandit models that parameterize rewards over context-action pairs, we keep the linear-cost representation explicit, sampling a cost-vector prediction and solving a known optimization problem, while the auxiliary estimator learns conditional cost structure from partial observations.

\paragraph{Online convex optimization.}
Online convex optimization (OCO) and its bandit variants study sequential convex or linear losses with full-information or bandit feedback \citep{zinkevich2003online, flaxman2005online, dani2008stochastic, abernethy2008competing, ito2019oracle}, bounding regret against the best fixed decision in hindsight. Our setting is one step removed: we update parameters of a stochastic policy that produces a decision via the optimization oracle, and because the induced policy objective is nonsmooth and non-convex, our analysis targets first-order stationarity rather than convex regret.

\section{Decision-focused on-policy learning with partial feedback}
\label{sec:bandit_clo}
We consider an online contextual stochastic linear optimization setting. At each time $t = 1, 2, \ldots$, the decision-maker observes a context $x_t \in \calX \subseteq \R^{d_x}$ (context dimension $d_x$) and then makes a decision $w_t \in \calW$, where $\calW \subset \R^{d_w}$ is a known compact feasible set, fixed across contexts. The context $x_t$ is drawn from an unknown marginal distribution $\calD_x$. Furthermore, the latent cost vector $c_t$, drawn from an unknown conditional distribution $p^{\star}(\cdot \mid x_t)$, is never observed directly: once a decision is made, the learner sees only a partial-feedback vector $v_t \in \R^{d_v}$ (with feedback dimension $d_v$) whose form depends on the implemented decision; for instance, a selected route reveals the realized travel time of the chosen path, not the time on every edge. We denote the unknown joint distribution of $(x_t,c_t) \in \calX\times \R^{d_w}$ by $\calD$ (assumed to be i.i.d.\ over time).

\begin{assumption}[Feedback structure]\label{ass:feedback}
At each time $t = 1, 2, \ldots$, given that $(x_t,c_t) \sim \calD$, we have the following:
\begin{enumerate}
    \item For every possible decision $w_t \in \calW$, there exists a known observation matrix $H(w_t) \in \R^{d_v \times d_w}$ such that the observed feedback $v_t \in \R^{d_v}$ satisfies $v_t = H(w_t)c_t$.
    \item For every possible decision $w_t \in \calW$, there exists a known linkage vector $e(w_t) \in \R^{d_v}$ such that $H(w_t)^{\top}e(w_t) = w_t$; consequently, $y_t = c_t^{\top} w_t = e(w_t)^{\top} v_t$ is always recoverable.
\end{enumerate}
\end{assumption}
In words, $H(w_t)$ selects which linear measurements of the latent cost vector are revealed by decision $w_t$, and $e(w_t)$ is a known weight vector whose inner product with the observation always returns the realized scalar cost: $e(w_t)^\top v_t = c_t^\top w_t = y_t$, regardless of the feedback regime. This structure captures the three feedback regimes considered in this paper.
\begin{itemize}
    \item \textbf{Full feedback.} Set $H(w) = I$ and $e(w) = w$. The feedback coincides with the cost vector, $v = c$, and the linkage identity $c^\top w = e(w)^\top v$ holds trivially.
    \item \textbf{Bandit feedback.} Set $H(w)=w^{\top}$ and $e(w)=1$. Then $v=c^{\top}w$ is the scalar realized cost, and the linkage identity again holds trivially.
    \item \textbf{Semi-bandit feedback.} Let $\calW \subseteq \{0,1\}^{d_w}$. Set $H(w)=\mathrm{diag}(w)$ and $e(w)=\mathbf{1}$. Then the feedback $v = \mathrm{diag}(w)c$ records the cost of each selected coordinate. The linkage condition holds since $H(w)^\top e(w) = \mathrm{diag}(w)\mathbf{1} = w$, and the identity $c^\top w = e(w)^\top v$ follows immediately.
\end{itemize}

In contextual linear optimization, a key assumption is that we can readily solve linear optimization problems over the compact decision set $\calW$; let $w^\star(\hat c) \in \argmin_{w \in \calW} \hat c^\top w$ denote a particular linear optimization oracle for $\calW$. We consider a class of distributional plug-in/predict-then-optimize policies, whereby $p_\theta(\cdot \mid x)$ denotes a parametric
conditional model for cost vector predictions, parameterized by $\theta \in \R^{d_\theta}$. Any $\theta \in \R^{d_\theta}$ induces a policy $\pi_\theta : \calX \to \Delta(\calW)$, where $\Delta(\calW)$ is the set of all distributions supported on $\calW$. That is, given the context $x_t \in \calX$ at time $t$, we first sample a predicted cost vector $\hat c_t \sim p_\theta(\cdot \mid x_t)$ and then make the decision $w_t = w^\star(\hat c_t)$. Given $\theta \in \R^{d_\theta}$, let us also define the mean prediction function $\bar \mu_\theta(x) := \bbe[\hat c \sim p_\theta(\cdot\mid x)]{\hat c}$ and the mean decision function $w^\star_\theta(x) := \bbe[\hat c \sim p_\theta(\cdot\mid x)]{w^\star(\hat c)}$. Furthermore, for a given parameter $\theta$, let $\calD_\theta$ denote the distribution over $(x,c,w)$ induced by $(x, c) \sim \calD$ and $w \sim \pi_\theta(\cdot\mid x)$. An important observation is that the conditional independence property $w \perp c \mid x$ holds for $(x,c,w) \sim \calD_\theta$.

Let $f^\star(x) := \bbe[{c\sim p^{\star}(\cdot\mid x)}]{c}$ denote the ground truth regression function.
We define the expected cost of policy parameterized by $\theta \in \R^{d_\theta}$ by
\begin{equation*}\label{eq:policy-decomp}
  \Vcost(\theta) := \bbe[{(x, c,w)\sim\calD_\theta}]{c^{\top} w} = \bbe[x \sim \calD_x]{ \bbe[c \sim p^{\star}(\cdot \mid x)]{ c^\top w^\star_\theta(x)} }
  = \bbe[{x \sim \calD_x}]{f^\star(x)^{\top}w^\star_\theta(x) } ,
\end{equation*}
where the latter equalities use $w \perp c \mid x$. It is well known and apparent from the above that, under full information, the decision maker would ideally make decision $w^\star(f^\star(x))$. Therefore, any policy such that $w^\star_\theta(x) = w^\star(f^\star(x))$ almost surely over $x$ is an optimal full information policy, for example a degenerate distribution with $\bar \mu_\theta(x) = f^\star(x)$. The distributional class serves two purposes: a usable gradient and exploration. With a point prediction, the decision $w^\star(\bar\mu_\theta(x))$ is piecewise constant in $\theta$, so its gradient is zero almost everywhere. Sampling $\hat c$ from $p_\theta(\cdot\mid x)$ smooths this map, and the realized cost $y_t$ alone then gives an unbiased gradient. Offline surrogates such as SPO+ and PFYL avoid the zero gradient by using the observed cost vector. Under partial feedback this vector is not observed, so the plug-in route below first estimates it with a nuisance model. Sampling in cost space also explores through the same oracle.

\subsection{Score function gradient estimation}
The smoothed map $w^\star_\theta(x)=\E_{\hat c\sim p_\theta(\cdot\mid x)}[w^\star(\hat c)]$ is a perturbed optimizer in the sense of \citet{berthet2020learning}, used here as a randomized policy whose sampled decision is implemented each round. Appendix~\ref{apdx:perturbed_optimizer} gives its regularity properties. Since $\theta$ enters $\Vcost(\theta)$ only through $p_\theta(\cdot \mid x)$, not through the cost or the optimizer $w^\star(\cdot)$, its gradient can be estimated from bandit feedback without differentiating through $w^\star$.

Full details of the score function gradient estimator derivation are given in Appendix~\ref{apdx:conv_analysis}; in short, given current parameters $\theta_t$ at time $t$, under well-behaved parametric distributions, differentiating under the expectation yields the one-sample score function gradient estimator\footnote{``Score function'' refers to the likelihood-ratio gradient $\nabla_\theta\log p_\theta$ of REINFORCE \citep{williams1992simple}. \citet{hu2024contextual} use the term differently, for pseudo-cost vectors.}
\begin{equation}\label{eq:score-grad}
  g_t^{\mathrm{score}}
  := y_t\,\nabla_\theta \log p_{\theta_t}(\hat c_t \mid x_t), \quad ~\text{where}~
  \quad \hat c_t \sim p_{\theta_t}(\cdot \mid x_t), w_t = w^\star(\hat c_t), y_t = c_t^\top w_t,
\end{equation}
which satisfies $\E[g_t^{\mathrm{score}}] = \nabla_\theta \Vcost(\theta_t)$. Note that $g_t^{\mathrm{score}}$ depends on the environment only through the scalar bandit feedback $y_t$, with no explicit reference to $f^\star$.
More generally, any terms preserving unbiasedness of the gradient estimator of $\nabla_\theta \Vcost(\theta)$ may replace $y_t$ in~\eqref{eq:score-grad}; for instance, subtracting a context-dependent baseline $b(x_t)$ from $y_t$ preserves unbiasedness while typically reducing variance.

\subsection{Decision-focused gradient estimation}
The score function gradient estimator, while unbiased, is structure-agnostic: it treats $y_t$ as a generic scalar cost and only weakly exploits the linear optimization structure. Decision-focused methods in the full-information setting instead exploit the direct dependence of $w^\star_\theta(\cdot)$ on $\theta$. Furthermore, particularly in the case of contextual linear optimization, the decomposition of $\Vcost$ above reveals that $\Vcost(\theta)$ depends on the cost distribution \emph{only} through its conditional mean $f^\star(x) = \E[c \mid x]$. At the same time, the observed feedback $v_t$ carries information about $f^\star$ that the score estimator never directly exploits. These observations suggest a complementary route: estimate $f^\star$ from accumulated feedback, plug the resulting nuisance estimate into $\Vcost(\cdot)$, and apply standard full-information decision-focused machinery such as direct gradients and surrogate losses.
This decision-focused, plug-in perspective trades the unbiasedness of~\eqref{eq:score-grad} for exploitation of the downstream optimization problem and a potentially substantial reduction in variance.

\paragraph{Nuisance estimation via least-squares feedback fitting.}
We introduce an auxiliary cost-vector regression model $f_\phi : \calX \to \R^{d_w}$ (the \emph{nuisance} estimator), parameterized by $\phi \in \Phi \subset \R^{d_\phi}$, which approximates the conditional cost mean $f^\star(x) = \E[c\mid x]$ by minimizing the following least-squares objective based on the observed feedback $v = H(w)c$.
\begin{equation}\label{eq:nuisance-loss}
  \min_{\phi \in \Phi} \RNE_\theta(f_\phi),\qquad \RNE_\theta(f_\phi)
  := \bbe[{(x,c,w)\sim\calD_\theta}]{\|v - H(w)f_\phi (x)\|^2}
\end{equation}
\begin{remark}[Universal minimizer]\label{rem:universal_min}
    A key property of this objective is that $f^\star$ is a \emph{universal} minimizer of $\RNE_\theta$ for every $\theta \in \R^{d_\theta}$. Consequently, the nuisance estimation target does not shift with the policy, which is central to our analysis.
\end{remark}
The regret analysis of the nuisance model is obtained using standard online convex optimization results and universality of the minimizer. The nuisance update is the projected stochastic gradient descent (SGD) step $\phi_{t+1}=\Pi_\Phi(\phi_t-\eta_t\nabla_\phi\ell_t(\phi_t))$ on $\ell_t(\phi)=\|v_t-H(w_t)f_\phi(x_t)\|^2$, with $\Pi_\Phi$ the Euclidean projection onto $\Phi$. We defer the full analysis of the nuisance estimator to Appendix~\ref{apdx:nuisance_analysis}.
\begin{proposition}[Nuisance regret]\label{prop:nuisance-regret}
Suppose $\Phi$ is convex and compact with diameter $D_\Phi$, $\ell_t$ has $G_\RNE$-bounded gradients and $\nabla^2_\phi \RNE_\theta(f_\phi) \succeq \lambda I$ for every $\theta$. Under  realizability ($f_{\phi^\star}=f^\star$ for some $\phi^\star\in\Phi$), for any sequence $\{\theta_t\}_{t=1}^T$, the average excess nuisance loss satisfies
\begin{enumerate}
  \item[\textup{(i)}] $\frac{1}{T}\sum_{t=1}^T \E\!\left[\RNE_{\theta_t}(f_{\phi_t}) - \RNE_{\theta_t}(f^\star)\right] \le \frac{D_\Phi G_\RNE}{\sqrt{T}}$ \quad with $\eta_t = D_\Phi / (G_\RNE\sqrt{T})$;
  \item[\textup{(ii)}] $\frac{1}{T}\sum_{t=1}^T \E\!\left[\RNE_{\theta_t}(f_{\phi_t}) - \RNE_{\theta_t}(f^\star)\right] \le \frac{G_\RNE^2(1+\log T)}{2\lambda T}$ \quad with $\eta_t = 1/(\lambda t)$.
\end{enumerate}
\end{proposition}
Three conditions drive this result. First, the target $f^\star$ is the same for every policy (Remark~\ref{rem:universal_min}), so the nuisance model tracks a fixed target while the data-generating policy changes. Second, realizability ($f^\star\in\calF:=\{f_\phi:\phi\in\Phi\}$, Assumption~\ref{ass:realizable}) puts this target in the nuisance class, so the plug-in gradient has no approximation error. Third, the
uniform strong-convexity condition $\nabla^2_\phi \RNE_\theta(f_\phi) \succeq \lambda I$ for every $\theta$ is an identifiability/exploration assumption on the policy-induced feedback; under pure bandit feedback it requires the sampled actions to span the cost directions relevant to the nuisance class.
This regret bound is the primary tool for controlling the bias of the plug-in gradient estimator introduced next. Given a nuisance estimate $f_\phi$, we define the plug-in policy objective, which is an approximation of $\Vcost(\theta)$, by
\begin{equation}\label{eq:plug-in}
  J(\phi,\theta)
  := \bbe[{x \sim \calD_x}]{f_\phi(x)^\top w^\star_\theta(x)} = \bbe[{x \sim \calD_x, \hat c \sim p_\theta(\cdot\mid x)}]{f_\phi(x)^\top w^\star(\hat{c})}.
\end{equation}

\paragraph{Plug-in gradient estimator.}
Under standard regularity assumptions (Appendix~\ref{apdx:perturbed_optimizer}), differentiating the plug-in objective yields
\begin{equation}\label{eq:plug-in-grad}
  \nabla_\theta J(\phi,\theta)
  = \E_{x \sim \calD_x}\!\left[\bigl(\nabla_\theta w^\star_\theta(x)\bigr)^{\!\top} f_\phi(x)\right],
  \quad \nabla_\theta w^\star_\theta(x) \in \R^{d_w\times d_\theta}.
\end{equation}

Since $\theta$ enters $J(\phi,\theta)$ solely through the sampling distribution, computing $\nabla_\theta J(\phi,\theta)$ reduces to differentiating $w^\star_\theta(x)$. For polyhedral $\calW$ and other cases, the mapping $\hat c \mapsto w^\star(\hat c)$ may be piecewise constant almost everywhere, so differentiating through it in the second expectation in~\eqref{eq:plug-in} may yield uninformative gradients, making it difficult to obtain a one-sample stochastic gradient estimator of $\nabla_\theta J(\phi,\theta)$. One approach to overcome this challenge is to again rely on a score function representation:
\begin{equation}\label{eq:plug-in-grad-stoch}
  g^{\mathrm{plug\text{-}in}}_t= \bigl(w_t^\top f_{\phi_t}(x_t)\bigr)\nabla_\theta \log p_{\theta_t}(\hat c_t \mid x_t), \quad ~\text{where}~
  \quad \hat c_t \sim p_{\theta_t}(\cdot \mid x_t), w_t = w^\star(\hat c_t).
\end{equation}
Using the above estimator resembles a traditional actor-critic method; however, again it only weakly exploits the optimization structure.

\paragraph{Surrogate plug-in gradient estimator.}
To better exploit the linear optimization structure, we build on the extensive literature of decision-focused {\em surrogate} methods that have been developed in the full information setting.
Namely, we replace $f_\phi(x)^\top w^\star(\hat c)$ in \eqref{eq:plug-in} by a differentiable proxy $h^{\mathrm{sur}}(f_\phi(x),\hat c)$ designed to provide informative gradients. Examples include the SPO+ loss~\citep{elmachtoub2022smart}, the Perturbed Fenchel--Young loss~\citep{berthet2020learning}, and the Perturbation Gradient loss~\citep{gupta2024decision}, among others. For any choice of surrogate function $h^{\mathrm{sur}}(f_\phi(x),\hat c)$, we define the surrogate plug-in objective
\begin{equation}\label{eq:Jsur}
  J^{\mathrm{sur}}(\phi,\theta):= \bbe[{x \sim \calD_x, \hat c \sim p_\theta(\cdot\mid x)}]{h^{\mathrm{sur}}\!\left(f_\phi(x),\,\hat c \right)}.
\end{equation}
Unlike the plug-in objective $J(\phi,\theta)$, which requires differentiating through $w^\star_\theta(x)$, we instead differentiate through the surrogate $h^{\mathrm{sur}}(f_\phi(x),\hat c)$, which is engineered to have informative gradients almost everywhere. Combined with a reparameterization $\hat c= \mu_\theta(x, \varepsilon)$ with $\varepsilon \sim \nu$ independent of $\theta$ (Assumption~\ref{ass:reparam}; for the Gaussian policy, $\mu_\theta(x,\varepsilon)=\bar\mu_\theta(x)+\sigma_{\mathrm G}\varepsilon$ with $\nu=\mathcal N(0,I)$), this yields the surrogate gradient estimator $g^{\mathrm{sur}}_t
  = \bigl(\nabla_\theta \mu_{\theta_t}(x_t,\varepsilon_t)\bigr)^\top
    \nabla_{\hat c} h^{\mathrm{sur}}\!\bigl(f_{\phi_t}(x_t),\hat c_t\bigr)$.
Conditioned on $(\phi_t,\theta_t)$, $g_t^{\mathrm{sur}}$ is unbiased for $\nabla_\theta J^{\mathrm{sur}}(\phi_t,\theta_t)$; bias relative to $\Vcost$ comes from the nuisance estimate and the surrogate choice. The full surrogate framework is in Appendix~\ref{apdx:surrogate}.

\subsection{Decision-focused hybrid policy gradient algorithm}
The score function gradient estimator~\eqref{eq:score-grad} and the plug-in objective~\eqref{eq:plug-in} (or its surrogate function variant \eqref{eq:Jsur}) provide two complementary gradient signals for $\theta$. Our algorithm exploits both: at each iteration, it updates the nuisance parameter $\phi$ by least-squares SGD step on~\eqref{eq:nuisance-loss}, and the parameter $\theta$ by an SGD step combining the two signals.
This complementarity motivates combining them through a convex mixture. For a mixing parameter $\alpha \in [0,1]$, we define the hybrid gradient estimator
\begin{equation}\label{eq:hybrid-grad}
  g_t^\alpha
  = \alpha\,g_t^{\mathrm{score}}
  + (1-\alpha)\,g^{\mathrm{plug\text{-}in}}_t
\end{equation}
The surrogate variant of the hybrid estimator is obtained by substituting $g^{\mathrm{sur}}_t$ for $g^{\mathrm{plug\text{-}in}}_t$, yielding $g_t^{\alpha,\mathrm{sur}} := \alpha g_t^{\mathrm{score}} + (1-\alpha) g_t^{\mathrm{sur}}$. At $\alpha = 1$ the estimator reduces to the unbiased score function gradient, while at $\alpha = 0$ it relies entirely on the plug-in gradient.

Algorithm~\ref{alg:main} instantiates this construction as an online procedure covering both the plug-in and surrogate variants.
\begin{algorithm}[!ht]
\caption{Decision-focused hybrid policy gradient (DFHPG)}
\label{alg:main}
\KwIn{Initial parameters $\phi_1 \in \Phi$, $\theta_1 \in \R^{d_\theta}$; mixing parameter $\alpha \in [0,1]$; step sizes $\eta_t, \beta > 0$}
\For{$t = 1, 2, \ldots, T$}{
  Observe context $x_t$\;
  Sample $\hat c_t \sim p_{\theta_t}(\cdot \mid x_t)$ and execute
  $w_t \in \argmin_{w \in \calW} \hat c_t^\top w$\;
  Observe feedback $v_t$ and set $y_t = e(w_t)^\top v_t$\;
  $\phi_{t+1} = \Pi_\Phi\!\Bigl(\phi_{t} - \eta_t\,\nabla_{\phi}\ell_{t}(\phi_{t})\Bigr) , \quad \ell_{t}(\phi) = \|v_{t} - H(w_{t}) f_{\phi}(x_{t})\|^{2} $
    \tcp{\textbf{Nuisance update}}
  $\theta_{t+1} = \theta_t - \beta\,\tilde g_t$, where $\tilde g_t = g_t^\alpha$ for the plug-in variant and $\tilde g_t = g_t^{\alpha,\mathrm{sur}}$ for the surrogate variant\tcp{\textbf{Hybrid policy update}}
}
\end{algorithm}

\subsection{Convergence results}
We formalize the convergence guarantees of Algorithm~\ref{alg:main}. Since the objective is non-convex in $\theta$, we target standard non-convex stationarity criteria. The analysis hinges on a key structural property: the hybrid update is conditionally biased only through the nuisance error, which is in turn controlled by the regret of the nuisance update (Proposition~\ref{prop:nuisance-regret}). Under standard regularity formalized in Appendix~\ref{apdx:conv_analysis}, the plug-in objective $J(\phi,\theta)$ is $L^J_\theta$-smooth in $\theta$ with gradient norm bounded by $G^J$.
\begin{theorem}[Non-convex stationarity, plug-in variant]\label{thm:main-alpha}
Let $V^\star := \inf_\theta \Vcost(\theta) > -\infty$, $\Var[g_t^{\mathrm{score}}\mid\phi_t,\theta_t]\le\sigma_{\mathrm{score}}^2$, and $\Var[g_t^{\mathrm{plug\text{-}in}}\mid\phi_t,\theta_t]\le\sigma_{\mathrm{plug\text{-}in}}^2$. With step size $\beta = 1/(L^J_\theta\sqrt{T})$, Algorithm~\ref{alg:main} satisfies
\begin{enumerate}
  \item[\textup{(i)}] With $\eta_t=D_\Phi/(G_\RNE\sqrt{T})$:\quad
  \begin{equation}\label{eq:main-alpha-i}
\frac{1}{T}\sum_{t=1}^T\E\bigl[\|\nabla_\theta\Vcost(\theta_t)\|^2\bigr]
\;\le\;\frac{2L^J_\theta\Delta_1+(G^J)^2+\sigma_\alpha^2}{\sqrt{T}}
+ \calO\Bigl((1-\alpha)^2 \frac{D_\Phi G_\RNE}{\lambda \sqrt{T}}\Bigr),
\end{equation}
  \item[\textup{(ii)}] With $\eta_t = 1/(\lambda t)$:
  \begin{equation}\label{eq:main-alpha-ii}
\frac{1}{T}\sum_{t=1}^T\E\bigl[\|\nabla_\theta\Vcost(\theta_t)\|^2\bigr]
\;\le\;\frac{2L^J_\theta\Delta_1+(G^J)^2+\sigma_\alpha^2}{\sqrt{T}}
+ \calO\Bigl(\frac{(1-\alpha)^2(1+\log T)}{\lambda^2 T}\Bigr),
\end{equation}
\end{enumerate}

where $\Delta_1 := \Vcost(\theta_1)-V^\star$, and the variance $\sigma_\alpha^2 := (\alpha\sigma_{\mathrm{score}}+(1-\alpha)\sigma_{\mathrm{plug\text{-}in}})^2$.

\end{theorem}
In both cases, the optimization term is $\calO(T^{-1/2})$, matching the standard non-convex SGD rate \citep{ghadimi2013stochastic}. The constant stepsize in (i) yields a $\calO(T^{-1/2})$ bias term, while the adaptive schedule in (ii) exploits strong convexity to sharpen the bias to $\calO(\log T/T)$. Supporting lemmas and proofs are deferred to Appendices~\ref{apdx:nuisance_analysis}--\ref{apdx:surrogate}.
The variance $\sigma_\alpha^2$ interpolates continuously between two extremes. At $\alpha=1$ one recovers the unbiased score function gradient~\citep{sutton2018reinforcement}. At $\alpha=0$ one relies entirely on the plug-in gradient, which enjoys lower variance once $f_{\phi_t}$ is well-trained, at the cost of a nuisance-driven bias. The optimal $\alpha$ shifts toward zero as the nuisance estimation improves, suggesting that adaptive schedules $\alpha_t \to 0$ may improve empirical performance.

We now turn to the surrogate variant. Define $V^\alpha(\theta) := \alpha\,\Vcost(\theta) + (1-\alpha)\,J^\mathrm{sur}(\phi^\star,\theta)$. Under standard regularity formalized in Appendix~\ref{apdx:surrogate}, $V^\alpha$ is $L^\alpha_\theta$-smooth in $\theta$ with gradient norm bounded by a constant $\tilde G_\alpha$ (Appendix~\ref{apdx:surrogate}). We state the result under the dynamic schedule $\eta_t = 1/(\lambda t)$; an analogous bound to Theorem~\ref{thm:main-alpha}(i) holds for the constant step size for our surrogate model.

\begin{theorem}[Non-convex stationarity, surrogate variant]\label{thm:sur-main}
Let $V^{\alpha,\star} := \inf_\theta V^\alpha(\theta)$. Running Algorithm~\ref{alg:main} with $\beta = 1/(L^\alpha_\theta \sqrt{T})$ and $\eta_t = 1/(\lambda t)$ yields
\begin{equation}\label{eq:sur-main-bound}
\begin{aligned}
\frac{1}{T}\sum_{t=1}^T \E\bigl[\|\nabla V^\alpha(\theta_t)\|^2\bigr]
&\leq \frac{2 L^\alpha_\theta \Delta_1^\alpha
+ \tilde G_\alpha^2 + \tilde\sigma_\alpha^2}{\sqrt{T}} \\
&\quad + \calO\Bigl((1-\alpha)^2\tfrac{1 + \log T}{\lambda^2 T}\Bigr).
\end{aligned}
\end{equation}
with $\Delta_1^\alpha = V^\alpha(\theta_1) - V^{\alpha,\star} $, $\tilde\sigma_\alpha^2 := (\alpha \sigma_{\mathrm{score}} + (1-\alpha)\sigma_{\mathrm{sur}})^2$, $\sigma_{\mathrm{sur}}^2\ge\Var[g_t^{\mathrm{sur}}\mid\phi_t,\theta_t]$.
\end{theorem}

Theorem~\ref{thm:sur-main} targets the surrogate-induced objective $V^\alpha$. How closely $V^\alpha$ tracks $\Vcost$ depends on the surrogate. For SPO+ and PG, calibration results relate surrogate risk to decision regret in the offline, full-information setting \citep{liu2021risk,gupta2024decision}.

\paragraph{Scope of the analysis.}\label{sec:scope} Theorems~\ref{thm:main-alpha} and~\ref{thm:sur-main} analyze the fixed-$\alpha$ update with a constant policy step size, the standard setting for non-convex stochastic gradient methods \citep{ghadimi2013stochastic}, and a constant or decreasing nuisance step size. They allow any variance-reduction baseline independent of $\hat c_t$ given $x_t$. For the plug-in term, Theorem~\ref{thm:sur-main} covers surrogates with bounded, Lipschitz gradients (Assumption~\ref{ass:surrogate}). The experiments use standard DFL surrogates (SPO+, PG, Differentiable Black-Box (DBB), and ranking losses), which are built from oracle solutions and are therefore non-smooth in $\hat c$, so the evidence for them is empirical. The experiments add four practical adaptations that the theorems do not model: a mixing weight $\alpha_t$ that decays over iterations, a step-size schedule, separate $L_2$ normalization of the score function and plug-in gradients before they are mixed, and a neural-network nuisance model.

\section{Numerical experiments}

\subsection{Benchmark problems}
\label{sec:synthetic}

The experiments use four online contextual optimization benchmarks with bandit feedback: three synthetic linear-objective tasks (top-$k$ selection, shortest path, and combinatorial pricing) and a real-data energy-scheduling instance. At each iteration \(t\), a context $x_t$ is observed, a feasible decision is selected by solving the linear optimization problem at a sampled cost prediction, and only the realized scalar cost $y_t = c_t^\top w_t$ of the selected decision is observed; the latent cost vector $c_t \sim p^\star(\cdot \mid x_t)$ is not revealed. Data generation for the synthetic benchmarks and the energy-scheduling data source and chronological split are given in Appendix~\ref{app:synthetic_datagen}; semi-bandit and full-information variants of the feedback model are reported in Appendix~\ref{app:feedback_mode_comparison}.

\paragraph{Top-$k$ selection.}
\label{sec:synthetic_topk}
This benchmark models the selection of exactly $k$ items from the candidate set \(\mathcal{I}_{\mathrm{top}}=\{1,\ldots,d\}\). The binary decision variable \(w_i\) indicates whether item \(i\) is selected. Given predicted item costs $\hat c_t$, the top-$k$ oracle solves
\[
    \min_{w\in\{0,1\}^d}\ \hat c_t^\top w
    \quad \mathrm{s.t.}\quad
    \mathbf{1}^\top w = k .
\] We use $d=15$ items and $k=2$; the item costs are a polynomial of a $5$-dimensional context with multiplicative noise, whose degree sets the nonlinearity.

\paragraph{Shortest path.}
\label{sec:synthetic_shortest_path}
This benchmark uses a fixed directed acyclic grid graph (DAG) \(\mathcal{G}=(\mathcal{V},\mathcal{E})\) with right/down edges, where \(|\mathcal{V}|=25\) and \(|\mathcal{E}|=q=40\). The binary decision variable \(w_e\) indicates whether edge \(e\) is on the selected path. Let \(A\) denote the node--arc incidence matrix and \(b_{sd}\in\mathbb{R}^{|\mathcal{V}|}\) the source--destination balance vector ($+1$ at the source, $-1$ at the destination, $0$ elsewhere). Given predicted edge costs $\hat c_t$, the shortest-path oracle solves
\[
    \min_{w\in\{0,1\}^q}\ \hat c_t^\top w
    \quad \mathrm{s.t.}\quad
    Aw=b_{sd} .
\] The edge costs follow the same polynomial construction from a $10$-dimensional context.

\paragraph{Combinatorial pricing.}
\label{sec:synthetic_pricing}
This benchmark assigns one of $K_{\mathrm{price}}$ discrete price levels to each of $n$ products. Let \(w_{i\ell}\) indicate whether product \(i\) is assigned price level \(\ell\), let \(L_{\mathrm{promo}}\le K_{\mathrm{price}}\) be the number of lowest price tiers counted as promotional, and let \(B_{\mathrm{promo}}\) be the promotion budget. Given predicted product--price costs $\hat c_t$, the induced decision solves
\[
    \min_{w\in\{0,1\}^{nK_{\mathrm{price}}}}\ \hat c_t^\top w
    \quad \mathrm{s.t.}\quad
    \sum_{\ell=1}^{K_{\mathrm{price}}} w_{i\ell}=1\ \forall i,\quad
    \sum_{i=1}^n\sum_{\ell=1}^{L_{\mathrm{promo}}} w_{i\ell}\le B_{\mathrm{promo}} .
\]
Conditional on $x_t$, latent demand is sampled for every product--price pair; the action determines which entries are monetized, and the benchmark is written as cost minimization with cost equal to negative realized revenue. We use $n=20$ products, $K_{\mathrm{price}}=4$ price levels, and a $25$-dimensional context; demand follows a logit purchase model whose nonlinearity in the context is set by a degree parameter.

\paragraph{Energy scheduling.}
\label{sec:synthetic_energy}
This benchmark schedules tasks on machines under day-ahead electricity prices, following \citet{Mandi_2024} with empirical SEMO data. Each iteration $t$ corresponds to one day. Each task $j$ has duration $d_j$ slots, and each resource $u$ has machine-$m$ capacity $\bar\rho_{mu}$ and per-task usage $\rho_{ju}$. The binary decision $w_{jms}\in\{0,1\}$ indicates that task $j$ starts on machine $m$ at slot $s$, with $w\in\{0,1\}^{d_w}$ indexing feasible (task, machine, start-slot) triples. Aggregating the day's per-slot prices over each task's run window yields an iteration-$t$ cost vector $c_t\in\mathbb{R}^{d_w}$ that depends linearly on the price profile. Given a predicted cost $\hat c_t$, the scheduling oracle solves
\[
    \min_{w\in\{0,1\}^{d_w}}\ \hat c_t^\top w
    \ \ \mathrm{s.t.}\ \
    \sum_{m,s} w_{jms}=1\ \forall j,\ \
    \sum_{j} \rho_{ju}\!\!\sum_{s:s\le s'<s+d_j}\!\! w_{jms} \le \bar\rho_{mu}\ \forall (m,u,s').
\] The electricity price in each of the $48$ half-hour slots of a day is predicted from $8$ features per slot in the empirical Single Electricity Market Operator (SEMO) data of \citet{Mandi_2024}.

\subsection{Experimental setup}\label{sec:setup}

The proposed method is reported as \textsc{DFHPG}. It uses the hybrid gradient update in Algorithm~\ref{alg:main}, with the fixed mixing weight $\alpha$ replaced by a time-varying $\alpha_t$ decaying from $\alpha_{\max}$ to $\alpha_{\min}$. The score function gradient subtracts a variance-reduction baseline $b_t$. Two ablations isolate the two gradient components: \textsc{DFHPG-1} uses only the score function update, \textsc{DFHPG-0} only the plug-in update. Appendix~\ref{app:dfhpg_details} specifies these components, and Appendix~\ref{app:hyperparameter_selection} the tuning and selected configurations.

We compare against three contextual bandit (CB) baselines that share \textsc{DFHPG}'s context, feedback, and linear optimization oracle, isolating the effect of the learning update from the action-selection rule. \textsc{GreedyCB} uses greedy exploitation with a point cost predictor; $\epsilon$-\textsc{GreedyCB} adds random exploration; \textsc{TSCB} samples a cost vector from the learned conditional cost distribution and solves the linear optimization at that sample, matching Thompson sampling's action-level randomization without an explicit Bayesian posterior. We also compare with two standard linear bandit algorithms. Optimism in the face of uncertainty for linear bandits (OFUL) \citep{abbasi2011improved} has a confidence bonus that is nonlinear in the decision, so our implementation enumerates the feasible set, and we run it on top-$k$ and shortest path. Linear Thompson sampling (LinTS) \citep{agrawal2013thompson} keeps the linear oracle and maintains a closed-form Gaussian posterior over the cost-model parameter, so we run it on all four benchmarks with each benchmark's cost model. Pseudocode and settings for all baselines are in Appendix~\ref{app:baseline_algorithms}.

The default distributional cost-vector prediction model is the Gaussian linear model; the Gaussian neural-network (NN), conditional normalizing-flow (CNF), and denoising diffusion probabilistic model (DDPM) variants are also supported and reported as alternatives in Section~\ref{sec:results} and Appendix~\ref{app:model_variants_remaining}. The Gaussian variants set $p_\theta(\cdot\mid x) = \mathcal N(\bar\mu_\theta(x),\, \sigma_{\mathrm{G}}^2 I_{d_w})$, where $\bar\mu_\theta$ is the mean prediction function parameterized either as a linear map (Gaussian linear) or as a two-layer neural network with $128$ hidden units (Gaussian NN), and $\sigma_{\mathrm{G}}>0$ is a fixed exploration scale. The CNF variant uses a two-layer conditional RealNVP flow~\citep{dinh2017density} with $128$-dimensional coupling networks; the DDPM variant uses a conditional denoiser~\citep{ho2020denoising} with a $128$-dimensional hidden layer. Update rules for the generative variants are in Appendix~\ref{app:generative_param}. All four linear optimization oracles are solved exactly: top-$k$ by sorting predicted item costs, shortest path and combinatorial pricing by dynamic programming, and energy scheduling as a mixed-integer program with Gurobi. The main metric reported is cumulative regret, $\operatorname{Regret}_T = \sum_{t=1}^T \bigl(f^\star(x_t)^\top w_t-\min_{w\in \calW} f^\star(x_t)^\top w\bigr)$, where $w_t \in \calW$ is the decision implemented at iteration $t$ and $f^\star(x_t)=\E[c_t\mid x_t]$ is the mean cost. Convergence figures plot the ergodic average regret per iteration, $\operatorname{Regret}_t/t$.

\subsection{Results}
\label{sec:results}

Table~\ref{tab:main-methods-selected-settings} reports final cumulative regret across the four benchmarks under both feedback regimes. Under pure bandit feedback, \textsc{DFHPG} achieves the lowest mean regret on top-$k$ selection and pricing, and the lowest among the contextual-bandit methods on shortest path; on energy, both \textsc{DFHPG} and \textsc{DFHPG-0} beat the contextual-bandit baselines by roughly $1.7$--$1.9\times$, and the two perform alike (each lower in about half of the $30$ replications), while \textsc{DFHPG-1} has far higher regret. \textsc{DFHPG-0} is the $\alpha=0$ member of the \textsc{DFHPG} family, and the score term helps on top-$k$ and shortest path but adds little on pricing and energy (Appendix~\ref{app:stationarity}). Semi-bandit feedback lowers the cumulative regret of \textsc{DFHPG} by roughly $26$--$58\%$ on the synthetic benchmarks and ${\sim}39\%$ on energy, and \textsc{DFHPG-0} has lower mean regret than \textsc{DFHPG} in every column: the coordinate-level signal feeds the nuisance fit directly, so the score function term is most useful in the pure-bandit regime. On pricing, the semi-bandit versions of the contextual-bandit baselines are lower still (Appendix~\ref{app:feedback_mode_comparison}). \textsc{DFHPG} has lower mean regret than OFUL and LinTS on top-$k$ and lower regret than LinTS in 27 of 30 energy replications. On shortest path, \textsc{DFHPG} stays competitive with OFUL and LinTS but does not beat them: it has higher regret than OFUL in $19$ and than LinTS in $16$ of the $30$ replications. Here the linear bandits benefit from re-solving the regularized least-squares problem for the linear cost model after every round, whereas \textsc{DFHPG} takes one stochastic gradient step, and with only $70$ feasible paths OFUL can also score every path. On pricing, \textsc{DFHPG} has $4\%$ lower regret than LinTS. More broadly, OFUL and LinTS need a cost model that is linear in its parameters, which keeps their least-squares updates and posterior in closed form, whereas \textsc{DFHPG} also runs with the neural, flow, and diffusion cost models of Appendix~\ref{app:model_variants_remaining}.

\begin{table}[!ht]
\centering
\scriptsize
\setlength{\tabcolsep}{3.5pt}
\caption{Final cumulative regret (mean $\pm$ standard error, 30 replications) with the Gaussian linear cost model: top-$k$ and shortest path at degree 8 and \(T=2{,}000\), pricing at degree 6 and \(T=10{,}000\). Best in bold, second-best underlined per feedback block. OFUL runs where the feasible set can be enumerated.}
\label{tab:main-methods-selected-settings}
\resizebox{\textwidth}{!}{%
\begin{tabular}{lcccc}
\toprule
Method & Top-$k$ (deg. 8) & Shortest path (deg. 8) & Pricing (deg. 6) & Energy \\
\midrule
\multicolumn{5}{l}{\textit{Pure bandit feedback}} \\
\midrule
\textsc{GreedyCB}
& $6.94{\times}10^5 \pm 3.38{\times}10^4$
& $1.98{\times}10^6 \pm 4.40{\times}10^4$
& $7.38{\times}10^4 \pm 1.6{\times}10^2$
& $6.43{\times}10^7 \pm 1.95{\times}10^5$ \\
\(\epsilon\)-\textsc{GreedyCB}
& $7.17{\times}10^5 \pm 3.98{\times}10^4$
& $2.06{\times}10^6 \pm 4.12{\times}10^4$
& $7.61{\times}10^4 \pm 1.8{\times}10^2$
& $6.58{\times}10^7 \pm 3.01{\times}10^5$ \\
\textsc{TSCB}
& $6.93{\times}10^5 \pm 3.87{\times}10^4$
& $2.00{\times}10^6 \pm 4.37{\times}10^4$
& $7.38{\times}10^4 \pm 2.0{\times}10^2$
& $6.43{\times}10^7 \pm 1.94{\times}10^5$ \\
OFUL
& $\underline{1.70{\times}10^5 \pm 1.45{\times}10^4}$
& $\boldsymbol{1.42{\times}10^6 \pm 3.73{\times}10^4}$
& -- & -- \\
LinTS
& $1.79{\times}10^5 \pm 1.90{\times}10^4$
& $\underline{1.45{\times}10^6 \pm 3.66{\times}10^4}$
& $7.29{\times}10^4 \pm 2.4{\times}10^2$
& $5.50{\times}10^7 \pm 1.99{\times}10^5$ \\
\textsc{DFHPG-1}
& $4.00{\times}10^5 \pm 4.49{\times}10^4$
& $1.81{\times}10^6 \pm 6.08{\times}10^4$
& $8.55{\times}10^4 \pm 2.7{\times}10^2$
& $1.60{\times}10^8 \pm 1.86{\times}10^7$ \\
\textsc{DFHPG}
& $\boldsymbol{1.68{\times}10^5 \pm 1.47{\times}10^4}$
& $1.50{\times}10^6 \pm 5.04{\times}10^4$
& $\boldsymbol{7.00{\times}10^4 \pm 2.9{\times}10^2}$
& $\underline{3.85{\times}10^7 \pm 8.79{\times}10^6}$ \\
\textsc{DFHPG-0}
& $1.99{\times}10^5 \pm 2.38{\times}10^4$
& $1.58{\times}10^6 \pm 5.15{\times}10^4$
& $\underline{7.02{\times}10^4 \pm 3.6{\times}10^2}$
& $\boldsymbol{3.50{\times}10^7 \pm 8.12{\times}10^6}$ \\
\midrule
\multicolumn{5}{l}{\textit{Semi-bandit feedback}} \\
\midrule
\textsc{DFHPG}
& $\underline{7.10{\times}10^4 \pm 8.52{\times}10^3}$
& $\underline{6.83{\times}10^5 \pm 3.11{\times}10^4}$
& $\underline{5.19{\times}10^4 \pm 1.0{\times}10^2}$
& $\underline{2.35{\times}10^7 \pm 4.21{\times}10^5}$ \\
\textsc{DFHPG-0}
& $\boldsymbol{6.65{\times}10^4 \pm 5.30{\times}10^3}$
& $\boldsymbol{5.31{\times}10^5 \pm 2.28{\times}10^4}$
& $\boldsymbol{5.18{\times}10^4 \pm 9.8{\times}10^1}$
& $\boldsymbol{2.33{\times}10^7 \pm 3.78{\times}10^5}$ \\
\bottomrule
\end{tabular}%
}
\end{table}

\paragraph{Stationarity.} Theorems~\ref{thm:main-alpha} and~\ref{thm:sur-main} guarantee stationarity of $\Vcost$ and $V^\alpha$, whereas Table~\ref{tab:main-methods-selected-settings} measures regret, and the two connect through $\nabla_\theta\Vcost$, the gradient of expected per-round regret. Because $f^\star$ is known on the synthetic benchmarks, we estimate $\|\nabla_\theta\Vcost(\theta_t)\|^2$ along the fixed-$\alpha$ runs at $\alpha=0$ (constant step sizes, no gradient normalization). Its ergodic average falls $4.3\times$ on top-$k$ and $8.5\times$ on shortest path between $t=100$ and $2000$, together with per-round regret (Table~\ref{tab:stationarity}; measurement details in Appendix~\ref{app:stationarity}).

\begin{table}[!ht]
\centering
\small
\caption{Ergodic averages up to iteration $t$ of $\|\nabla_\theta\Vcost(\theta_s)\|^2$ and of per-round regret, fixed-$\alpha$ configuration at $\alpha=0$. Mean $\pm$ standard error over 30 replications.}
\label{tab:stationarity}
\begin{tabular}{lccccc}
\toprule
$t$ & 100 & 250 & 500 & 1000 & 2000 \\
\midrule
\multicolumn{6}{l}{\textit{Top-$k$}} \\
$\|\nabla\Vcost\|^2$ ($\times10^5$) & $1.50\pm0.13$ & $0.98\pm0.09$ & $0.66\pm0.06$ & $0.44\pm0.04$ & $0.35\pm0.05$ \\
Regret per round & $310\pm38$ & $277\pm21$ & $224\pm17$ & $170\pm10$ & $148\pm7$ \\
\multicolumn{6}{l}{\textit{Shortest path}} \\
$\|\nabla\Vcost\|^2$ ($\times10^5$) & $8.68\pm1.86$ & $4.35\pm0.91$ & $1.83\pm0.60$ & $1.40\pm0.41$ & $1.02\pm0.23$ \\
Regret per round & $1009\pm77$ & $996\pm54$ & $904\pm36$ & $876\pm36$ & $762\pm27$ \\
\bottomrule
\end{tabular}
\end{table}

\begin{figure}[!ht]
\centering
\begin{subfigure}[t]{0.49\textwidth}
\centering
\includegraphics[width=\textwidth]{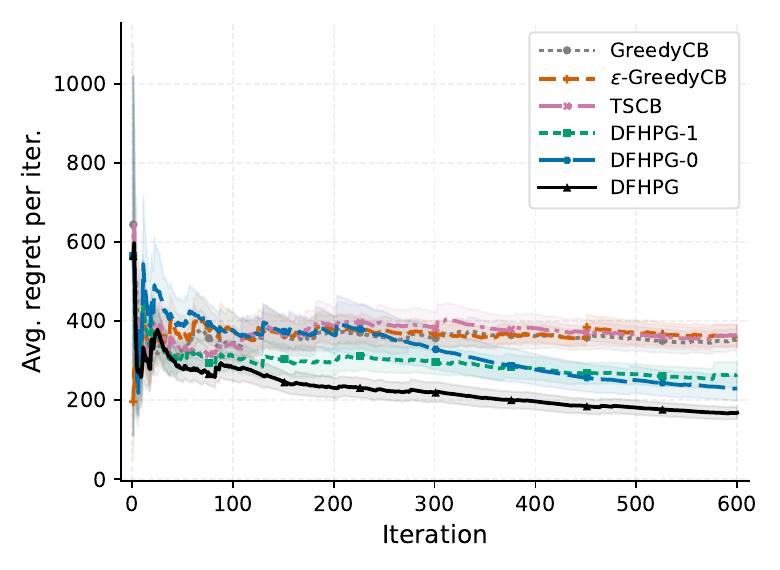}
\caption{Top-$k$ selection}
\label{fig:main-convergence-topk}
\end{subfigure}\hfill
\begin{subfigure}[t]{0.49\textwidth}
\centering
\includegraphics[width=\textwidth]{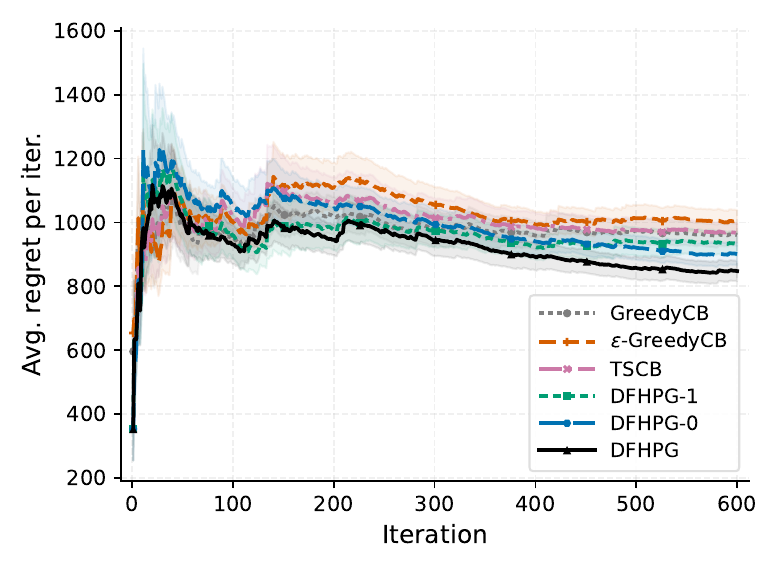}
\caption{Shortest path}
\end{subfigure}

\vspace{0.75em}
\begin{subfigure}[t]{0.49\textwidth}
\centering
\includegraphics[width=\textwidth]{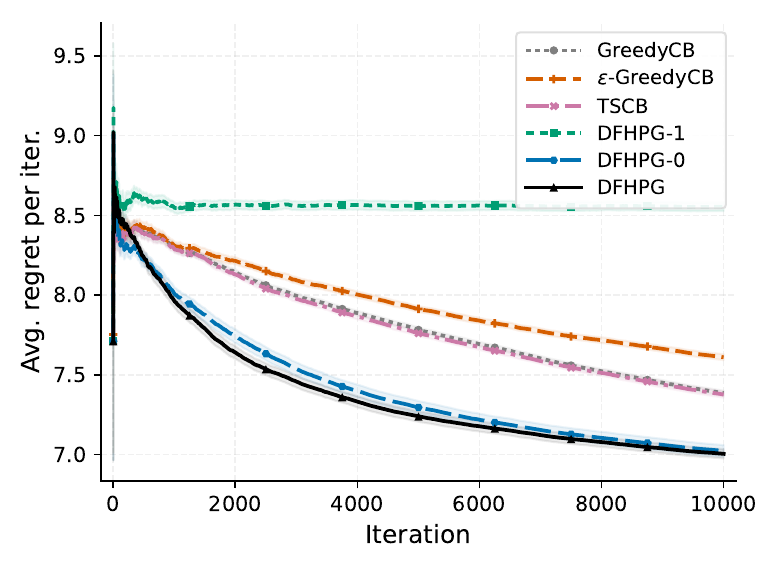}
\caption{Pricing}
\end{subfigure}\begin{subfigure}[t]{0.49\textwidth}
\centering
\includegraphics[width=\textwidth]{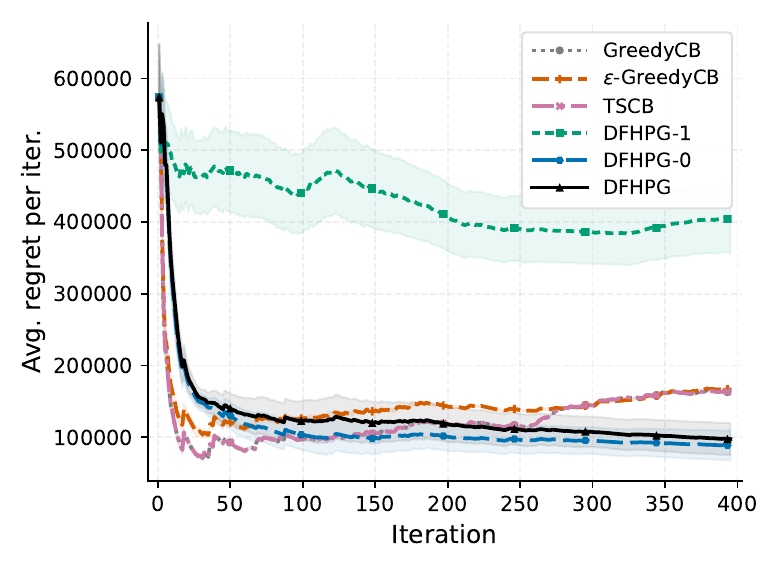}
\caption{Energy scheduling}
\end{subfigure}
\caption{Convergence plots under pure bandit feedback. Top-$k$ and shortest path are zoomed to the first 600 iterations of the \(T=2{,}000\) run, pricing shows its full \(T=10{,}000\) horizon, and energy scheduling uses its native test-window length. Each curve reports average regret per iteration, with shaded $\pm 1$ standard-error bands across replications.}
\label{fig:main-convergence}
\end{figure}

Figure~\ref{fig:main-convergence-topk} traces top-$k$ dynamics over the first 600 iterations and shows that the score function and plug-in components are complementary in time. For the first $\sim\!300$ iterations \textsc{DFHPG-1} sits below \textsc{DFHPG-0}: the score function gradient is unbiased and immediately usable, while the plug-in gradient relies on a nuisance that has not accumulated enough data. The order reverses past $\sim\!300$ iterations, when the nuisance stabilizes and the lower-variance plug-in signal pulls \textsc{DFHPG-0} below \textsc{DFHPG-1}; through its time-varying weighting, hybrid \textsc{DFHPG} follows the better component in each phase and ends below the contextual-bandit baselines. On shortest path, all methods pay for the noisy first iterations, but the hybrid curve settles below the contextual-bandit baselines and remains the lowest curve through the later part of this window. The pricing panel shows the full \(T=10{,}000\) horizon: \textsc{DFHPG} and \textsc{DFHPG-0} separate from the contextual-bandit baselines after the first few hundred iterations and stay lowest, while \textsc{DFHPG-1} stays highest.

\begin{figure}[!ht]
\centering
\includegraphics[width=0.6\textwidth]{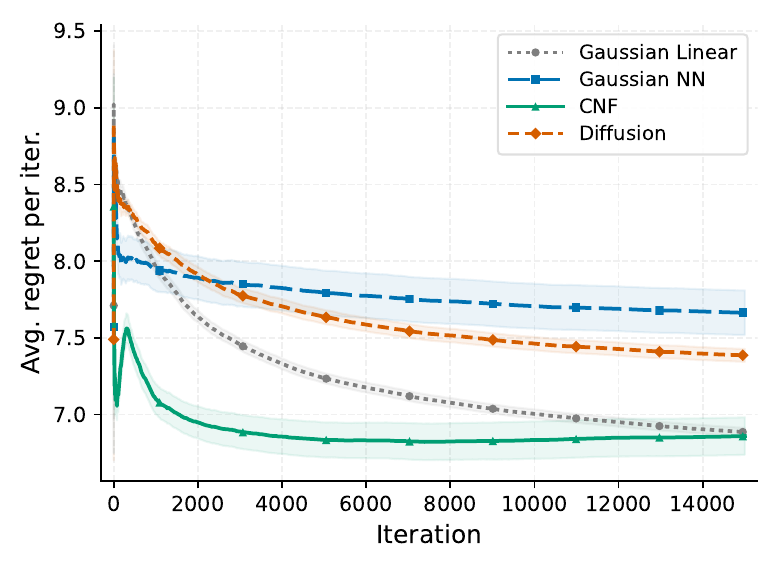}
\caption{Distributional cost models on pricing at $T=15{,}000$. Each curve reports average regret per iteration, with shaded $\pm 1$ standard-error bands across replications.}
\label{fig:main-model-variants}
\end{figure}

Figure~\ref{fig:main-model-variants} compares distributional cost models on pricing at the extended horizon $T=15{,}000$. CNF has the lowest average regret for most of the run but levels off after about $5{,}000$ iterations, while the Gaussian linear model keeps improving and nearly matches it at $T=15{,}000$ ($6.86\pm0.12$ and $6.89\pm0.03$ per iteration). The Gaussian NN and diffusion models remain above both ($7.66\pm0.15$ and $7.39\pm0.04$). The same comparison on the other three benchmarks (Appendix~\ref{app:model_variants_remaining}) shows that no single cost model dominates across benchmarks, and the \textsc{DFHPG} framework accommodates all four. Sensitivity studies on problem difficulty, the mixing weight, the surrogate loss, and the feedback mode are in Appendix~\ref{app:additional_experimental_results}.

\section{Conclusion}

\textsc{DFHPG} couples an unbiased score function update with a decision-focused plug-in update, and under the stated assumptions its fixed-weight update reaches first-order stationarity at rate $\calO(T^{-1/2})$. Under bandit feedback, \textsc{DFHPG} and the plug-in-only \textsc{DFHPG-0} reduce cumulative regret relative to the contextual-bandit baselines and the score-only ablation on all four benchmarks, with \textsc{DFHPG} strongest on the synthetic benchmarks and \textsc{DFHPG-0} on energy. \textsc{DFHPG} is also more flexible than linear bandit algorithms: it needs neither a closed-form posterior (LinTS) nor an enumerable feasible set (OFUL), only a distributional cost model with the required gradient estimators and a linear optimization oracle. Extending the framework beyond linear feedback and risk-neutral objectives is a natural next step.

\section*{Acknowledgments}

This research was partly supported by the NSF AI Institute for Advances in Optimization (Award 2112533). The authors sincerely thank Amira Hijazi for discussions that helped scope this project.

\bibliographystyle{apalike}
\bibliography{refs}

\begin{thebibliography}{}

\bibitem[Abbasi-Yadkori et~al., 2011]{abbasi2011improved}
Abbasi-Yadkori, Y., P{\'a}l, D., and Szepesv{\'a}ri, C. (2011).
\newblock Improved algorithms for linear stochastic bandits.
\newblock {\em Advances in neural information processing systems}, 24.

\bibitem[Abernethy et~al., 2008]{abernethy2008competing}
Abernethy, J., Hazan, E., and Rakhlin, A. (2008).
\newblock Competing in the dark: An efficient algorithm for bandit linear
  optimization.
\newblock In {\em Proceedings of the 21st Annual Conference on Learning
  Theory}, pages 263--274.

\bibitem[Agrawal and Goyal, 2013]{agrawal2013thompson}
Agrawal, S. and Goyal, N. (2013).
\newblock Thompson sampling for contextual bandits with linear payoffs.
\newblock In {\em Proceedings of the 30th International Conference on Machine
  Learning}, volume~28 of {\em Proceedings of Machine Learning Research}, pages
  127--135.

\bibitem[Bennouna et~al., 2025]{bennouna2025data}
Bennouna, O., Bennouna, A., Amin, S., and Ozdaglar, A. (2025).
\newblock What data enables optimal decisions? {An} exact characterization for
  linear optimization.
\newblock In {\em Advances in Neural Information Processing Systems},
  volume~38, pages 184042--184071.

\bibitem[Bennouna et~al., 2026]{bennouna2026informativeness}
Bennouna, O., Bennouna, A., Amin, S., and Ozdaglar, A. (2026).
\newblock Data informativeness in linear optimization under uncertainty.

\bibitem[Berthet et~al., 2020]{berthet2020learning}
Berthet, Q., Blondel, M., Teboul, O., Cuturi, M., Vert, J.-P., and Bach, F.
  (2020).
\newblock Learning with differentiable perturbed optimizers.
\newblock In {\em Advances in Neural Information Processing Systems},
  volume~33, pages 9508--9519.

\bibitem[Bouvier et~al., 2025]{bouvier2025primal}
Bouvier, L., Prunet, T., Lecl{\`e}re, V., and Parmentier, A. (2025).
\newblock Primal-dual algorithm for contextual stochastic combinatorial
  optimization.

\bibitem[Bubeck and Cesa-Bianchi,
  2012]{bubeck2012regretanalysisstochasticnonstochastic}
Bubeck, S. and Cesa-Bianchi, N. (2012).
\newblock Regret analysis of stochastic and nonstochastic multi-armed bandit
  problems.
\newblock {\em Foundations and Trends in Machine Learning}, 5(1):1--122.

\bibitem[Capitaine et~al., 2026]{capitaine2025online}
Capitaine, A., Haddouche, M., Moulines, E., Jordan, M.~I., Boursier, E., and
  Durmus, A.~O. (2026).
\newblock Online decision-focused learning.
\newblock In {\em International Conference on Learning Representations}.

\bibitem[Dani et~al., 2008]{dani2008stochastic}
Dani, V., Hayes, T.~P., and Kakade, S.~M. (2008).
\newblock Stochastic linear optimization under bandit feedback.
\newblock In {\em Proceedings of the 21st Annual Conference on Learning
  Theory}, pages 355--366.

\bibitem[Dinh et~al., 2017]{dinh2017density}
Dinh, L., Sohl-Dickstein, J., and Bengio, S. (2017).
\newblock Density estimation using real nvp.
\newblock In {\em International Conference on Learning Representations}.

\bibitem[Elmachtoub and Grigas, 2022]{elmachtoub2022smart}
Elmachtoub, A.~N. and Grigas, P. (2022).
\newblock Smart “predict, then optimize”.
\newblock {\em Management Science}, 68(1):9--26.

\bibitem[Flaxman et~al., 2005]{flaxman2005online}
Flaxman, A.~D., Kalai, A.~T., and McMahan, H.~B. (2005).
\newblock Online convex optimization in the bandit setting: gradient descent
  without a gradient.
\newblock In {\em Proceedings of the Sixteenth Annual ACM-SIAM Symposium on
  Discrete Algorithms}, SODA '05, pages 385--394, USA. Society for Industrial
  and Applied Mathematics.

\bibitem[Ghadimi and Lan, 2013]{ghadimi2013stochastic}
Ghadimi, S. and Lan, G. (2013).
\newblock Stochastic first- and zeroth-order methods for nonconvex stochastic
  programming.
\newblock {\em SIAM Journal on Optimization}, 23(4):2341--2368.

\bibitem[Gupta and Huang, 2024]{gupta2024decision}
Gupta, V. and Huang, M. (2024).
\newblock Decision-focused learning with directional gradients.
\newblock {\em Advances in Neural Information Processing Systems},
  37:79194--79220.

\bibitem[Ho et~al., 2020]{ho2020denoising}
Ho, J., Jain, A., and Abbeel, P. (2020).
\newblock Denoising diffusion probabilistic models.
\newblock {\em Advances in Neural Information Processing Systems},
  33:6840--6851.

\bibitem[Hu et~al., 2024a]{hu2024multi}
Hu, X., Lee, J. C.~H., Lee, J. H.~M., and Stuckey, P.~J. (2024a).
\newblock Multi-stage predict+optimize for (mixed integer) linear programs.
\newblock {\em Advances in Neural Information Processing Systems},
  37:64794--64827.

\bibitem[Hu et~al., 2024b]{hu2024contextual}
Hu, Y., Kallus, N., Mao, X., and Wu, Y. (2024b).
\newblock Contextual linear optimization with bandit feedback.
\newblock {\em Advances in Neural Information Processing Systems},
  37:26976--27004.

\bibitem[Hu et~al., 2025]{hu2025contextual}
Hu, Y., Kallus, N., Mao, X., and Wu, Y. (2025).
\newblock Contextual linear optimization under partial feedback.
\newblock {\em arXiv preprint arXiv:2405.16564}.

\bibitem[Ito et~al., 2019]{ito2019oracle}
Ito, S., Hatano, D., Sumita, H., Takemura, K., Fukunaga, T., Kakimura, N., and
  Kawarabayashi, K.-I. (2019).
\newblock Oracle-efficient algorithms for online linear optimization with
  bandit feedback.
\newblock In {\em Advances in Neural Information Processing Systems},
  volume~32.

\bibitem[Konda and Tsitsiklis, 2000]{konda2000actor}
Konda, V.~R. and Tsitsiklis, J.~N. (2000).
\newblock Actor-critic algorithms.
\newblock In {\em Advances in Neural Information Processing Systems},
  volume~12, pages 1008--1014.

\bibitem[Kreutzer et~al., 2017]{kreutzer2017bandit}
Kreutzer, J., Sokolov, A., and Riezler, S. (2017).
\newblock Bandit structured prediction for neural sequence-to-sequence
  learning.
\newblock In {\em Proceedings of the 55th Annual Meeting of the Association for
  Computational Linguistics (Volume 1: Long Papers)}, pages 1503--1513.

\bibitem[Krishnamurthy et~al., 2016]{krishnamurthy2016contextual}
Krishnamurthy, A., Agarwal, A., and Dudik, M. (2016).
\newblock Contextual semibandits via supervised learning oracles.
\newblock In {\em Advances in Neural Information Processing Systems},
  volume~29.

\bibitem[Li et~al., 2010]{li2010contextual}
Li, L., Chu, W., Langford, J., and Schapire, R.~E. (2010).
\newblock A contextual-bandit approach to personalized news article
  recommendation.
\newblock In {\em Proceedings of the 19th international conference on World
  wide web}, pages 661--670.

\bibitem[Liu and Grigas, 2021]{liu2021risk}
Liu, H. and Grigas, P. (2021).
\newblock Risk bounds and calibration for a smart predict-then-optimize method.
\newblock In {\em Advances in Neural Information Processing Systems},
  volume~34, pages 22083--22094.

\bibitem[Liu and Grigas, 2022]{liu2022online}
Liu, H. and Grigas, P. (2022).
\newblock Online contextual decision-making with a smart predict-then-optimize
  method.

\bibitem[Mandi et~al., 2022]{mandi2022decision}
Mandi, J., Bucarey, V., Tchomba, M. M.~K., and Guns, T. (2022).
\newblock Decision-focused learning: Through the lens of learning to rank.
\newblock In {\em Proceedings of the 39th International Conference on Machine
  Learning}, volume 162 of {\em Proceedings of Machine Learning Research},
  pages 14935--14947.

\bibitem[Mandi et~al., 2024]{Mandi_2024}
Mandi, J., Kotary, J., Berden, S., Mulamba, M., Bucarey, V., Guns, T., and
  Fioretto, F. (2024).
\newblock Decision-focused learning: Foundations, state of the art, benchmark
  and future opportunities.
\newblock {\em Journal of Artificial Intelligence Research}, 80:1623--1701.

\bibitem[Minervini et~al., 2023]{minervini2023adaptive}
Minervini, P., Franceschi, L., and Niepert, M. (2023).
\newblock Adaptive perturbation-based gradient estimation for discrete latent
  variable models.
\newblock In {\em Proceedings of the AAAI Conference on Artificial
  Intelligence}, volume~37, pages 9200--9208.

\bibitem[Mulamba et~al., 2021]{mulamba2021contrastive}
Mulamba, M., Mandi, J., Diligenti, M., Lombardi, M., Bucarey, V., and Guns, T.
  (2021).
\newblock Contrastive losses and solution caching for predict-and-optimize.
\newblock In {\em Proceedings of the Thirtieth International Joint Conference
  on Artificial Intelligence}, pages 2833--2840.

\bibitem[Niepert et~al., 2021]{niepert2021implicit}
Niepert, M., Minervini, P., and Franceschi, L. (2021).
\newblock Implicit {MLE}: Backpropagating through discrete exponential family
  distributions.
\newblock In {\em Advances in Neural Information Processing Systems},
  volume~34, pages 14567--14579.

\bibitem[Orabona, 2019]{orabona2019modern}
Orabona, F. (2019).
\newblock A modern introduction to online learning.

\bibitem[Powell, 2022]{powell2022designing}
Powell, W.~B. (2022).
\newblock Designing lookahead policies for sequential decision problems in
  transportation and logistics.
\newblock {\em IEEE Open Journal of Intelligent Transportation Systems},
  3:313--327.

\bibitem[Sadana et~al., 2025]{sadana2025survey}
Sadana, U., Chenreddy, A., Delage, E., Forel, A., Frejinger, E., and Vidal, T.
  (2025).
\newblock A survey of contextual optimization methods for decision-making under
  uncertainty.
\newblock {\em European Journal of Operational Research}, 320(2):271--289.

\bibitem[Sahoo et~al., 2023]{sahoo2023backpropagation}
Sahoo, S.~S., Paulus, A., Vlastelica, M., Musil, V., Kuleshov, V., and Martius,
  G. (2023).
\newblock Backpropagation through combinatorial algorithms: Identity with
  projection works.
\newblock In {\em International Conference on Learning Representations (ICLR)}.

\bibitem[Silvestri et~al., 2026]{silvestri2026score}
Silvestri, M., Berden, S., Signorelli, G., Mahmuto{\u{g}}ullar{\i}, A.~{\.I}.,
  Mandi, J., Amos, B., Guns, T., and Lombardi, M. (2026).
\newblock Score function gradient estimation to widen the applicability of
  decision-focused learning.
\newblock {\em Journal of Artificial Intelligence Research},
  85(10):10:1--10:34.

\bibitem[Sokolov et~al., 2016]{sokolov2016stochastic}
Sokolov, A., Kreutzer, J., Riezler, S., and Lo, C. (2016).
\newblock Stochastic structured prediction under bandit feedback.
\newblock In {\em Advances in Neural Information Processing Systems},
  volume~29.

\bibitem[Sutton and Barto, 2018]{sutton2018reinforcement}
Sutton, R.~S. and Barto, A.~G. (2018).
\newblock {\em Reinforcement Learning: An Introduction}.
\newblock MIT Press, 2 edition.

\bibitem[Sutton et~al., 2000]{sutton2000policy}
Sutton, R.~S., McAllester, D., Singh, S., and Mansour, Y. (2000).
\newblock Policy gradient methods for reinforcement learning with function
  approximation.
\newblock In {\em Advances in Neural Information Processing Systems},
  volume~12, pages 1057--1063.

\bibitem[Tang and Khalil, 2024]{tang2024pyepo}
Tang, B. and Khalil, E.~B. (2024).
\newblock {PyEPO}: A {PyTorch}-based end-to-end predict-then-optimize library
  for linear and integer programming.
\newblock {\em Mathematical Programming Computation}, 16(3):297--335.

\bibitem[Van~Hentenryck and Bent, 2006]{van2006online}
Van~Hentenryck, P. and Bent, R. (2006).
\newblock {\em Online stochastic combinatorial optimization}.
\newblock The MIT Press.

\bibitem[Vlastelica~Pogan{\v{c}}i{\'c} et~al.,
  2020]{vlastelica2020differentiation}
Vlastelica~Pogan{\v{c}}i{\'c}, M., Paulus, A., Musil, V., Martius, G., and
  Rol{\'\i}nek, M. (2020).
\newblock Differentiation of blackbox combinatorial solvers.
\newblock In {\em International Conference on Learning Representations (ICLR)}.

\bibitem[Wang et~al., 2026]{wang2025gendfl}
Wang, P.~Z., Chen, S., Liang, J., Fioretto, F., and Zhu, S. (2026).
\newblock Gen-{DFL}: Decision-focused generative learning for robust decision
  making.
\newblock In {\em International Conference on Learning Representations}.

\bibitem[Williams, 1992]{williams1992simple}
Williams, R.~J. (1992).
\newblock Simple statistical gradient-following algorithms for connectionist
  reinforcement learning.
\newblock {\em Machine Learning}, 8(3--4):229--256.

\bibitem[Ye et~al., 2026a]{ye2025contextual}
Ye, T., Cheng, S., Hijazi, A., and Van~Hentenryck, P. (2026a).
\newblock Contextual stochastic optimization for omnichannel multicourier order
  fulfillment under delivery time uncertainty.
\newblock {\em Manufacturing \& Service Operations Management},
  28(4):1068--1090.

\bibitem[Ye et~al., 2026b]{ye2026learning}
Ye, T., Tong, S., Guan, C., Basciftci, B., and Van~Hentenryck, P. (2026b).
\newblock Learning optimization proxies for sequential contextual stochastic
  programs: {An} order fulfillment application.

\bibitem[Zhao et~al., 2026]{zhao2026diffusion}
Zhao, Z., Yeh, C., Kong, L., and Wang, K. (2026).
\newblock Diffusion-{DFL}: Decision-focused diffusion models for stochastic
  optimization.
\newblock In {\em International Conference on Learning Representations}.

\bibitem[Zhu et~al., 2022]{zhu2022contextual}
Zhu, Y., Foster, D.~J., Langford, J., and Mineiro, P. (2022).
\newblock Contextual bandits with large action spaces: Made practical.
\newblock In {\em Proceedings of the 39th International Conference on Machine
  Learning}, volume 162 of {\em Proceedings of Machine Learning Research},
  pages 27428--27453.

\bibitem[Zierahn et~al., 2023]{zierahn2023nonstochastic}
Zierahn, L., van~der Hoeven, D., Cesa-Bianchi, N., and Neu, G. (2023).
\newblock Nonstochastic contextual combinatorial bandits.
\newblock In {\em Proceedings of The 26th International Conference on
  Artificial Intelligence and Statistics}, volume 206 of {\em Proceedings of
  Machine Learning Research}, pages 8771--8813.

\bibitem[Zinkevich, 2003]{zinkevich2003online}
Zinkevich, M. (2003).
\newblock Online convex programming and generalized infinitesimal gradient
  ascent.
\newblock In {\em Proceedings of the 20th international conference on machine
  learning (icml-03)}, pages 928--936.

\end{thebibliography}

\appendix
\paragraph{Organization.} Appendices~\ref{apdx:perturbed_optimizer}--\ref{apdx:surrogate} contain the assumptions and proofs. Appendices~\ref{app:synthetic_datagen}--\ref{app:hyperparameter_selection} give the experimental details: benchmark data generation, \textsc{DFHPG} implementation, baselines, and hyperparameters. Appendix~\ref{app:additional_experimental_results} reports additional results.
\appendix
\section{Regularity of the perturbed optimizer and score function representation}\label{apdx:perturbed_optimizer}
We consider the perturbed optimizers
$w^\star_\theta(x) := \E_{\hat c\sim p_\theta(\cdot \mid x)}\!\bigl[w^\star(\hat c)\bigr]$. The smoothing induced by $p_\theta$ makes $ w^\star_\theta(x)$ differentiable in $\theta$ even though the underlying optimizer $\hat c\mapsto w^\star(\hat c)$ is piecewise constant. The following assumptions and supporting results formalize this differentiability and the resulting regularity of the policy value $\Vcost$.

\begin{assumption}[Bounded feasible set and costs]\label{ass:bounded}
$\calW$ is compact with $B_{\calW} := \sup_{w \in \calW}\|w\| < \infty$, and there exists $B_c > 0$ such that $\|c_t\| \le B_c$ almost surely.
\end{assumption}

\begin{assumption}[Smooth parametric family]\label{ass:score-bound}
The parametric family $p_\theta(\cdot \mid x)$ is twice continuously differentiable in $\theta$ for all $x \in \calX$, and there exist constants $M, M_\nabla < \infty$ such that for all $\theta, \theta' \in \R^{d_\theta}$ and $x \in \calX$,
\begin{align}
\int \|\nabla_\theta p_\theta(\hat c\mid x)\|\, \mathrm{d} \hat c
  &\le M \label{eq:density-grad-bound} \\
\int \|\nabla_\theta p_\theta(\hat c\mid x) - \nabla_{\theta} p_{\theta'}(\hat c\mid x)\|\, \mathrm{d} \hat c
  &\le M_\nabla \|\theta - \theta'\| \label{eq:density-grad-lip}
\end{align}
\end{assumption}
Conditions~\eqref{eq:density-grad-bound}--\eqref{eq:density-grad-lip} bound the parametric family in $L^1$ on the gradient and its variation, and are satisfied by all common smooth parameterizations, e.g., Gaussian families with smoothly parameterized mean and covariance, or exponential families with bounded sufficient statistics. They ensure that $\theta \mapsto \nabla_\theta w^\star_\theta(x)$ is itself Lipschitz, which is needed to establish $L^J_\theta$-smoothness of $\Vcost$ in Lemma~\ref{lem:J-reg}(ii).

\begin{assumption}[Reparameterization]\label{ass:reparam}
Under $p_\theta(\cdot \mid x)$, we assume the reparameterization $\hat c= \mu_\theta(x, \varepsilon)$ with $\varepsilon \sim \nu$ independent of $\theta$. The map $\mu_\theta$ is continuously differentiable in $\theta$ with $\|\nabla_\theta \mu_\theta(x, \varepsilon)\| \le G_\mu$ almost surely for all $\theta \in \R^{d_\theta}$ and $x \in \calX$, and there exists $L_\mu < \infty$ such that $\|\nabla_\theta \mu_\theta(x, \varepsilon) - \nabla_\theta \mu_{\theta'}(x, \varepsilon)\| \le L_\mu \|\theta - \theta'\|$ almost surely.
\end{assumption}

\begin{corollary}[Bounded and Lipschitz mean-decision gradient]\label{cor:smooth-decision}
Under Assumptions~\ref{ass:bounded} and~\ref{ass:score-bound}, $\theta \mapsto  w^\star_\theta(x)$ is differentiable for every $x \in \calX$, with the score function representation
\begin{equation}\label{eq:F-jacobian}
  \nabla_\theta  w^\star_\theta(x)
  = \E_{\hat c\sim p_\theta(\cdot \mid x)}\!\bigl[w^\star(\hat c)\,\nabla_\theta \log p_\theta(\hat c\mid x)^\top\bigr].
\end{equation}
Moreover, the Jacobian satisfies the bound
\begin{equation}\label{eq:F-jacobian-bound}
  \|\nabla_\theta  w^\star_\theta(x)\| \;\le\; B_{\nabla w^\star_\theta} := B_{\calW}\, M, \qquad \forall \theta \in \R^{d_\theta},\ x \in \calX,
\end{equation}
and is $L_w$-Lipschitz in $\theta$ with $L_w := B_{\calW}\, M_\nabla$:
\begin{equation}\label{eq:F-jacobian-lip}
  \|\nabla_\theta  w^\star_\theta(x) - \nabla_\theta  w^\star_{\theta'}(x)\|
  \;\le\; L_w \, \|\theta - \theta'\|,
  \qquad \forall \theta, \theta' \in \R^{d_\theta},\ x \in \calX.
\end{equation}
\end{corollary}

\begin{proof}
Under Assumption~\ref{ass:score-bound}, dominated convergence applied to $ w^\star_\theta(x) = \int w^\star(\hat c)\, p_\theta(\hat c\mid x)\, d\hat c$ yields
\[
  \nabla_\theta  w^\star_\theta(x)
  = \int w^\star(\hat c)\,\nabla_\theta p_\theta(\hat c\mid x)\, d\hat c
  = \E_{\hat c\sim p_\theta(\cdot \mid x)}\!\bigl[w^\star(\hat c)\,\nabla_\theta \log p_\theta(\hat c\mid x)^\top\bigr],
\]
where the second equality uses the log-derivative identity $\nabla_\theta p_\theta = p_\theta\,\nabla_\theta \log p_\theta$.

\emph{Boundedness.} Applying the triangle inequality to~\eqref{eq:F-jacobian} in integral form, together with $\|w^\star(\hat c)\| \le B_{\calW}$ (Assumption~\ref{ass:bounded}) and~\eqref{eq:density-grad-bound},
\begin{align*}
  \|\nabla_\theta w^\star_\theta(x)\|
  \;\le\; \int \|w^\star(\hat c)\|\,\|\nabla_\theta p_\theta(\hat c\mid x)\|\,\mathrm{d}\hat c
  \;\le\; B_{\calW}\, \int \|\nabla_\theta p_\theta(\hat c\mid x)\|\,\mathrm{d}\hat c
  \;\le\; B_{\calW}\, M,
\end{align*}

\emph{Lipschitz continuity.} Subtracting the integral representations at $\theta$ and $\theta'$,
\[
  \nabla_\theta w^\star_\theta(x) - \nabla_\theta w^\star_{\theta'}(x)
  = \int w^\star(\hat c)\,\bigl(\nabla_\theta p_\theta(\hat c\mid x) - \nabla_{\theta} p_{\theta'}(\hat c\mid x)\bigr)^\top\,\mathrm{d}\hat c.
\]
Applying the triangle inequality, $\|w^\star(\hat c)\| \le B_{\calW}$, and~\eqref{eq:density-grad-lip} yields
\[
  \|\nabla_\theta w^\star_\theta(x) - \nabla_\theta w^\star_{\theta'}(x)\|
  \;\le\; B_{\calW}\,\int \|\nabla_\theta p_\theta(\hat c\mid x) - \nabla_{\theta} p_{\theta'}(\hat c\mid x)\|\,\mathrm{d}\hat c
  \;\le\; B_{\calW}\, M_\nabla\,\|\theta - \theta'\|.
\]
Setting $L_w := B_{\calW}\, M_\nabla$ yields~\eqref{eq:F-jacobian-lip}.
\end{proof}

\section{Online regret analysis of the nuisance estimator }\label{apdx:nuisance_analysis}
The nuisance update can be analyzed as an online convex optimization problem independently of the policy sequence $\{\theta_t\}_{t=1}^T$. Rather than tracking the moving per-iteration minimizer $\argmin_{\phi \in \Phi}\RNE_{\theta_t}(f_\phi)$, we compare against the fixed realizer $\phi^\star \in \Phi$ of Assumption~\ref{ass:realizable} below, a valid comparison since $f^\star$ is the universal minimizer of $\RNE_\theta$ for every $\theta$ simultaneously (Remark~\ref{rem:universal_min}). To analyze the nuisance update via online convex optimization arguments, we make the following assumptions, standard in the OCO literature~\citep{orabona2019modern}.

\begin{assumption}[Realizability]\label{ass:realizable}
There exists $\phi^\star \in \Phi$ such that $f_{\phi^\star} = f^\star$ (so $f^\star \in \calF:=\{f_\phi:\phi\in\Phi\}$ with an explicit parameter-space comparator).
\end{assumption}

\begin{assumption}\label{ass:nuisance}
\begin{enumerate}
  \item[(i)] $\Phi$ is convex and compact with diameter $D_{\Phi} := \sup_{\phi,\phi'\in\Phi}\|\phi-\phi'\| < \infty$.
  \item[(ii)] There exists $G_\RNE<\infty$ such that $\|\nabla_{\phi}\ell_{t}(\phi)\|\le G_\RNE$ almost surely for all $t$ and $\phi\in\Phi$.
  \item[(iii)] There exists $\lambda > 0$ such that, for every $\theta \in \R^{d_\theta}$, $\nabla^2_\phi \RNE_\theta(f_\phi)
    \;\succeq\; \lambda\, I_{d_\phi}.$
\end{enumerate}
\end{assumption}
Assumption~\ref{ass:nuisance}(i)--(ii) are standard online convex optimization assumptions, bounding the feasible set diameter and gradient magnitudes respectively. Assumption~\ref{ass:nuisance}(iii) requires the nuisance objective to be strongly convex in $\phi$, uniformly over $\theta$; We first restate Proposition~\ref{prop:nuisance-regret} for an arbitrary comparator $u\in\Phi$ and give its proof; the main-text statement is the case $u=\phi^\star$, for which $f_u=f^\star$.

\begin{proposition}[Nuisance regret]
Under Assumptions~\ref{ass:realizable}--\ref{ass:nuisance}, for any sequence $\{\theta_t\}_{t=1}^T$ and any fixed comparator $u\in\Phi$:
\begin{enumerate}
  \item[\textup{(i)}] With $\eta_t=D_\Phi/(G_\RNE\sqrt{T})$
    $$\frac{1}{T}\sum_{t=1}^T
      \E\!\bigl[\RNE_{\theta_t}(f_{\phi_t})
              -\RNE_{\theta_t}(f_{u})\bigr]
      \;\le\; \frac{D_\Phi G_\RNE}{\sqrt{T}}$$
  \item[\textup{(ii)}] With $\eta_t=1/(\lambda t)$
    $$\frac{1}{T}\sum_{t=1}^T
      \E\!\bigl[\RNE_{\theta_t}(f_{\phi_t})
              -\RNE_{\theta_t}(f_{u})\bigr]
      \;\le\; \frac{G_\RNE^2(1+\log T)}{2\lambda T}$$
\end{enumerate}
\end{proposition}

\begin{proof}
Both bounds follow from standard online convex optimization arguments \citep{orabona2019modern}; we include the derivations for completeness.
By non-expansiveness of the Euclidean projection $\Pi_\Phi$ and the update rule $\phi_{t+1}=\Pi_\Phi(\phi_t-\eta_t g_t)$, where $g_t:=\nabla_\phi\ell_t(\phi_t)$ is the stochastic gradient in Algorithm~\ref{alg:main},
\begin{equation}\label{eq:descent-ne}
  \|\phi_{t+1}-u\|^2
  \le \|\phi_t-\eta_t g_t-u\|^2
  = \|\phi_t-u\|^2
    - 2\eta_t\langle g_t,\phi_t-u\rangle
    + \eta_t^2\|g_t\|^2
\end{equation}
Taking expectations, using unbiasedness of $g_t$ and convexity of $\RNE_{\theta_t}$ in $\phi$,
$$\E\bigl[\langle\nabla_\phi\RNE_{\theta_t}(f_{\phi_t}),\phi_t-u\rangle\bigr]
  \ge
  \E\bigl[\RNE_{\theta_t}(f_{\phi_t}) - \RNE_{\theta_t}(f_{u})\bigr] $$
Substituting into~\eqref{eq:descent-ne} and rearranging,
$$2\eta_t\,\E\bigl[\RNE_{\theta_t}(f_{\phi_t})-\RNE_{\theta_t}(f_{u})\bigr]
  \le
  \E\bigl[\|\phi_t-u\|^2\bigr]
  - \E\bigl[\|\phi_{t+1}-u\|^2\bigr]
  + \eta_t^2 G_{\RNE}^2 $$
Summing over $t=1,\dots,T$, using a constant step-size $\eta_t=\eta$ and telescoping,
$$2\eta\sum_{t=1}^T
  \E\bigl[\RNE_{\theta_t}(f_{\phi_t})-\RNE_{\theta_t}(f_{u})\bigr]
  \le D_\Phi^2 + \eta^2 G_{\RNE}^2 T$$
Setting $\eta=D_\Phi/(G_{\RNE}\sqrt{T})$ yields the claimed bound.
The bound (ii) follows from standard strong-convexity arguments in online convex optimization.
When $\RNE_{\theta_t}$ is $\lambda$-strongly convex in $\phi$,
\begin{equation}\label{eq:sc-subgrad}
  \E\bigl[\RNE_{\theta_t}(f_{\phi_t})-\RNE_{\theta_t}(f_{u})\bigr]
  \le
  \E\bigl[\langle g_t,\phi_t-u\rangle\bigr]
  - \frac{\lambda}{2}\E\bigl[\|\phi_t-u\|^2\bigr]
\end{equation}
Non-expansiveness of $\Pi_\Phi$ gives
$\E[\langle g_t,\phi_t-u\rangle] \le
  \tfrac{1}{2\eta_t}(\E[\|\phi_t-u\|^2] - \E[\|\phi_{t+1}-u\|^2])
  + \tfrac{\eta_t G_{\RNE}^2}{2}$
Substituting into~\eqref{eq:sc-subgrad} and choosing $\eta_t = 1/(\lambda t)$,
\[\E\bigl[\RNE_{\theta_t}(f_{\phi_t})-\RNE_{\theta_t}(f_{u})\bigr]
  \le
  \frac{\lambda(t-1)}{2}\E\bigl[\|\phi_t-u\|^2\bigr]
  - \frac{\lambda t}{2}\E\bigl[\|\phi_{t+1}-u\|^2\bigr]
  + \frac{G_{\RNE}^2}{2\lambda t}\]
Summing telescopes the distance terms (with non-positive boundary contribution), giving
$\sum_{t=1}^T \E[\RNE_{\theta_t}(f_{\phi_t})-\RNE_{\theta_t}(f_{u})]
  \le \tfrac{G_{\RNE}^2(1+\log T)}{2\lambda}$
Dividing by $T$ completes the proof.
\end{proof}

\section{Convergence analysis: plug-in hybrid variant}\label{apdx:conv_analysis}
The nuisance regret translates into a bound on the gradient bias by combining strong convexity of $\RNE_\theta(\cdot)$ with the Lipschitz parameterization of $f_\phi(\cdot)$. We define the plug-in bias as
$$B(\phi,\theta) :=J(\phi,\theta)-\Vcost(\theta).$$
Under realizability (Assumption~\ref{ass:realizable}) $$\Vcost(\theta) = J(\phi^\star ,\theta)$$ so equivalently $$B(\phi,\theta) =J(\phi,\theta)-J(\phi^\star,\theta)$$
\begin{assumption}\label{ass:policy-class}
The nuisance estimator function class is bounded and $L_{f_\phi}$-Lipschitz, i.e., there exists $B_{f_\phi} < \infty$ such that $\|f_\phi(x)\| \le B_{f_\phi}$ and $\|f_\phi(x)-f_{\phi'}(x)\| \le L_{f_\phi}\,\|\phi-\phi'\|$ for all $\phi \in \Phi$ and $x \in \calX$.
\end{assumption}

\begin{lemma}[Regularity of $J$]\label{lem:J-reg}
Under Assumptions~\ref{ass:bounded},\ref{ass:policy-class} and Corollary~\ref{cor:smooth-decision},
the plug-in policy objective $J(\phi,\theta)$ satisfies:
\begin{enumerate}
  \item[(i)] $\|\nabla_\theta J(\phi,\theta)\| \le G^J
             := B_{f_\phi} B_{\nabla w^\star_\theta}$;
  \item[(ii)] $\|\nabla_\theta J(\phi,\theta) - \nabla_\theta J(\phi,\theta')\| \le L^J_\theta\|\theta-\theta'\|$
              with $L^J_\theta := B_{f_\phi} L_w = B_{f_\phi}\, B_\calW\, M_\nabla$;
  \item[(iii)] $|J(\phi,\theta)-J(\phi',\theta)| \le L^J_\phi\|\phi-\phi'\|$
               with $L^J_\phi := B_\calW L_{f_\phi}$.
\end{enumerate}
\end{lemma}

\begin{proof}
Part (i). Differentiating under the expectation, $\nabla_\theta J(\phi,\theta) = \E_x[(\nabla_\theta w^{\star}_{\theta}(x))^{\top}f_\phi(x)]$, so $\|\nabla_\theta J(\phi,\theta)\| \le \E_x[\|\nabla_\theta w^{\star}_{\theta}(x)\|\|f_\phi(x)\|] \le B_{\nabla w^\star_\theta} B_{f_\phi}$, where the last step uses Corollary~\ref{cor:smooth-decision} and Assumption~\ref{ass:policy-class}.

Part (ii). For any $\theta,\theta'\in \R^{d_\theta}$,
\begin{align*}
  \|\nabla_\theta J(\phi,\theta)-\nabla_\theta J(\phi,\theta')\|
  &\le \E_x\bigl[\|f_\phi(x)\|
      \|\nabla_\theta w^{\star}_{\theta}(x)
       - \nabla_\theta w^{\star}_{\theta'}(x)\|\bigr] \\
  &\le B_{f_\phi}\, L_w\,\|\theta-\theta'\|
\end{align*}
where the second inequality uses Assumption~\ref{ass:policy-class} and the Lipschitz property of $\theta\mapsto\nabla_\theta w^{\star}_{\theta}(x)$ established in Corollary~\ref{cor:smooth-decision}.

Part (iii). For any $\phi,\phi'\in\Phi$,
\begin{align*}
  |J(\phi,\theta)-J(\phi',\theta)| &\le \E_x\bigl[\|f_\phi(x)-f_{\phi'}(x)\|\|w^{\star}_{\theta}(x)\|\bigr] \le B_\calW\, L_{f_\phi}\,\|\phi-\phi'\|.
\end{align*}
\end{proof}

\begin{proposition}[Averaged bias gradient]\label{prop:bias-avg}
Under Assumptions~\ref{ass:realizable},\ref{ass:nuisance},\ref{ass:policy-class}, and Corollary~\ref{cor:smooth-decision},
\begin{enumerate}
  \item[\textup{(i)}] With $\eta_t = D_\Phi/(G_\RNE\sqrt{T})$,
    $$\frac{1}{T}\sum_{t=1}^{T}
      \E\!\bigl[\|\nabla_\theta B(\phi_t,\theta_t)\|^2\bigr]
      \;\le\;
       \frac{2 B_{\nabla w^\star_\theta}^2L_{f_\phi}^2 D_\Phi G_\RNE}{\lambda\sqrt{T}}
      $$
  \item[\textup{(ii)}] With $\eta_t = 1/(\lambda t)$,
    $$ \frac{1}{T}\sum_{t=1}^{T}
      \E\!\bigl[\|\nabla_\theta B(\phi_t,\theta_t)\|^2\bigr]
      \;\le\;  \frac{B_{\nabla w^\star_\theta}^2L_{f_\phi}^2G_\RNE^2(1+\log T)}{\lambda^2 T}$$
\end{enumerate}
\end{proposition}

\begin{proof}
Since $B(\phi,\theta)=J(\phi,\theta)-J(\phi^\star,\theta)$ by Assumption~\ref{ass:realizable}, then $$\nabla_\theta B(\phi_t,\theta_t)
  = \E_x\bigl[
      (\nabla_\theta w^{\star}_{\theta_t}(x))^{\top}
      (f_{\phi_t}(x)-f^\star(x))\bigr].$$
By Jensen's inequality, Cauchy-Schwarz, Corollary~\ref{cor:smooth-decision}, and the Lipschitz parameterization (Assumption~\ref{ass:policy-class})
\begin{align*}
    \|\nabla_\theta B(\phi_t,\theta_t)\|^2
  \le B_{\nabla w^\star_\theta}^2\,\E_x\bigl[\|f_{\phi_t}(x)-f^\star(x)\|^2\bigr]
  \le B_{\nabla w^\star_\theta}^2 L_{f_\phi}^2\,\|\phi_t - \phi^\star\|^2.
\end{align*}

Under Assumption~\ref{ass:realizable}, $f^\star$ minimizes $\RNE_\theta(f)$ for every $\theta$ (Remark~\ref{rem:universal_min}). Combined with $\lambda$-strong convexity (Assumption~\ref{ass:nuisance}(iii)), this yields the quadratic-growth inequality
\begin{equation}\label{eq:quad-growth}
  \tfrac{\lambda}{2}\,\|\phi_t - \phi^\star\|^2
  \le \RNE_{\theta_t}(f_{\phi_t}) - \RNE_{\theta_t}(f^\star) .
\end{equation}
Combining the previous bound with~\eqref{eq:quad-growth} yields
$$
  \|\nabla_\theta B(\phi_t,\theta_t)\|^2
  \le \frac{2 B_{\nabla w^\star_\theta}^2L_{f_\phi}^2}{\lambda}
       \bigl(\RNE_{\theta_t}(f_{\phi_t}) - \RNE_{\theta_t}(f^\star)\bigr).
$$
Averaging over $t = 1,\dots,T$ and applying Proposition~\ref{prop:nuisance-regret} yields the two stated bounds.
\end{proof}

The hybrid gradient is conditionally biased, but only through the nuisance error, as the following identity formalizes.

\begin{lemma}[Hybrid gradient bias]\label{lem:hybrid-bias}
Under Assumptions~\ref{ass:bounded},\ref{ass:score-bound},\ref{ass:realizable}, and the log-derivative identity for $g_t^{\mathrm{score}}$,
\begin{equation}\label{eq:hybrid-cond-mean}
  \E\!\bigl[g_t^\alpha \mid \phi_t,\theta_t\bigr]
  = \nabla_\theta\Vcost(\theta_t)
    + (1-\alpha)\,\nabla_\theta B(\phi_t,\theta_t)
\end{equation}
\end{lemma}

\begin{proof}
By Assumption~\ref{ass:score-bound} and dominated convergence, we may exchange differentiation and expectation: $$\nabla_\theta\Vcost(\theta_t)= \E_{x,c,\hat c}[c^\top w^\star(\hat c)\,\nabla_\theta \log p_{\theta_t}(\hat c \mid x)]= \E[g_t^{\mathrm{score}} \mid \phi_t,\theta_t],$$ where the second equality uses $w_t = w^\star(\hat c_t)$ and $y_t = c_t^\top w_t$.

For the plug-in estimator, since $w_t=w^\star(\hat c_t)$ and applying the log-derivative identity,
\begin{align*}
\E[g_t^{\mathrm{plug\text{-}in}} \mid \phi_t, \theta_t]
&= \E_{x,\hat c}\!\bigl[(w^\star(\hat c)^\top f_{\phi_t}(x))\,\nabla_\theta \log p_{\theta_t}(\hat c\mid x)\bigr]\\
&= \E_x\!\left[f_{\phi_t}(x)^\top \E_{\hat c}\!\bigl[w^\star(\hat c)\,\nabla_\theta \log p_{\theta_t}(\hat c\mid x)^\top\bigr]\right]\\
&= \E_x\!\left[f_{\phi_t}(x)^\top \nabla_\theta w^\star_{\theta_t}(x)\right] = \nabla_\theta J(\phi_t,\theta_t),
\end{align*}
using Corollary~\ref{cor:smooth-decision} for the third equality.
Combining the two by linearity of expectation yields the claim.
\end{proof}

\begin{assumption}[Gradient variances] \label{ass:gradient-bound}
There exist $\sigma_{\mathrm{score}}, \sigma_{\mathrm{plug\text{-}in}}< \infty$ such that: $$
  \mathrm{Var}\bigl[g_t^{\mathrm{score}} \,\big|\, \phi_t, \theta_t\bigr]
  \le \sigma_{\mathrm{score}}^2,
  \qquad
  \mathrm{Var}\bigl[g_t^{\mathrm{plug\text{-}in}} \,\big|\, \phi_t, \theta_t\bigr]
  \le \sigma_{\mathrm{plug\text{-}in}}^2.
$$
\end{assumption}

\begin{lemma}[One-Step Descent]\label{lem:one-step-alpha}
Under Assumptions~\ref{ass:realizable}--\ref{ass:gradient-bound}, the iterates of the policy update satisfy
\begin{equation}\label{eq:one-step-alpha}
\begin{aligned}
  \E\!\bigl[\Vcost(\theta_{t+1})\bigr]
  &\leq \E\!\bigl[\Vcost(\theta_t)\bigr]
  -\frac{\beta}{2} \E\!\bigl[\|\nabla_\theta\Vcost(\theta_t)\|^2\bigr] \\
  &\quad +\frac{\beta(1-\alpha)^2}{2}
   \E\!\bigl[\|\nabla_\theta B(\phi_t,\theta_t)\|^2\bigr]
  +\frac{L^J_\theta\beta^2}{2}((G^J)^2+\sigma_\alpha^2).
\end{aligned}
\end{equation}
with $\sigma_\alpha^2 := (\alpha\sigma_{\mathrm{score}}+(1-\alpha)\sigma_{\mathrm{plug\text{-}in}})^2$.
\end{lemma}

\begin{proof}[Proof of Lemma~\ref{lem:one-step-alpha}]
Since $\Vcost(\theta)=J(\phi^\star,\theta)$ is $L^J_\theta$-smooth in $\theta$ (Lemma~\ref{lem:J-reg}(ii)), the descent lemma applied to the update $\theta_{t+1}=\theta_t-\beta g_t^\alpha$ gives
\begin{equation}\label{eq:descent-theta}
  \Vcost(\theta_{t+1})
  \leq \Vcost(\theta_t)
    + \bigl\langle\nabla_\theta\Vcost(\theta_t),\,
                  \theta_{t+1}-\theta_t\bigr\rangle
    + \frac{L^J_\theta}{2}\|\theta_{t+1}-\theta_t\|^2
\end{equation}

Substituting $\theta_{t+1}-\theta_t=-\beta g_t^\alpha$ and taking conditional expectation given $(\phi_t,\theta_t)$,
$$\E\bigl[\langle\nabla_\theta\Vcost(\theta_t),\,
    -\beta g_t^\alpha\rangle\mid\phi_t,\theta_t\bigr]
  = -\beta\bigl\langle\nabla_\theta\Vcost(\theta_t),\,
    \E[g_t^\alpha\mid\phi_t,\theta_t]\bigr\rangle$$
Applying Lemma~\ref{lem:hybrid-bias}, $\E[g_t^\alpha\mid\phi_t,\theta_t] = \nabla_\theta\Vcost(\theta_t)+(1-\alpha)\nabla_\theta B(\phi_t,\theta_t)$, so
$$-\beta\bigl\langle\nabla_\theta\Vcost,\,
    \nabla_\theta\Vcost+(1-\alpha)\nabla_\theta B\bigr\rangle
  = -\beta\|\nabla_\theta\Vcost\|^2
    - \beta\langle\nabla_\theta\Vcost,\,(1-\alpha)\nabla_\theta B\rangle$$
Applying the inequality $-\langle a,b\rangle \leq \frac{1}{2}\|a\|^2+\frac{1}{2}\|b\|^2$ to the cross term,
\begin{equation}\label{eq:inner-bound}
  \E\bigl[\langle\nabla_\theta\Vcost,\,
           -\beta g_t^\alpha\rangle\mid\phi_t,\theta_t\bigr]
  \leq
  -\frac{\beta}{2}\|\nabla_\theta\Vcost(\theta_t)\|^2
  + \frac{\beta(1-\alpha)^2}{2}\|\nabla_\theta B(\phi_t,\theta_t)\|^2
\end{equation}

We bound $\E[\|g_t^\alpha\|^2\mid\phi_t,\theta_t]$ using the bias-variance decomposition
\begin{equation}\label{eq:bv-decomp}
  \E\bigl[\|g_t^\alpha\|^2\mid\phi_t,\theta_t\bigr]
  = \mathrm{Var}\bigl[g_t^\alpha\mid\phi_t,\theta_t\bigr]
    + \bigl\|\E\bigl[g_t^\alpha\mid\phi_t,\theta_t\bigr]\bigr\|^2
\end{equation}
For the variance term, write $g_t^\alpha = \alpha g_t^{\mathrm{score}} +(1-\alpha)g_t^{\mathrm{plug\text{-}in}}$. Expanding,
\begin{align*}
  \mathrm{Var}\bigl[g_t^\alpha\mid\phi_t,\theta_t\bigr]
  &= \alpha^2\mathrm{Var}\bigl[g_t^{\mathrm{score}}\mid\phi_t,\theta_t\bigr]
   + (1-\alpha)^2\mathrm{Var}\bigl[g_t^{\mathrm{plug\text{-}in}}\mid\phi_t,\theta_t\bigr] \\
  &\quad + 2\alpha(1-\alpha)
    \mathrm{Cov}\bigl[g_t^{\mathrm{score}},\, g_t^{\mathrm{plug\text{-}in}}\mid\phi_t,\theta_t\bigr]
\end{align*}
We bound the covariance term via Cauchy-Schwarz
$$\bigl|\mathrm{Cov}\bigl[g_t^{\mathrm{score}},\,
  g_t^{\mathrm{plug\text{-}in}}\mid\phi_t,\theta_t\bigr]\bigr|
  \leq \sigma_{\mathrm{score}}\,\sigma_{\mathrm{plug\text{-}in}}$$
where the second inequality uses Assumption~\ref{ass:gradient-bound}.
Substituting and completing the square,
\begin{equation}\label{eq:var-bound}
  \mathrm{Var}\bigl[g_t^\alpha\mid\phi_t,\theta_t\bigr]
  \leq \bigl(\alpha\sigma_{\mathrm{score}}+(1-\alpha)\sigma_{\mathrm{plug\text{-}in}}\bigr)^2
  =: \sigma_\alpha^2
\end{equation}
For the squared mean term, by Lemma~\ref{lem:hybrid-bias}, the triangle inequality, $\|\E[g_t^\alpha \mid \phi_t, \theta_t]\| \le \alpha \|\nabla_\theta \Vcost(\theta_t)\| + (1-\alpha) \|\nabla_\theta J(\phi_t, \theta_t)\| =\alpha \|\nabla_\theta J(\phi^\star, \theta_t)\| + (1-\alpha) \|\nabla_\theta J(\phi_t, \theta_t)\| \le G^J$ where the last inequality is given by Lemma~\ref{lem:J-reg}.
Substituting into~\eqref{eq:bv-decomp},
\begin{equation}\label{eq:quad-bound}
  \E\bigl[\|g_t^\alpha\|^2\mid\phi_t,\theta_t\bigr]
  \leq \sigma_\alpha^2 + (G^J)^2
\end{equation}

Substituting~\eqref{eq:inner-bound} and~\eqref{eq:quad-bound} into~\eqref{eq:descent-theta} and taking full expectations,
\begin{align*}
  \E\bigl[\Vcost(\theta_{t+1})\bigr]
  \leq \E\bigl[\Vcost(\theta_t)\bigr]
   & - \frac{\beta}{2}
       \E\bigl[\|\nabla_\theta\Vcost(\theta_t)\|^2\bigr]\\&
    + \frac{\beta(1-\alpha)^2}{2}
       \E\bigl[\|\nabla_\theta B(\phi_t,\theta_t)\|^2\bigr]
    + \frac{L^J_\theta\beta^2}{2}(\sigma_\alpha^2+(G^J)^2)
\end{align*}
which is the claimed inequality.
\end{proof}
We are now in a position to state Theorem~\ref{thm:main-alpha} with explicit constants and provide its proof. The result follows directly by combining the one-step descent inequality (Lemma~\ref{lem:one-step-alpha}) with the averaged bias bound of Proposition~\ref{prop:bias-avg}.
\begin{theorem}[Non-convex stationarity, plug-in variant]
Let $V^\star := \inf_\theta \Vcost(\theta) > -\infty$ Under Assumptions~\ref{ass:bounded}--\ref{ass:gradient-bound}, with step size
$\beta = 1/(L^J_\theta\sqrt{T})$, Algorithm~\ref{alg:main} satisfies,
\begin{equation}\label{eq:main-alpha-master}
\frac{1}{T}\sum_{t=1}^T\E\bigl[\|\nabla_\theta\Vcost(\theta_t)\|^2\bigr]
\;\le\;\frac{2L^J_\theta\Delta_1+(G^J)^2+\sigma_\alpha^2}{\sqrt{T}}
+ (1-\alpha)^2 B_{\nabla w^\star_\theta}^2 L_{f_\phi}^2 \cdot \mathcal{B}_{\mathrm{bias}}(T),
\end{equation}
where $\Delta_1 := \Vcost(\theta_1)-V^\star$, and the bias decay $\mathcal{B}_{\mathrm{bias}}(T)$ depends on the nuisance step size:
\begin{enumerate}
  \item[\textup{(i)}] With $\eta_t=D_\Phi/(G_\RNE\sqrt{T})$:\quad $\mathcal{B}_{\mathrm{bias}}(T) = \frac{2 D_\Phi G_\RNE}{\lambda \sqrt{T}}$.
  \item[\textup{(ii)}] With $\eta_t = 1/(\lambda t)$:\quad $\mathcal{B}_{\mathrm{bias}}(T) = \frac{G_\RNE^2(1+\log T)}{\lambda^2 T}$.
\end{enumerate}
\end{theorem}
\begin{proof}
Rearranging Lemma~\ref{lem:one-step-alpha},
\begin{align*}
\frac{\beta}{2}\,\E\bigl[\|\nabla_\theta\Vcost(\theta_t)\|^2\bigr]
  &\le \E\bigl[\Vcost(\theta_t)\bigr]-\E\bigl[\Vcost(\theta_{t+1})\bigr]
     \\&+\frac{\beta(1-\alpha)^2}{2}\E\bigl[\|\nabla_\theta B(\phi_t,\theta_t)\|^2\bigr]
     +\frac{L^J_\theta\beta^2}{2}((G^J)^2+\sigma_\alpha^2)
\end{align*}
Summing over $t=1,\dots,T$ and using
$\E[\Vcost(\theta_{T+1})]\ge V^\star$,
$$\frac{\beta}{2}\sum_{t=1}^T\E\bigl[\|\nabla_\theta\Vcost(\theta_t)\|^2\bigr]
  \le \Vcost(\theta_1)-V^\star
     +\frac{\beta(1-\alpha)^2}{2}\sum_{t=1}^T\E\bigl[\|\nabla_\theta B\|^2\bigr]
     +\frac{L^J_\theta\beta^2 T}{2}((G^J)^2+\sigma_\alpha^2)$$
Dividing by $\beta T/2$ and setting $\beta=1/(L^J_\theta\sqrt{T})$,
$$\frac{1}{T}\sum_{t=1}^T\E\bigl[\|\nabla_\theta\Vcost(\theta_t)\|^2\bigr]
  \le \frac{2L^J_\theta(\Vcost(\theta_1)-V^\star)+(G^J)^2+\sigma_\alpha^2}{\sqrt{T}}
     + (1-\alpha)^2\frac{1}{T}\sum_{t=1}^T\E\bigl[\|\nabla_\theta B\|^2\bigr] $$
Applying Proposition~\ref{prop:bias-avg}(i)-(ii) yields the results.
\end{proof}

\section{Convergence analysis: surrogate hybrid variant}\label{apdx:surrogate}
This appendix extends the hybrid gradient framework to the surrogate setting, where the plug-in policy gradient is replaced by a differentiable proxy. A \emph{surrogate cost function} $h^{\mathrm{sur}}:\R^{d_w}\times\R^{d_w}\to\R$ replaces the inner objective $c^\top w^\star(\hat c)$ with a smooth, globally defined surrogate.

For a given policy $\theta$, the plug-in surrogate objective is defined in Equation~\ref{eq:Jsur}
  $$J^{\mathrm{sur}}(\phi,\theta):= \E_x\!\left[ \E_{\hat c\sim p_\theta(\cdot\mid x)}\!\left[ h^{\mathrm{sur}}\!\left(f_\phi(x),\,\hat c\right) \right] \right]$$

Realizability (Assumption~\ref{ass:realizable}) implies $\Vcost(\theta) = J(\phi^\star,\theta)$ as an exact identity, reducing the problem to the plug-in objective $J(\phi,\theta)$ and its surrogate $J^{\mathrm{sur}}(\phi,\theta)$. Our analysis therefore takes $J^{\mathrm{sur}}(\phi^\star,\theta)$ directly as the surrogate target.

The surrogate nuisance bias is entirely captured by the surrogate analog:$$B^{\mathrm{sur}}(\phi,\theta) := J^{\mathrm{sur}}(\phi,\theta) - J^{\mathrm{sur}}(\phi^\star,\theta). $$The bias vanishes when $f_\phi = f^\star$, isolating the contribution of nuisance error. Whenever the surrogate is faithful and $\phi$ is close to $\phi^\star$, a good policy under $J^{\mathrm{sur}}(\phi,\cdot)$ remains good under $J(\phi,\cdot)$, and therefore under $\Vcost$.
\begin{assumption}[Surrogate regularity]\label{ass:surrogate}
The surrogate $h^{\mathrm{sur}}(c, \hat c)$ satisfies:
\begin{enumerate}
  \item[(i)] $\hat c\mapsto h^{\mathrm{sur}}(c, \hat c)$ is continuously differentiable
        for every $c \in \R^{d_w}$, and there exists $G_\mathrm{sur} < \infty$ such that
        $\|\nabla_{\hat c} h^{\mathrm{sur}}(c, \hat c)\| \le G_\mathrm{sur}$ almost surely.
    \item[(ii)] There exist $L_c^\mathrm{sur}, L_{\hat c}^\mathrm{sur} < \infty$ such that for all $c, c', \hat c, \hat c' \in \R^{d_w}$,
    \begin{align*}
        \|\nabla_{\hat c} h^{\mathrm{sur}}(c,\hat c) - \nabla_{\hat c} h^{\mathrm{sur}}(c',\hat c)\| &\le L_c^\mathrm{sur} \|c - c'\|, \\
        \|\nabla_{\hat c} h^{\mathrm{sur}}(c,\hat c) - \nabla_{\hat c} h^{\mathrm{sur}}(c,\hat c')\| &\le L_{\hat c}^\mathrm{sur} \|\hat c- \hat c'\|.
    \end{align*}
\end{enumerate}
\end{assumption}

To construct a tractable estimator of $\nabla_\theta J^{\mathrm{sur}}$ we assume that $p_\theta(\cdot \mid x)$ admits a reparameterization $\hat c= \mu_\theta(x, \varepsilon)$ with $\varepsilon \sim \nu$ independent of $\theta$ (Assumption~\ref{ass:reparam}).
At iteration $t$, having sampled $\varepsilon_t \sim \nu$ and computed $\hat c_t = \mu_{\theta_t}(x_t, \varepsilon_t)$, the chain rule yields the surrogate gradient estimator
$$ g^{\mathrm{sur}}_t :=
  \bigl(\nabla_\theta \mu_{\theta_t}(x_t, \varepsilon_t)\bigr)^\top
  \nabla_{\hat c} h^{\mathrm{sur}}\!\bigl(f_{\phi_t}(x_t), \hat c_t\bigr)$$

\begin{proposition}[Unbiasedness of $g^\mathrm{sur}_t$]
    Under Assumptions~\ref{ass:reparam} and~\ref{ass:surrogate},
    $$\E[g^{\mathrm{sur}}_t \mid \phi_t, \theta_t]=\nabla_\theta J^{\mathrm{sur}}(\phi_t, \theta_t)$$
\end{proposition}

\begin{proof}
By Assumption~\ref{ass:reparam}, $\hat c_t=\mu_{\theta_t}(x_t,\varepsilon_t)$
with $\varepsilon_t\sim\nu$ independent of $\theta_t$, so the reparameterization
trick and Assumptions~\ref{ass:surrogate}(i) and~\ref{ass:reparam} justify
exchanging differentiation and expectation:
\[
  \nabla_\theta J^{\mathrm{sur}}(\phi_t,\theta_t)
  = \E_{x,\varepsilon}\!\left[
      \bigl(\nabla_\theta\mu_{\theta_t}(x,\varepsilon)\bigr)^\top
      \nabla_{\hat c}h^{\mathrm{sur}}\!\bigl(f_{\phi_t}(x),\mu_{\theta_t}(x,\varepsilon)\bigr)
    \right]
  = \E\bigl[g^{\mathrm{sur}}_t\mid\phi_t,\theta_t\bigr]. \qedhere
\]
\end{proof}

The surrogate hybrid gradient estimator is
$ g^{\alpha,\mathrm{sur}}_t := \alpha\, g_t^{\mathrm{score}} + (1-\alpha)\, g^{\mathrm{sur}}_t.$

We operate under the same regime as the plug-in policy case. Define the target objective
\begin{equation}\label{eq:Valpha}
  V^\alpha(\theta) := \alpha\, \Vcost(\theta) + (1-\alpha)\, J^{\mathrm{sur}}(\phi^\star, \theta)
\end{equation}

\begin{assumption}[Surrogate gradient moments]\label{ass:sur-moments}
There exist $\sigma_{\mathrm{sur}}<\infty$ such that $\mathrm{Var}[g^{\mathrm{sur}}_t \mid \phi_t, \theta_t] \le \sigma_{\mathrm{sur}}^2$.
\end{assumption}
\begin{assumption}[Surrogate gradient norm bound]\label{ass:sur-grad-bound}
There exists $G_{\mathrm{sur}}^J < \infty$ such that
$\|\nabla_\theta J^{\mathrm{sur}}(\phi,\theta)\| \le G_{\mathrm{sur}}^J$
for all $\phi \in \Phi$, $\theta \in \R^{d_\theta}$.
Under Assumptions~\ref{ass:reparam}--\ref{ass:surrogate}, this holds with $G_{\mathrm{sur}}^J := G_\mu G_{\mathrm{sur}}$.
\end{assumption}
 Let $\tilde G_\alpha := \alpha G^J + (1-\alpha)G^J_{\mathrm{sur}}$. By Lemma~\ref{lem:J-reg}(i) and Assumption~\ref{ass:sur-grad-bound}, $\tilde G_\alpha$ bounds the norm of $\E[g^{\alpha,\mathrm{sur}}_t\mid\phi_t,\theta_t]$.
\begin{lemma}[Surrogate hybrid gradient bias]\label{lem:sur-bias}
Under Assumptions~\ref{ass:bounded},\ref{ass:score-bound},\ref{ass:reparam},\ref{ass:realizable},\ref{ass:surrogate} and the log-derivative identity for $g_t^{\mathrm{score}}$,
\begin{equation*}
  \E\bigl[g^{\alpha,\mathrm{sur}}_t \mid \phi_t, \theta_t\bigr]
  = \nabla_\theta V^\alpha(\theta_t) + (1-\alpha)\, \nabla_\theta B^{\mathrm{sur}}(\phi_t, \theta_t)
\end{equation*}
\end{lemma}

\begin{proof}
By the log-derivative identity, $\E[g_t^{\mathrm{score}}\mid\phi_t,\theta_t]=\nabla_\theta\Vcost(\theta_t)$. Unbiasedness of the reparameterization estimator gives $\E[g^{\mathrm{sur}}_t\mid\phi_t,\theta_t]=\nabla_\theta J^{\mathrm{sur}}(\phi_t,\theta_t) =\nabla_\theta J^{\mathrm{sur}}(\phi^\star, \theta_t)+\nabla_\theta B^{\mathrm{sur}}(\phi_t,\theta_t)$. Linearity of expectation yields the claim.
\end{proof}

\begin{lemma}[Smoothness of $V^\alpha$]\label{lem:Valpha-smooth}
Under Assumptions~\ref{ass:realizable},\ref{ass:surrogate} and~\ref{ass:reparam}, $J^{\mathrm{sur}}(\phi^\star,\cdot)$ is $L^{\mathrm{sur}}_\theta$-smooth  with $L^{\mathrm{sur}}_\theta := L_\mu G_{\mathrm{sur}} + G_\mu^2 L_{\hat c}^{\mathrm{sur}}$. Consequently $V^\alpha$ is $L^\alpha_\theta$-smooth with $L^\alpha_\theta := \alpha L^J_\theta + (1-\alpha) L^{\mathrm{sur}}_\theta$.
\end{lemma}

\begin{proof}
We show $J^{\mathrm{sur}}(\phi^\star, \cdot)$ is $L^{\mathrm{sur}}_\theta$-smooth,
the smoothness of $V^\alpha$ then follows because a convex combination of smooth
functions is smooth.
By Assumption~\ref{ass:reparam}, we may write
\[
\nabla_\theta J^{\mathrm{sur}}(\phi^\star,\theta)
= \E_{x,\varepsilon}\!\left[(\nabla_\theta \mu_\theta(x,\varepsilon))^\top \nabla_{\hat c} h^{\mathrm{sur}}(f^\star(x),\mu_\theta(x,\varepsilon))\right].
\]
For any $\theta,\theta'\in\R^{d_\theta}$, denote $\hat c = \mu_\theta(x,\varepsilon)$ and $\hat c' = \mu_{\theta'}(x,\varepsilon)$. Adding and subtracting $(\nabla_\theta\mu_{\theta'}(x,\varepsilon))^\top \nabla_{\hat c}h^{\mathrm{sur}}(f^\star(x),\hat c')$ and applying the triangle inequality,
\begin{align*}
&\|\nabla_\theta J^{\mathrm{sur}}(\phi^\star,\theta) - \nabla_\theta J^{\mathrm{sur}}(\phi^\star,\theta')\| \\
&\quad\le \E_{x,\varepsilon}\!\left[\|\nabla_\theta\mu_\theta(x,\varepsilon)-\nabla_\theta\mu_{\theta'}(x,\varepsilon)\|\,\|\nabla_{\hat c}h^{\mathrm{sur}}(f^\star(x),\hat c)\|\right] \\
&\qquad + \E_{x,\varepsilon}\!\left[\|\nabla_\theta\mu_{\theta'}(x,\varepsilon)\|\,\|\nabla_{\hat c}h^{\mathrm{sur}}(f^\star(x),\hat c)-\nabla_{\hat c}h^{\mathrm{sur}}(f^\star(x),\hat c')\|\right] \\
&\quad\le L_\mu G_{\mathrm{sur}}\,\|\theta-\theta'\| + G_\mu L_{\hat c}^{\mathrm{sur}} \,\E_{x,\varepsilon}\!\left[\|\hat c-\hat c'\|\right] \\
&\quad\le \bigl(L_\mu G_{\mathrm{sur}} + G_\mu^2 L_{\hat c}^{\mathrm{sur}}\bigr)\,\|\theta-\theta'\|,
\end{align*}
where the second inequality uses Assumption~\ref{ass:reparam} ($\|\nabla_\theta\mu\|\le G_\mu$, $L_\mu$-Lipschitz) and Assumption~\ref{ass:surrogate}(i)-(ii); the third inequality uses $\|\hat c-\hat c'\|=\|\mu_\theta(x,\varepsilon)-\mu_{\theta'}(x,\varepsilon)\|\le G_\mu\|\theta-\theta'\|$ from Assumption~\ref{ass:reparam}. This proves smoothness of $J^{\mathrm{sur}}(\phi^\star,\cdot)$. Since $\Vcost = J(\phi^\star,\cdot)$ is $L^J_\theta$-smooth (Lemma~\ref{lem:J-reg}(ii)) and $V^\alpha$ is a convex combination, $V^\alpha$ is $L^\alpha_\theta$-smooth with $L^\alpha_\theta := \alpha L^J_\theta + (1-\alpha)L^{\mathrm{sur}}_\theta$.
\end{proof}

\begin{proposition}[Averaged surrogate bias gradient]\label{prop:sur-bias-grad}
Under Assumptions~\ref{ass:realizable}--\ref{ass:nuisance},\ref{ass:surrogate}--\ref{ass:sur-moments}, with step size $\eta_t=1/(\lambda t)$ in the nuisance update,
$$\frac{1}{T}\sum_{t=1}^T \E\bigl[\|\nabla_\theta B^{\mathrm{sur}}(\phi_t, \theta_t)\|^2\bigr]
  \le \frac{{(L_c^\mathrm{sur})^2} G_\mu^2\,L_{f_\phi}^2 G_\RNE^2 (1 + \log T)}{\lambda^2 T}
  = \calO\!\left(\frac{\log T}{T}\right)$$
\end{proposition}

\begin{proof}
Differentiating under the expectation (justified by Assumptions~\ref{ass:surrogate}(i) and~\ref{ass:reparam}),
$$\nabla_\theta B^{\mathrm{sur}}(\phi_t, \theta_t)
  = \E_{x,\varepsilon}\!\Bigl[
      (\nabla_\theta \mu_{\theta_t}(x, \varepsilon))^\top
      \bigl(\nabla_{\hat c} h^{\mathrm{sur}}(f_{\phi_t}(x), \hat c)
      - \nabla_{\hat c} h^{\mathrm{sur}}(f^\star(x), \hat c)\bigr)\Bigr]$$
Jensen's inequality, Cauchy-Schwarz, $\|\nabla_\theta\mu\|\le G_\mu$, and Assumption~\ref{ass:surrogate}(ii) give
\begin{equation}\label{eq:sur-bias-bound1}
  \|\nabla_\theta B^{\mathrm{sur}}(\phi_t, \theta_t)\|^2
  \le(L_c^\mathrm{sur})^2\, G_\mu^2 \,\E_x\bigl[\|f_{\phi_t}(x) - f^\star(x)\|^2\bigr]
\end{equation}
By strong convexity
$\E_x[\|f_{\phi_t}(x)-f^\star(x)\|^2] \le \frac{2L_{f_\phi}^2}{\lambda}\,\E[\RNE_{\theta_t}(f_{\phi_t})-\RNE_{\theta_t}(f^\star)]$.
Averaging~\eqref{eq:sur-bias-bound1} over $t$ and applying Proposition~\ref{prop:nuisance-regret}(ii) yields the claim.
\end{proof}

\begin{lemma}[One-Step Descent, Surrogate Variant]\label{lem:sur-descent}
 Under Assumptions~\ref{ass:realizable}--\ref{ass:nuisance},\ref{ass:surrogate}--\ref{ass:sur-moments}, the iterates of the hybrid update $\theta_{t+1} = \theta_t - \beta g^{\alpha,\mathrm{sur}}_t$ satisfy
\begin{equation}\label{eq:sur-descent}
\begin{aligned}
\E[V^\alpha(\theta_{t+1})] - \E[V^\alpha(\theta_t)]
&\le -\frac{\beta}{2}\E\big[\|\nabla V^\alpha(\theta_t)\|^2\big] \\
&\quad +\frac{\beta(1-\alpha)^2}{2}
   \E\big[\|\nabla_\theta B^{\mathrm{sur}}(\phi_t, \theta_t)\|^2\big] \\
&\quad +\frac{L^\alpha_\theta \beta^2}{2}(\tilde G_\alpha^2 + \tilde\sigma_\alpha^2).
\end{aligned}
\end{equation}
with $\tilde\sigma_\alpha^2 := (\alpha \sigma_{\mathrm{score}} + (1-\alpha)\sigma_{\mathrm{sur}})^2$ and $\tilde G_\alpha := \alpha G^J + (1-\alpha) G^J_\mathrm{sur}$.
\end{lemma}
\begin{proof}
By Lemma~\ref{lem:Valpha-smooth}, $V^\alpha$ is $L^\alpha_\theta$-smooth, so the descent lemma yields
\begin{equation*}
V^\alpha(\theta_{t+1}) \le V^\alpha(\theta_t)
+ \langle \nabla V^\alpha(\theta_t), \theta_{t+1} - \theta_t\rangle
+ \frac{L^\alpha_\theta}{2}\|\theta_{t+1} - \theta_t\|^2
\end{equation*}
Substituting $\theta_{t+1} - \theta_t = -\beta g^{\alpha,\mathrm{sur}}_t$ and taking conditional expectation given $(\phi_t, \theta_t)$, Lemma~\ref{lem:sur-bias} gives
\begin{equation*}
\E[\langle \nabla V^\alpha, -\beta g^{\alpha,\mathrm{sur}}_t \rangle \mid \phi_t, \theta_t]
= -\beta \|\nabla V^\alpha(\theta_t)\|^2
- \beta\langle \nabla V^\alpha(\theta_t), (1-\alpha)\nabla_\theta B^{\mathrm{sur}}(\phi_t, \theta_t)\rangle
\end{equation*}
Applying Young's inequality to the cross term,
\begin{equation}\label{eq:sur-cross-term}
\E[\langle \nabla V^\alpha, -\beta g^{\alpha,\mathrm{sur}}_t \rangle \mid \phi_t, \theta_t]
\le -\frac{\beta}{2}\|\nabla V^\alpha(\theta_t)\|^2
+ \frac{\beta(1-\alpha)^2}{2}\|\nabla_\theta B^{\mathrm{sur}}(\phi_t, \theta_t)\|^2
\end{equation}
For the quadratic term, the bias-variance decomposition gives
$\E[\|g^{\alpha,\mathrm{sur}}_t\|^2 \mid \phi_t, \theta_t]
= \mathrm{Var}[g^{\alpha,\mathrm{sur}}_t \mid \phi_t, \theta_t]
+ \|\E[g^{\alpha,\mathrm{sur}}_t \mid \phi_t, \theta_t]\|^2$.
By the same Cauchy-Schwarz argument as in the proof of Lemma~\ref{lem:one-step-alpha} (applied with $\sigma_{\mathrm{sur}}$ in place of $\sigma_{\mathrm{plug\text{-}in}}$), and linearity of expectation
\begin{equation*}
\mathrm{Var}[g^{\alpha,\mathrm{sur}}_t \mid \phi_t, \theta_t]
\le (\alpha \sigma_{\mathrm{score}} + (1-\alpha)\sigma_{\mathrm{sur}})^2 = \tilde\sigma_\alpha^2,
\quad
\|\E[g^{\alpha,\mathrm{sur}}_t \mid \phi_t, \theta_t]\|^2
\le \tilde G_\alpha^2
\end{equation*}
with $\tilde G_\alpha := \alpha G^J + (1-\alpha) G^J_\mathrm{sur}$.
Substituting into the descent lemma and taking full expectations gives
\eqref{eq:sur-descent}.
\end{proof}

We restate Theorem~\ref{thm:sur-main} with explicit constants. The result follows by the same argument as Theorem~\ref{thm:main-alpha}, replacing Lemma~\ref{lem:one-step-alpha} with Lemma~\ref{lem:sur-descent} and Proposition~\ref{prop:bias-avg} with Proposition~\ref{prop:sur-bias-grad}.

\begin{theorem}[Non-convex stationarity, surrogate variant]
Let $V^{\alpha,\star} := \inf_\theta V^\alpha(\theta)$ where $V^\alpha(\theta) := \alpha\,\Vcost(\theta) + (1-\alpha)\,J^\mathrm{sur}(\phi^\star,\theta)$. Running Algorithm~\ref{alg:main} with $\beta = 1/(L^\alpha_\theta \sqrt{T})$ and $\eta_t = 1/(\lambda t)$ yields
\begin{equation}\label{eq:sur-main-bound-explicit}
\frac{1}{T}\sum_{t=1}^T \E\bigl[\|\nabla V^\alpha(\theta_t)\|^2\bigr]
\;\leq\; \frac{2 L^\alpha_\theta \Delta_1^\alpha
+ \tilde G_\alpha^2 + \tilde\sigma_\alpha^2}{\sqrt{T}}
+ \frac{(1-\alpha)^2{(L_c^\mathrm{sur})^2} G^2_\mu\,L_{f_\phi}^2 G_\RNE^2 (1 + \log T)}{\lambda^2 T}
\end{equation}
with $\Delta_1^\alpha = V^\alpha(\theta_1) - V^{\alpha,\star} $, $\tilde\sigma_\alpha^2 := (\alpha \sigma_{\mathrm{score}} + (1-\alpha)\sigma_{\mathrm{sur}})^2$ and $\tilde G_\alpha := \alpha G^J + (1-\alpha) G_{\mathrm{sur}}^J$.
\end{theorem}

\begin{proof}[Proof of Theorem~\ref{thm:sur-main}]
Following the same proof as Theorem~\ref{thm:main-alpha}, rearranging Lemma~\ref{lem:sur-descent}, summing over $t=1,\dots,T$, telescoping, and using $\E[V^\alpha(\theta_{T+1})]\ge V^{\alpha,\star}$,
$$ \frac{\beta}{2}\sum_{t=1}^T \E\bigl[\|\nabla V^\alpha(\theta_t)\|^2\bigr]
  \le V^\alpha(\theta_1) - V^{\alpha,\star}
     +\frac{\beta(1-\alpha)^2}{2}\sum_{t=1}^T \E\bigl[\|\nabla_\theta B^{\mathrm{sur}}\|^2\bigr]
     +\frac{L^\alpha_\theta \beta^2 T}{2}(\tilde G_\alpha^2+\tilde\sigma_\alpha^2) $$
Dividing by $\beta T/2$, substituting $\beta=1/(L^\alpha_\theta\sqrt{T})$, and applying Proposition~\ref{prop:sur-bias-grad} yields~\eqref{eq:sur-main-bound-explicit}.
\end{proof}

\section{Benchmark data generation}
\label{app:synthetic_datagen}

\subsection{Top-$k$ selection}
The default top-$k$ configuration has \(d=15\) items, cardinality \(k=2\), and context dimension \(d_x=5\). Contexts are sampled i.i.d. according to \(x_t \sim \mathcal{N}(0,I_{d_x})\). At environment construction time, the generator draws and fixes a binary coefficient matrix \(\omega_{\mathrm{top}}\in\{0,1\}^{d\times d_x}\) with i.i.d. Bernoulli\((0.5)\) entries. Let \(\deg\) denote the polynomial degree parameter and \(\bar\xi\) the multiplicative-noise radius. At iteration \(t\), it forms a latent vector \(u_t\in\mathbb{R}^d\) as \(u_t = [1 + (1 + \omega_{\mathrm{top}} x_t/\sqrt{d_x})^{\deg}] \odot \xi_t\), where \(\xi_{t,j}\sim \mathrm{Unif}[1-\bar\xi,\,1+\bar\xi]\) when multiplicative noise is enabled, while \(\xi_{t,j}=1\) otherwise. The realized cost vector is then \(c_t = u_t\). Because $x_t$ is Gaussian, these costs, and the shortest-path costs below, have unbounded support, whereas the analysis assumes bounded costs (Assumption~\ref{ass:bounded}).

\subsection{Shortest path}
The default shortest-path configuration has context dimension \(d_x=10\) and \(q=40\) edges over the fixed \(5\times 5\) right/down DAG described in Section~\ref{sec:synthetic_shortest_path}. Contexts are sampled i.i.d. as \(x_t \sim \mathcal{N}(0,I_{d_x})\). The generator draws an independent binary coefficient matrix \(\omega_{\mathrm{sp}}\in\{0,1\}^{q\times d_x}\) and forms \(u_t\in\mathbb{R}^q\) by the same polynomial construction as in the top-$k$ instance, with \(\omega_{\mathrm{sp}}\) replacing \(\omega_{\mathrm{top}}\). The generated vector is used directly as the realized edge-cost vector, \(c_t = u_t\).

\subsection{Combinatorial pricing}
The default pricing configuration has \(n=20\) products, \(K_{\mathrm{price}}=4\) normalized price levels \(\tilde\pi_\ell\) equally spaced in $[0.1, 0.9]$, promotion threshold \(L_{\mathrm{promo}}=2\), budget \(B_{\mathrm{promo}}=7\), and contexts \(x_t \sim \mathcal N(0,I_{d_x})\) with \(d_x=25\). The generator fixes unit directions \(b_i\) for each product and \(b_0\) for the market. Demand follows a logit purchase model with willingness to pay \(v_i(x) = 0.65 + 0.5\tanh\bigl(0.75\,[\,0.4\, g_{\deg}(b_0^\top x) + 0.6\, g_{\deg}(b_i^\top x)\,]\bigr)\), where \(g_{\deg}(z)=\operatorname{sign}(z)(|z|/0.674)^{\deg}\) keeps its typical size fixed as the degree changes. The mean demand for product \(i\) at level \(\ell\) is \(\mu_{i\ell}(x) = 3\lambda_\ell\,\sigma\bigl((v_i(x)-\tilde\pi_\ell)/0.1\bigr)\), with \(\sigma\) the logistic function and a promotion lift \(\lambda_\ell=1.25\) on the promotional levels (\(\lambda_\ell=1\) otherwise). Demand \(D_{i\ell,t}\) is Poisson with this mean times a common lognormal shock of scale \(0.1\), the cost vector has entries \(c_{i\ell,t}=-\tilde\pi_\ell D_{i\ell,t}\), and the bandit feedback is the realized cost \(v_t = c_t^\top w_t\). The paper uses \(\deg\in\{2,3,4,6\}\), with default \(\deg=6\).

\subsection{Energy scheduling}
\label{app:energy_scheduling_results}
This benchmark follows the decision-focused learning benchmark suite of \citet{Mandi_2024} and uses empirical day-level SEMO price data. The context is a \(48\times 8\) slot-feature matrix, and \(\mathcal{J}\), \(\mathcal{M}\), \(\mathcal{U}\), \(\mathcal{T}_\mathrm{slot}=\{1,\dots,48\}\) index tasks, machines, resources, and slots. With per-task power \(P_j\) and slot length \(\delta_{\mathrm{slot}}\), the iteration-$t$ cost vector $c_t\in\mathbb{R}^{d_w}$ of Section~\ref{sec:synthetic_energy} has components
\(
  c_{jms,t} \;:=\; P_j\,\delta_{\mathrm{slot}}\!\!\sum_{s'=s}^{s+d_j-1}\!\pi^{\mathrm{el}}_{t,s'},
\)
giving the energy cost of task $j$ when started on machine $m$ at slot $s$ under price profile $\pi^{\mathrm{el}}_t$, where $\pi^{\mathrm{el}}_{t,s'}$ is the electricity price in slot $s'$ of day $t$. The scheduling instance is fixed within a run, with feasible start-slot sets \(\mathcal{T}_j=\{s: a_j\le s,\ s+d_j-1\le b_j\}\) given by per-task earliest start \(a_j\) and latest end \(b_j\); remaining instance specifics (covariate definitions, exact task counts, and resource limits) follow \citet{Mandi_2024}. This benchmark uses chronological train/test splits over days rather than freshly sampled Gaussian contexts. Because prices are observed data, the mean cost $f^\star(x_t)$ in the regret is the realized cost vector $c_t$ on this benchmark.

\subsection{Feedback modes}
\paragraph{Feedback modes.}
All four benchmarks support pure bandit, semi-bandit, and full-information feedback in the cost-vector representation $c_t \in \mathbb{R}^{d_w}$ that matches Assumption~\ref{ass:feedback}. The three feedback modes share the same $(H, e, d_v)$ across benchmarks: pure bandit reveals the scalar realized cost $v_t = c_t^\top w_t$ via $H(w) = w^\top$, $e(w) = 1$, $d_v = 1$; semi-bandit reveals $v_t = \operatorname{diag}(w_t)\,c_t$, the cost contribution of each action-selected coordinate, via $H(w) = \operatorname{diag}(w)$, $e(w) = \mathbf{1}$, $d_v = d_w$; full information reveals $v_t = c_t$ via $H(w) = I$, $e(w) = w$, $d_v = d_w$. What differs across benchmarks is only what each coordinate of $c_t$ represents: per-item costs (top-$k$), per-edge costs (shortest path), per-(product, price-level) realized demand contributions (pricing), and per-(task, machine, start-slot) energy costs (energy). On the energy benchmark, semi-bandit feedback exposes these per-(task, machine, start-slot) entries of the lifted $c_t$, not the original 48-slot price profile. Under full information, \textsc{DFHPG} plugs the observed cost vector $c_t$ directly into the plug-in gradient in place of the nuisance estimate $f_{\phi_t}(x_t)$, and the nuisance model is not used.

\section{Implementation details of \textsc{DFHPG}}
\label{app:dfhpg_details}

This appendix specifies the components of the reported \textsc{DFHPG} configurations: the variance-reduction baselines, the mixing-weight schedule, the surrogate losses, and the generative cost models.

\subsection{Score function variance-reduction baselines}
\label{app:advantage_choices}

The score function component in Algorithm~\ref{alg:main} uses the realized scalar cost $y_t = e(w_t)^\top v_t = c_t^\top w_t$ as the weight on $\nabla_\theta \log p_{\theta_t}(\hat c_t\mid x_t)$. In experiments we replace $y_t$ with $y_t - b_t$ for some baseline $b_t$,
\[
  g_t^{\mathrm{score}}
  =
  (y_t - b_t)\,\nabla_\theta\log p_{\theta_t}(\hat c_t\mid x_t),
\]
which preserves unbiasedness whenever $b_t$ is independent of $\hat c_t$ given $x_t$ and typically reduces variance. We consider two baselines.

\paragraph{Moving-average baseline.}
A scalar running average of past realized costs $y_t$. Initialize $\bar y_1 = 0$ and update
\[
  b_t^{\mathrm{ma}} = \bar y_t,
  \qquad
  \bar y_{t+1} = \psi\,\bar y_t + (1-\psi)\,y_t,
\]
with default momentum $\psi = 0.95$. This baseline adds no oracle call but ignores the context $x_t$.

\paragraph{Nuisance-induced baseline.}
A context-dependent baseline computed from the nuisance estimate. Let
\[
  w_t^\phi
  := w^\star(f_{\phi_t}(x_t))
  \in \argmin_{w \in \calW} f_{\phi_t}(x_t)^\top w
\]
denote the nuisance-induced optimal decision, and define
\[
  b_t^{\mathrm{ni}} = f_{\phi_t}(x_t)^\top w_t^\phi,
\]
the nuisance's own predicted optimum at $x_t$. When the implemented decision $w_t$ matches $w_t^\phi$, the gradient weight $y_t - b_t^{\mathrm{ni}}$ collapses to the residual $y_t - f_{\phi_t}(x_t)^\top w_t$. When $w_t$ is worse than $w_t^\phi$ under the nuisance, the nuisance-predicted part $f_{\phi_t}(x_t)^\top w_t - b_t^{\mathrm{ni}}$ of the weight is nonnegative, so on average the gradient descent step reduces the probability of generating $\hat c_t$. This choice uses the context but adds one oracle solve per iteration. Since $b_t^{\mathrm{ni}}$ depends only on $x_t$ and past data, a miscalibrated nuisance affects the variance of the score function gradient but not its unbiasedness.

\subsection{Adaptive mixing-weight schedule}
\label{app:alpha_schedule}

Algorithm~\ref{alg:main} treats the mixing parameter $\alpha$ as a fixed scalar; in our experiments we instead use a time-varying $\alpha_t$, applied in the hybrid gradient $g_t^{\alpha_t} = \alpha_t\,g_t^{\mathrm{score}} + (1-\alpha_t)\,g_t^{\mathrm{plug\text{-}in}}$ (and analogously with $g_t^{\mathrm{sur}}$ in the surrogate variant). The schedule decays $\alpha_t$ from $\alpha_{\max}$ toward $\alpha_{\min}$, weighting the unbiased score function gradient more heavily early on (when the nuisance estimate is unreliable) and the lower-variance plug-in gradient later (as the nuisance estimate stabilizes). We use two variants that differ only in how the decay is gated. Both gates are specified in absolute iterations rather than as a fraction of $T$ to avoid leaking the horizon into the schedule.

\paragraph{Fixed-warmup variant.}
For the first $T_{\mathrm{warm}}$ iterations (default $T_{\mathrm{warm}} = 100$), the target is held at $\alpha_t^\star = \alpha_{\max}$; afterwards $\alpha_t^\star = \alpha_{\min}$.

\paragraph{Exponential-gate variant.}
Used by the energy configurations and the pricing CNF configuration. The target blends $\alpha_{\max}$ and $\alpha_{\min}$ via an exponentially decaying weight,
\[
    \gamma_t = \exp\!\left(-(t-1)/\tau_{\mathrm{sched}}\right), \qquad
    \alpha_t^\star = \gamma_t\,\alpha_{\max} + (1-\gamma_t)\,\alpha_{\min},
\]
where $\tau_{\mathrm{sched}}$ controls the decay timescale (smaller $\tau_{\mathrm{sched}}$ moves to the $\alpha_{\min}$ regime faster).

\paragraph{Smoothed update.}
In both variants, the actual mixing weight applies smoothing with rate $\eta_\alpha$ (default $0.05$) and re-clips,
\[
    \alpha_t = \mathrm{clip}\!\left((1-\eta_\alpha)\,\alpha_{t-1} + \eta_\alpha\,\alpha_t^\star,\;\alpha_{\min},\;\alpha_{\max}\right).
\]

\paragraph{Remark.}
Both variants use $\alpha_{\min}$ as a fixed post-warmup target. More adaptive heuristics are straightforward to substitute: for example, $\alpha_{\min}$ can be replaced by a reliability-driven target that responds to the running prediction error of the nuisance estimate (so that $\alpha_t$ rises again whenever the nuisance loses accuracy).

\subsection{Surrogate-loss definitions}
\label{app:surrogate_definitions}

We follow the gradient form of each surrogate from the reference implementations of PyEPO~\citep{tang2024pyepo}; pool-based variants follow the cumulative-pool formulation of the PredOpt-Benchmarks suite of \citet{Mandi_2024}. Throughout, $\hat c$ is the predicted cost vector, $c$ the target cost (the true cost in offline DFL; the nuisance estimate $f_\phi(x)$ in our hybrid plug-in update), $w^\star(\cdot) \in \arg\min_{w \in \calW}(\cdot)^\top w$ the linear-optimization oracle, $w^\star_{\mathrm{ptb}}(c) := \E_\xi[w^\star(c + \sigma\xi)]$ a perturbed-oracle expectation with $\xi \sim \mathcal{N}(0, I)$ and perturbation scale $\sigma>0$, and $\calL_{\mathrm{task}}(w; c) = c^\top w$ the regret-style task loss. In the tables, $\beta_{\mathrm{mse}}\in[0,1]$ is the weight on the mean-squared-error (MSE) term, $\lambda_{\mathrm{ptb}}>0$ is the interpolation step of DBB and IMLE, and $M$ is the number of perturbation samples in DPO. DPO, NCE, MAP, LTR, and EMA stand for differentiable perturbed optimizer, noise-contrastive estimation, maximum a posteriori, learning to rank, and exponential moving average. All subgradients below are stated with respect to $\hat c$ in the cost-minimization convention; the parameter update applies $\hat c \leftarrow \hat c - \eta\,\nabla_{\hat c}h^{\mathrm{sur}}$.

\paragraph{Oracle-call surrogates.} These call $w^\star$ once or twice per gradient evaluation.
\begin{center}
\small
\resizebox{\textwidth}{!}{%
\begin{tabular}{ll}
\toprule
Loss & Subgradient $\nabla_{\hat c}\,h^{\mathrm{sur}}(c, \hat c)$ \\
\midrule
SPO+~\citep{elmachtoub2022smart} & $2\bigl(w^\star(c) - w^\star(2\hat c - c)\bigr)$ \\
MSE+SPO+ & $\beta_{\mathrm{mse}}(\hat c - c) + (1-\beta_{\mathrm{mse}})\,\nabla_{\hat c}h_{\mathrm{SPO+}}$ \\
DBB~\citep{vlastelica2020differentiation} & $\bigl(w^\star(\hat c + \lambda_{\mathrm{ptb}} c) - w^\star(\hat c)\bigr)/\lambda_{\mathrm{ptb}}$ \\
IMLE~\citep{niepert2021implicit} & $\bigl(w^\star_{\mathrm{ptb}}(\hat c + \lambda_{\mathrm{ptb}} c) - w^\star_{\mathrm{ptb}}(\hat c)\bigr)/\lambda_{\mathrm{ptb}}$ \\
AIMLE~\citep{minervini2023adaptive} & IMLE with adaptive $\lambda_{\mathrm{ptb}} = \kappa_{\mathrm{aim}}\,\|\hat c\|/\|c\|$, $\kappa_{\mathrm{aim}}$ updated by EMA on grad sparsity \\
NID~\citep{sahoo2023backpropagation} & $-c$ \quad (negative-identity straight-through) \\
PFYL~\citep{berthet2020learning} & $w^\star(c) - w^\star_{\mathrm{ptb}}(\hat c)$ \\
DPO~\citep{berthet2020learning} & $\frac{1}{M\sigma}\sum_{i=1}^{M} \xi_i\,\bigl\langle w^\star(\hat c + \sigma\xi_i),\, c\bigr\rangle$,\; $\xi_i \sim \mathcal{N}(0, I)$ \\
PG~\citep{gupta2024decision} & $\bigl(w^\star(\hat c) - w^\star(\hat c - \sigma c)\bigr)/\sigma$ \\
\bottomrule
\end{tabular}%
}
\end{center}

\paragraph{Pool-based surrogates.} Let $\mathcal{P}_t$ denote the cumulative pool of feasible solutions at iteration $t$, grown online by appending $w^\star(\hat c_t)$ and the solution $w^\star(c)$ for the target cost $c$ at each step. Let $w_b := \arg\min_{w \in \mathcal{P}_t} c^\top w$ denote the pool's best solution under the target cost and $\mathcal{P}_t^- := \mathcal{P}_t \setminus \{w_b\}$ the rest of the pool; let $\delta > 0$ be the LTR margin and $\tau_{\mathrm{lw}} > 0$ the listwise temperature.
\begin{center}
\small
\begin{tabular}{ll}
\toprule
Loss & Subgradient $\nabla_{\hat c}\,h^{\mathrm{sur}}(c, \hat c)$ \\
\midrule
NCE~\citep{mulamba2021contrastive} & $w^\star(c) - \tfrac{1}{|\mathcal{P}_t|}\sum_{w \in \mathcal{P}_t} w$ \\
contrastiveMAP~\citep{mulamba2021contrastive} & $w^\star(c) - \arg\min_{w \in \mathcal{P}_t}\hat c^\top w$ \\
pointwiseLTR~\citep{mandi2022decision} & $\tfrac{2}{|\mathcal{P}_t|}\sum_{w \in \mathcal{P}_t}\bigl((\hat c - c)^\top w\bigr)w$ \\
pairwiseDiff~\citep{mandi2022decision} & $\tfrac{2}{|\mathcal{P}_t^-|}\sum_{r \in \mathcal{P}_t^-}\Delta_{br}(w_b - w_r)$,\; $\Delta_{br} = (\hat c - c)^\top(w_b - w_r)$ \\
pairwiseLTR~\citep{mandi2022decision} & $\tfrac{1}{|\mathcal{P}_t^-|}\sum_{r \in \mathcal{P}_t^- \,:\, \hat c^\top w_b > \hat c^\top w_r - \delta}(w_b - w_r)$ \quad (max-margin hinge) \\
listwiseLTR~\citep{mandi2022decision} & $\tfrac{1}{\tau_{\mathrm{lw}}}\bigl(\sigma_{\mathrm{soft}}(-c^\top\!\mathcal{P}_t/\tau_{\mathrm{lw}}) - \sigma_{\mathrm{soft}}(-\hat c^\top\!\mathcal{P}_t/\tau_{\mathrm{lw}})\bigr)\,\mathcal{P}_t$ \\
NCE\_c~\citep{mulamba2021contrastive} & Same subgradient as NCE; cached/centered pool variant \\
MAP\_c (MAP-C)~\citep{mulamba2021contrastive} & $w^\star(c) - \arg\min_{w \in \mathcal{P}_t}(\hat c-c)^\top w$ \quad (cached MAP variant) \\
SPOCaching~\citep{mulamba2021contrastive} & $w^\star(c) - \arg\min_{w \in \mathcal{P}_t}(2\hat c - c)^\top w$ \\
\bottomrule
\end{tabular}
\end{center}
Here $\sigma_{\mathrm{soft}}$ denotes softmax and $\hat c^\top\!\mathcal{P}_t \in \R^{|\mathcal{P}_t|}$ stacks the predicted-cost utilities of pool entries. The negative signs on the softmax arguments convert cost minimization to score maximization for the listwise ranking step. The heatmaps use implementation labels: SPOPlus and MSE+SPOPlus correspond to the SPO+ and MSE+SPO+ rows above, and MAP-C is the selected-configuration spelling of MAP\_c.

\paragraph{Use in the hybrid update.} With the reparameterization $\hat c = \mu_\theta(x,\varepsilon)$ of Assumption~\ref{ass:reparam}, any surrogate above enters the policy parameters via the chain rule
$g_t^{\mathrm{sur}} = (\nabla_\theta\mu_{\theta_t}(x_t,\varepsilon_t))^\top\,\nabla_{\hat c}h^{\mathrm{sur}}(f_{\phi_t}(x_t), \hat c_t)$,
i.e., the nuisance estimate $f_\phi(x)$ supplies the target cost $c$ in each formula above. Theorem~\ref{thm:sur-main} covers the convergence of the resulting hybrid update for any surrogate whose subgradient satisfies the regularity bounds of Assumption~\ref{ass:surrogate}.

\subsection{Generative parameterizations}
\label{app:generative_param}

The CNF variant uses a two-layer conditional RealNVP flow~\citep{dinh2017density} with $128$-dimensional coupling networks; the DDPM variant uses a conditional denoiser~\citep{ho2020denoising} with a $128$-dimensional hidden layer.

\paragraph{Conditional normalizing-flow score function.}
For a conditional normalizing flow with tractable density $p_\theta(\cdot\mid x)$, the score function gradient uses the exact log-density,
\[
  g_{t,\mathrm{CNF}}^{\mathrm{score}}
  =
  y_t\,\nabla_\theta\log p_{\theta_t}\!\left(\hat c_t\mid x_t\right),
\]

\paragraph{Conditional diffusion score function.}
For a conditional diffusion model $p_\theta(\cdot\mid x)$, the log-density is intractable, so the score function gradient is approximated using an evidence lower bound (ELBO) proxy $\mathrm{ELBO}_\theta(c,x)$ for $\log p_\theta(c\mid x)$~\citep{ho2020denoising},
\[
  g_{t,\mathrm{diff}}^{\mathrm{score}}
  =
  y_t\,\nabla_\theta\mathrm{ELBO}_{\theta_t}\!\left(\hat c_t,x_t\right),
\]
Because $\mathrm{ELBO}_\theta$ is a biased proxy for $\log p_\theta$, the diffusion variant falls outside Theorem~\ref{thm:sur-main}'s coverage and is reported as an empirical extension; the pricing and energy diffusion configurations in Appendix~\ref{app:selected_hyperparameters} use $\alpha_{\max} = \alpha_{\min} = 0$, and the top-$k$ and shortest-path configurations keep a score function component.

\paragraph{Plug-in/surrogate-gradient variants for generative cost models.}
Both generative parameterizations admit the same plug-in/surrogate-gradient construction. At iteration $t$, sample $K$ cost scenarios from the policy,
\[
  \hat c_t^{(1)},\ldots,\hat c_t^{(K)} \sim p_{\theta_t}(\cdot\mid x_t),
  \qquad
  \bar c_t := \frac{1}{K}\sum_{k=1}^K \hat c_t^{(k)},
\]
and let $g_{\mathrm{loss}}\bigl(\hat c, f_\phi(x)\bigr)$ denote the analytical plug-in/surrogate gradient at predicted cost $\hat c$ and nuisance estimate $f_\phi(x)$. We use two variants, defined as surrogate losses whose gradient with respect to $\theta$ flows only through the reparameterized samples $\hat c_t^{(k)}(\theta)$ and $\bar c_t(\theta)$, with $g_{\mathrm{loss}}$ evaluated at the current iterate and held fixed. The gradient of either loss is the plug-in term of the generative hybrid update below.

\textit{Per-scenario gradient variant.} Evaluate the surrogate gradient separately for each generated scenario and average,
\[
  g_{t,k} = g_{\mathrm{loss}}\bigl(\hat c_t^{(k)},\, f_{\phi_t}(x_t)\bigr),
  \qquad
  L_t^{\mathrm{ps}}(\theta)
  =
  \frac{1}{K}\sum_{k=1}^K
  \bigl\langle \hat c_t^{(k)}(\theta),\, g_{t,k}\bigr\rangle .
\]
For oracle-call point surrogates, this requires $K$ surrogate evaluations.

\textit{Sample-average approximation (SAA) variant.} Average scenarios first, then evaluate the surrogate gradient once at the empirical mean,
\[
  g_t^{\mathrm{SAA}} = g_{\mathrm{loss}}\bigl(\bar c_t,\, f_{\phi_t}(x_t)\bigr),
  \qquad
  L_t^{\mathrm{SAA}}(\theta)
  =
  \bigl\langle \bar c_t(\theta),\, g_t^{\mathrm{SAA}}\bigr\rangle .
\]
For oracle-call point surrogates, SAA requires a single surrogate evaluation rather than $K$.

The two variants give the same gradient only when $g_{\mathrm{loss}}(\bar c_t,\, f_{\phi_t}(x_t)) = \tfrac{1}{K}\sum_{k=1}^K g_{t,k}$, for example when the surrogate gradient is affine in the predicted cost or independent of it after fixing the candidate pool. For oracle-call losses such as SPO+, DBB, PG, implicit maximum likelihood estimation (IMLE), adaptive IMLE (AIMLE), and PFYL, this need not hold because the oracle solution can change with the predicted cost.

\paragraph{Hybrid policy update for generative cost models.}
Both generative models combine three gradients: the score function gradient $g_t^{\mathrm{score}}$ above ($g_{t,\mathrm{CNF}}^{\mathrm{score}}$ or $g_{t,\mathrm{diff}}^{\mathrm{score}}$), the plug-in term $g_t^{\mathrm{aux}} := \nabla_\theta L_t^{\mathrm{aux}}(\theta_t)$ with $L_t^{\mathrm{aux}} \in \{L_t^{\mathrm{ps}}, L_t^{\mathrm{SAA}}\}$, and a regularizer $g_t^{\mathrm{reg}}$ that fits the model to the nuisance estimate $\check c_t := f_{\phi_t}(x_t)$. For the CNF, $g_t^{\mathrm{reg}} = -\nabla_\theta \log p_{\theta_t}(\check c_t\mid x_t)$. For the diffusion model, it is the denoising mean-squared-error (MSE) regularizer
\[
  g_t^{\mathrm{reg}}
  =
  \nabla_\theta\,\mathbb{E}_{\tau,\,\varepsilon}\!\Bigl[\bigl\|\varepsilon - \varepsilon_{\theta_t}(\check c_t^{(\tau)}, \tau, x_t)\bigr\|^2\Bigr],
\]
where $\varepsilon\sim\mathcal N(0,I_{d_w})$ is the forward-process noise, $\check c_t^{(\tau)}$ is the forward-noised version of $\check c_t$ at diffusion step $\tau$, and $\varepsilon_\theta(\cdot,\tau,x)$ is the denoiser network, which predicts that noise. The hybrid policy update is
\[
  g_t^{\alpha_t}
  =
  \alpha_t\, g_t^{\mathrm{score}}
  + (1-\alpha_t)\Bigl[\kappa_{\mathrm{dfl}}\, g_t^{\mathrm{aux}} + \lambda_{\mathrm{reg}}\, g_t^{\mathrm{reg}}\Bigr].
\]
The hyperparameters $\kappa_{\mathrm{dfl}}$ and $\lambda_{\mathrm{reg}}$ are tuned per setting (Appendix~\ref{app:selected_hyperparameters}).

The $\lambda_{\mathrm{reg}}\, g_t^{\mathrm{reg}}$ term fits the model to the nuisance estimate $f_{\phi_t}(x_t)$. Without it, the auxiliary surrogate alone can move the sampler off the estimated conditional cost density.

\section{Baseline algorithms}
\label{app:baseline_algorithms}

The contextual-bandit baselines below use the same pure bandit feedback $v_t = c_t^\top w_t$ as Algorithm~\ref{alg:main}. GreedyCB exploits a point cost predictor $f_\theta$, $\epsilon$-GreedyCB replaces that point prediction by a standard-normal random cost vector on exploration iterations (default $\epsilon = 0.1$), and TSCB samples from the learned conditional cost distribution.

\IncMargin{1em}
\begin{algorithm}[H]
  \SetAlgoLined
  \KwIn{Initial parameters $\theta_1 \in \R^{d_\theta}$; stepsize $\eta > 0$}
  \For{$t = 1, 2, \dots, T$}{
    Observe context $x_t$\;
    Commit to $w_t \in \argmin_{w \in \calW} f_{\theta_t}(x_t)^\top w$\;
    Observe feedback $v_t$\;
    $L_t^{\mathrm{G}}(\theta)
      \gets
      \bigl(f_\theta(x_t)^\top w_t - v_t\bigr)^2$\;
    $\theta_{t+1}
      \gets
      \theta_t - \eta \,
      \nabla_\theta L_t^{\mathrm{G}}(\theta)\big|_{\theta=\theta_t}$\;
  }
  \caption{Greedy contextual bandit (\textsc{GreedyCB})}
  \label{alg:bandit_scalar_regression}
\end{algorithm}

\begin{algorithm}[H]
  \SetAlgoLined
  \KwIn{Initial parameters $\theta_1 \in \R^{d_\theta}$; stepsize $\eta > 0$; exploration probability $\epsilon \in [0,1]$}
  \For{$t = 1, 2, \dots, T$}{
    Observe context $x_t$\;
    Draw $U_t\sim \operatorname{Unif}[0,1]$\;
    \eIf{$U_t\le \epsilon$}{
      Sample a random exploratory cost vector $\tilde c_t\sim \mathcal N(0,I_{d_w})$\;
    }{
      $\tilde c_t \gets f_{\theta_t}(x_t)$\;
    }
    Commit to $w_t \in \argmin_{w \in \calW} \tilde c_t^\top w$\;
    Observe feedback $v_t$\;
    $L_t^{\epsilon\text{-}\mathrm{G}}(\theta)
      \gets
      \bigl(f_\theta(x_t)^\top w_t - v_t\bigr)^2$\;
    $\theta_{t+1}
      \gets
      \theta_t - \eta \,
      \nabla_\theta L_t^{\epsilon\text{-}\mathrm{G}}(\theta)\big|_{\theta=\theta_t}$\;
  }
  \caption{$\epsilon$-Greedy contextual bandit ($\epsilon$-\textsc{GreedyCB})}
  \label{alg:epsilon_greedy_contextual_bandit}
\end{algorithm}

\begin{algorithm}[H]
  \SetAlgoLined
  \KwIn{Initial parameters $\theta_1 \in \R^{d_\theta}$ of a conditional distribution $p_\theta(\cdot\mid x)$; stepsize $\eta > 0$}
  \For{$t = 1, 2, \dots, T$}{
    Observe context $x_t$\;
    Sample $\hat c_t\sim p_{\theta_t}(\cdot\mid x_t)$\;
    Commit to $w_t \in \argmin_{w \in \calW} \hat c_t^\top w$\;
    Observe feedback $v_t$\;
    Let $q_\theta(\cdot\mid x_t,w_t)$ be the distribution of $c^\top w_t$
    induced by $c\sim p_\theta(\cdot\mid x_t)$\;
    $L_t^{\mathrm{TS}}(\theta)
      \gets
      -\log q_\theta(v_t\mid x_t,w_t)$\;
    $\theta_{t+1}
      \gets
      \theta_t - \eta \,
      \nabla_\theta L_t^{\mathrm{TS}}(\theta_t)$\;
  }
  \caption{Thompson-style contextual bandit (\textsc{TSCB})}
  \label{alg:thompson_contextual_bandit}
\end{algorithm}
\DecMargin{1em}

\paragraph{Semi-bandit adaptation.}
Under semi-bandit feedback, the CB baselines retain their action-selection rules but replace the scalar squared-loss update by the vector form $\|v_t - H(w_t)\,f_\theta(x_t)\|^2$, i.e., the same nuisance least-squares objective minimized by $f_\phi$ in Algorithm~\ref{alg:main}, applied to the CB cost predictor. For coordinate-mask observations $H(w_t) = \operatorname{diag}(w_t)$ this reduces to a per-coordinate masked regression on the action-selected entries. Under semi-bandit feedback the CB baselines keep their pure-bandit learning rates, so the comparison isolates the richer feedback signal.

\paragraph{Linear bandit baselines (OFUL and LinTS).} Both algorithms use the same linear cost model $f_\theta$ as the Gaussian linear variant of the other methods. Because this model is linear in $\theta$, the predicted cost of decision $w$ in context $x$ is an inner product, $f_\theta(x)^\top w=\theta^\top z(x,w)$, where the feature vector $z(x,w)$ combines the decision and the context. For example, on top-$k$, shortest path, and pricing the model is $f_\theta(x)=\Theta x$, where $\Theta\in\R^{d_w\times d_x}$ is a coefficient matrix whose entries form $\theta$. The predicted cost is then $w^\top\Theta x=\sum_{i,j}\Theta_{ij}\,w_i x_j$, so $z(x,w)$ collects the products $w_i x_j$. OFUL and LinTS treat $z(x,w)$ as the feature vector of decision $w$. Both keep the ridge statistics $V_t$ and $b_t$ of the observed pairs $(z_s,v_s)$, with $z_s=z(x_s,w_s)$, and use the noise scale $R=1$. We write $\|u\|_A=\sqrt{u^\top A u}$. OFUL (Algorithm~\ref{alg:oful}) minimizes a lower confidence bound on the cost, with the confidence radius $\rho_t$ computed from \citet[Theorem~2]{abbasi2011improved} at $\delta=0.05$ and parameter-norm bound $S=5$. Its bonus is nonlinear in $w$, so each round enumerates $\calW$ ($105$ decisions on top-$k$, $70$ paths on shortest path). LinTS (Algorithm~\ref{alg:lints}) samples a parameter from the Gaussian posterior, which starts from the prior $\mathcal N(0,R^2I)$, and calls the linear oracle once per round.

\IncMargin{1em}

\begin{algorithm}[H]
  \SetAlgoLined
  \KwIn{Feature map $z$; noise scale $R$; confidence level $\delta\in(0,1)$; parameter-norm bound $S$}
  $V_1\gets I$, $b_1\gets 0$\;
  \For{$t = 1, 2, \dots, T$}{
    Observe context $x_t$\;
    $\hat\theta_t\gets V_t^{-1}b_t$ and $\rho_t\gets R\sqrt{\log\det V_t+2\log(1/\delta)}+S$\;
    Commit to $w_t\in\argmin_{w\in\calW}\bigl\{\hat\theta_t^\top z(x_t,w)-\rho_t\,\|z(x_t,w)\|_{V_t^{-1}}\bigr\}$, by enumerating $\calW$\;
    Observe feedback $v_t$ and set $z_t\gets z(x_t,w_t)$\;
    $V_{t+1}\gets V_t+z_tz_t^\top$ and $b_{t+1}\gets b_t+v_t z_t$\;
  }
  \caption{Optimism in the face of uncertainty for linear bandits (OFUL)}
  \label{alg:oful}
\end{algorithm}

\begin{algorithm}[H]
  \SetAlgoLined
  \KwIn{Feature map $z$ and linear cost model $f_\theta$; noise scale $R$}
  $V_1\gets I$, $b_1\gets 0$\;
  \For{$t = 1, 2, \dots, T$}{
    Observe context $x_t$\;
    $\hat\theta_t\gets V_t^{-1}b_t$\;
    Sample $\tilde\theta_t\sim\mathcal N\bigl(\hat\theta_t,\,R^2V_t^{-1}\bigr)$\;
    Commit to $w_t\in\argmin_{w\in\calW} f_{\tilde\theta_t}(x_t)^\top w$\;
    Observe feedback $v_t$ and set $z_t\gets z(x_t,w_t)$\;
    $V_{t+1}\gets V_t+z_tz_t^\top$ and $b_{t+1}\gets b_t+v_t z_t$\;
  }
  \caption{Linear Thompson sampling (LinTS)}
  \label{alg:lints}
\end{algorithm}

\DecMargin{1em}

\section{Hyperparameter selection and compute}
\label{app:hyperparameter_selection}

\subsection{Tuning protocol}
\label{app:tuning_protocol}

Hyperparameters were selected by random search per (benchmark, distributional cost model) on dedicated validation runs that share no seeds with the reported paper runs. Validation used 10 seeds for top-$k$ and shortest path, and 5 seeds for pricing and energy; the paper results use 30 seeds. Validation horizons were $T=2{,}000$ for the point-model sweep, $T=5{,}000$ for a two-stage generative screen, and $T=394$ on the energy \texttt{load3} split (validation fraction $0.5$); paper horizons are $T=2{,}000$ for point models ($T=10{,}000$ for pricing, whose raw-context linear model has $2{,}000$ parameters: at $T=2{,}000$ every method on pricing is still within $11\%$ of the random policy, and by $T=10{,}000$ all methods have learned enough to separate from it), $T=15{,}000$ for the generative comparison, and the held-out experiment split for energy. Candidates were ranked by mean final cumulative regret on the validation seeds. For the pricing generative models, the $T=5{,}000$ screen covered $48$ CNF and $72$ diffusion candidates, and the top three per model were confirmed on $10$ validation seeds. Full search spaces and per-benchmark sweep specifications are available in the supplementary code.

\subsection{Selected configurations}
\label{app:selected_hyperparameters}

Across (benchmark, distributional cost model) settings, the selected configurations share the following defaults:
\begin{itemize}
\item Policy and nuisance learning rates on the same shifted-inverse schedule, $\eta^t = \eta^0 / (1 + t/o)$ with offset $o = 100$. The reported $\eta_\theta^0$ and $\eta_\phi^0$ are the initial values of the policy step size ($\beta$ in Algorithm~\ref{alg:main}) and the nuisance step size ($\eta_t$).
\item Adaptive mixing-weight schedule (Appendix~\ref{app:alpha_schedule}): smoothing factor $0.05$ on $\alpha_t$ updates, fixed-warmup gate of $T_{\mathrm{warm}} = 100$ iterations.
\item For pricing, energy, top-$k$ NN, and the generative models, the score function and plug-in gradients are $L_2$-normalized separately before the convex mix $\alpha\,g_t^{\mathrm{score}} + (1-\alpha)\,g_t^{\mathrm{plug\text{-}in}}$, so that $\alpha$ sets the direction independently of the raw magnitudes. The top-$k$ Gaussian linear and shortest-path Gaussian configurations use raw gradients.
\end{itemize}
Surrogate names follow Appendix~\ref{app:surrogate_definitions}; baselines are moving-average (\textsc{ma}) and nuisance-induced (\textsc{ni}) from Appendix~\ref{app:advantage_choices}.

Table~\ref{tab:app-hp-point} reports the point-model configurations across all four benchmarks. Table~\ref{tab:app-hp-gen} reports the generative-model configurations across the same four benchmarks. All configurations use the adaptive mixing-weight schedule of Appendix~\ref{app:alpha_schedule}. In the energy configurations and the pricing CNF configuration, the mixing weight decays over a time scale given in column $\tau_{\mathrm{sched}}$. The other configurations hold it at $\alpha_{\max}$ for the first 100 iterations and then switch to $\alpha_{\min}$, so their $\tau_{\mathrm{sched}}$ entry is --. Top-$k$ and shortest-path entries use $\deg=8$ and pricing entries $\deg=6$; synthetic distributional-model comparisons use $T=15{,}000$, and energy uses its native test-window length.

\begin{table}[!ht]
\centering
\small
\caption{Point-model configurations across the four benchmarks. Columns: initial policy learning rate (LR) $\eta_\theta^0$, initial nuisance LR $\eta_\phi^0$, plug-in surrogate, exploration scale $\sigma_{\mathrm{G}}$, mixing bounds $\alpha_{\max}, \alpha_{\min}$, time scale $\tau_{\mathrm{sched}}$ over which the mixing weight decays (-- for configurations that hold $\alpha_{\max}$ for the first 100 iterations instead, Appendix~\ref{app:alpha_schedule}), baseline.}
\label{tab:app-hp-point}
\resizebox{\textwidth}{!}{%
\begin{tabular}{llcccccccc}
\toprule
Benchmark & Class & $\eta_\theta^0$ & $\eta_\phi^0$ & Surrogate & $\sigma_{\mathrm{G}}$ & $\alpha_{\max}$ & $\alpha_{\min}$ & $\tau_{\mathrm{sched}}$ & Baseline \\
\midrule
Top-$k$       & Gaussian linear & 6.8e-2 & 5.3e-2 & pairwiseDiff & 8.6e-1 & 0.3 & 0.02 & --  & \textsc{ma} \\
Top-$k$       & Gaussian NN     & 7.5e-2 & 1.1e-2 & IMLE         & 2.6e-1 & 0.1 & 0.05 & --  & \textsc{ma} \\
Shortest path & Gaussian linear & 3.0e-2 & 1.5e-2 & pairwiseDiff & 3.9e-1 & 0.5 & 0.05 & --  & \textsc{ni} \\
Shortest path & Gaussian NN     & 3.0e-2 & 1.0e-2 & pairwiseDiff & 5.3e-1 & 0.3 & 0.02 & --  & \textsc{ni} \\
Pricing & Gaussian linear & 8.2e-1 & 1.6e-3 & DBB & 4.9e-5 & 0.5 & 0.05 & -- & \textsc{ni} \\
Pricing & Gaussian NN & 5.3e-1 & 1.0e-3 & DBB & 1.4e-5 & 0.5 & 0.05 & -- & \textsc{ni} \\
Energy        & Gaussian linear & 5.0e-3 & 2.5e-2 & MAP-C        & 1.9e-3 & 0.1 & 0.0  & 100 & \textsc{ni} \\
Energy        & Gaussian NN     & 1.0e-3 & 1.0e-3 & DBB          & 5.0e-3 & 0.8 & 0.2  & 100 & \textsc{ma} \\
\bottomrule
\end{tabular}%
}
\end{table}

\begin{table}[!ht]
\centering
\footnotesize
\caption{Generative-model configurations across the four benchmarks. The auxiliary plug-in path uses the per-scenario gradient (Appendix~\ref{app:generative_param}) with the auxiliary surrogate as listed, $K$ scenarios per iteration, surrogate-loss weight $\kappa_{\mathrm{dfl}}$, and generative-regularizer weight $\lambda_{\mathrm{reg}}$. Energy CNF uses two coupling layers with hidden width $128$; energy diffusion uses 20 training and 8 inference timesteps with ELBO weighting.}
\label{tab:app-hp-gen}
\resizebox{\textwidth}{!}{%
\begin{tabular}{llcccccccccc}
\toprule
Benchmark & Class & $\eta_\theta^0$ & $\eta_\phi^0$ & $\alpha_{\max}$ & $\alpha_{\min}$ & $\tau_{\mathrm{sched}}$ & Baseline & Aux.\ surrogate & $K$ & $\kappa_{\mathrm{dfl}}$ & $\lambda_{\mathrm{reg}}$ \\
\midrule
Top-$k$       & CNF       & 3.0e-2 & 6.0e-3 & 0.2 & 0.01 & --  & \textsc{ni} & pairwiseDiff & 32  & 3.0  & 1.0  \\
Top-$k$       & Diffusion & 1.3e-2 & 1.3e-2 & 0.2 & 0.01 & --  & \textsc{ma} & pairwiseLTR  & 200 & 10.0 & 1.0  \\
Shortest path & CNF       & 9.0e-3 & 6.8e-3 & 0.5 & 0.02 & --  & \textsc{ma} & SPO+         & 128 & 1.0  & 0.03 \\
Shortest path & Diffusion & 1.0e-2 & 1.0e-2 & 0.1 & 0.0  & --  & \textsc{ni} & pairwiseDiff & 200 & 1.0  & 0.1  \\
Pricing & CNF & 9.0e-2 & 9.0e-2 & 0.5 & 0.02 & 500 & \textsc{ni} & pairwiseLTR & 200 & 10.0 & 0.3 \\
Pricing & Diffusion & 7.4e-3 & 1.0e-3 & 0.0 & 0.0 & -- & \textsc{ma} & pairwiseLTR & 64 & 3.0 & 0.3 \\
Energy        & CNF       & 1.5e-2 & 1.1e-2 & 0.8 & 0.05 & 500 & \textsc{ni} & pairwiseLTR  & 100 & 1.0  & 0.1  \\
Energy        & Diffusion & 3.0e-3 & 3.0e-3 & 0.0 & 0.0  & 100 & \textsc{ni} & pairwiseLTR  & 100 & 1.0  & 0.1  \\
\bottomrule
\end{tabular}%
}
\end{table}

The pricing-diffusion and energy-diffusion entries use $\alpha_{\max} = \alpha_{\min} = 0$ (plug-in only), consistent with Appendix~\ref{app:alpha_ablation}.

\subsection{Compute resources}
\label{app:compute_resources}

All experiments ran on a shared high-performance computing cluster running Red Hat Enterprise Linux 9.6. Each CPU trial uses one Intel Xeon Gold 6226 core at 2.70~GHz with 2~GB of RAM. GPU trials use a single NVIDIA L40S or RTX 6000.

\section{Additional experimental results}
\label{app:additional_experimental_results}

This section collects supplementary tables and figures for the numerical study. The main text reports the headline comparisons and convergence plots; this appendix gives the stationarity measurement details, the cost-model comparison on the remaining benchmarks, sensitivity analyses, and the feedback-mode comparison. Compute resources used to run the experiments are reported in Appendix~\ref{app:compute_resources}.

\subsection{Stationarity measurement details}
\label{app:stationarity}

On top-$k$ and shortest path the conditional mean $f^\star$ is known to the evaluator, so $\nabla_\theta\Vcost(\theta)=\E\bigl[f^\star(x)^\top w^\star(\hat c)\,\nabla_\theta\log p_\theta(\hat c\mid x)\bigr]$ can be estimated without differentiating through the oracle. At every iterate we form two independent Monte Carlo batch means of this expression, each from $2{,}000$ samples, and take their inner product, which is an unbiased estimate of $\|\nabla_\theta\Vcost(\theta_t)\|^2$; negative estimates are kept. The runs use the fixed-$\alpha$ configuration at $\alpha=0$, where the nuisance-bias term of Theorem~\ref{thm:main-alpha} is largest (constant policy and nuisance step sizes, raw gradients without normalization) with the Table~\ref{tab:main-methods-selected-settings} replications, the surrogate of Table~\ref{tab:app-hp-point}, and the neural-network nuisance model. The fixed-$\alpha$ ablation uses the same configuration with $\alpha\in\{0,0.25\}$: $\alpha=0.25$ lowers cumulative expected cost on top-$k$ from $3.03{\times}10^5$ to $2.31{\times}10^5$ (better in 25 of 30 replications) and performs comparably to $\alpha=0$ on shortest path (better in 14 of 30). Pricing is excluded: its selected sampling scale $\sigma_{\mathrm G}\approx5\times10^{-5}$ makes the score-function estimate too noisy at feasible sample sizes.

\subsection{Distributional cost models}
\label{app:model_variants_remaining}

Section~\ref{sec:results} compares the cost models on pricing. Figure~\ref{fig:app-model-variants-remaining} reports the same comparison on top-$k$ selection, shortest path, and energy scheduling; the synthetic benchmarks use $T=15{,}000$ and energy uses its native test-window length. On top-$k$ selection the Gaussian models finish slightly below CNF and diffusion, and on shortest path the Gaussian linear model has the lowest final regret. On energy scheduling the Gaussian NN has substantially higher regret than the other three models. No single cost model dominates across benchmarks, and the \textsc{DFHPG} framework accommodates all four.

\begin{figure}[!ht]
\centering
\begin{subfigure}[t]{0.49\textwidth}
\centering
\includegraphics[width=\textwidth]{figures/block02_t15000_pricing_avg_regret_per_round.pdf}
\caption{Pricing ($T=15{,}000$)}
\label{fig:app-model-variants-pricing}
\end{subfigure}\hfill
\begin{subfigure}[t]{0.49\textwidth}
\centering
\includegraphics[width=\textwidth]{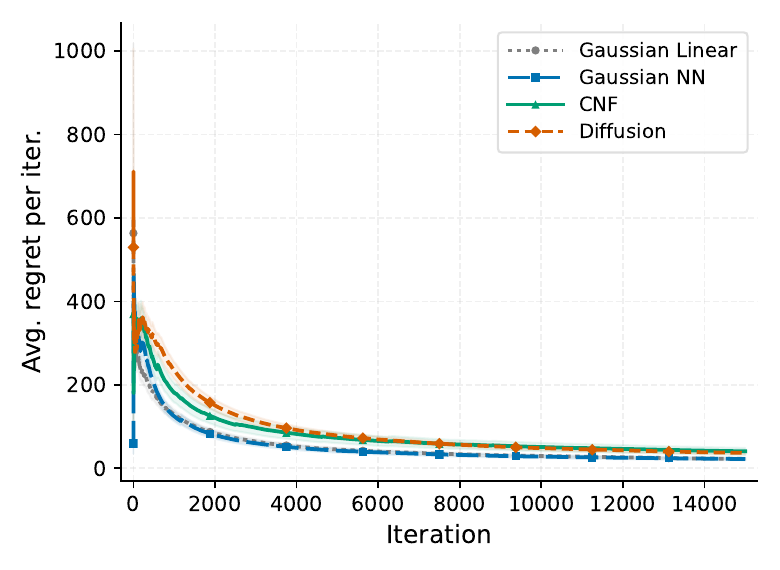}
\caption{Top-$k$ selection ($T=15{,}000$)}
\end{subfigure}

\vspace{0.75em}
\begin{subfigure}[t]{0.49\textwidth}
\centering
\includegraphics[width=\textwidth]{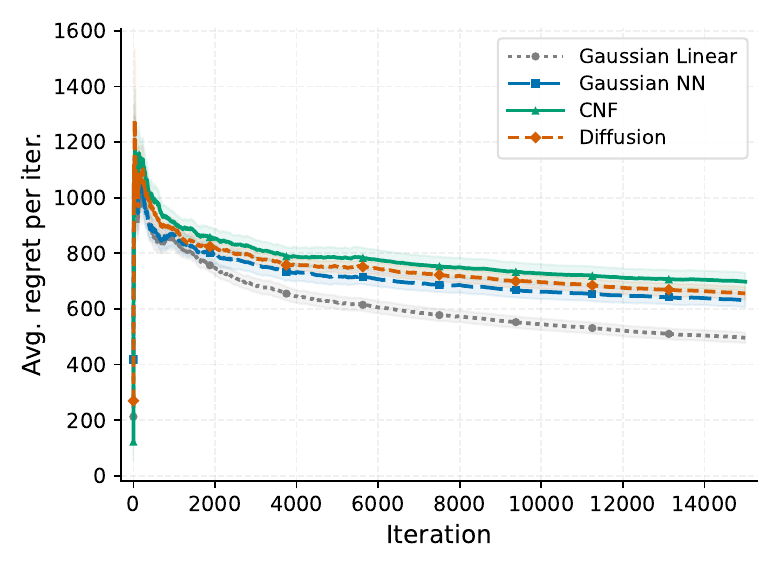}
\caption{Shortest path ($T=15{,}000$)}
\end{subfigure}\hfill
\begin{subfigure}[t]{0.49\textwidth}
\centering
\includegraphics[width=\textwidth]{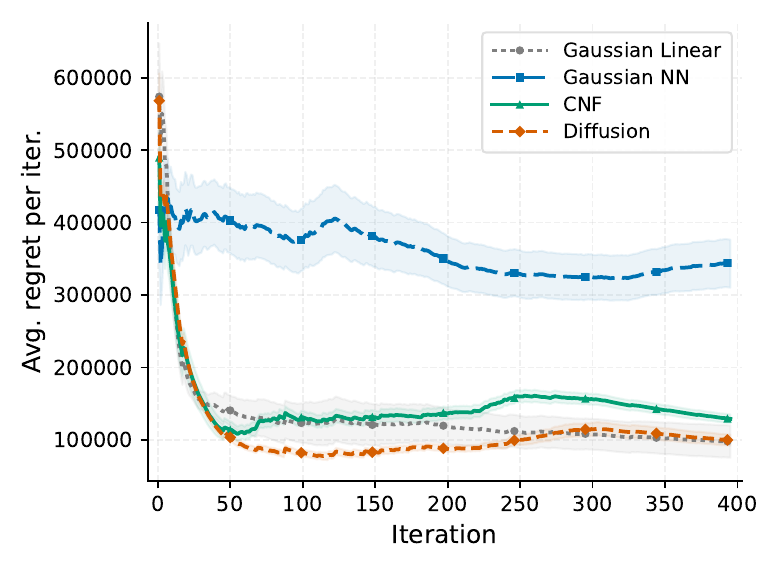}
\caption{Energy scheduling}
\end{subfigure}
\caption{Comparison of distributional cost models on pricing, top-$k$ selection, shortest path, and energy scheduling. Each curve reports average regret per iteration, with shaded $\pm 1$ standard-error bands across replications.}
\label{fig:app-model-variants-remaining}
\end{figure}

\subsection{Alternative surrogate choices}
\label{app:surrogate_choices_results}

The surrogate loss for the decision-focused plug-in component is treated as a tunable hyperparameter rather than fixed to a default. This appendix reports the per-benchmark sweep over standard surrogate choices, implemented following the PyEPO reference implementations~\citep{tang2024pyepo}, across the pricing, shortest-path, top-$k$, and energy-scheduling benchmarks. All synthetic comparisons in this section use the Gaussian linear cost model; the energy comparison uses the same setup as in the main results. Each heatmap changes only the surrogate loss used by the decision-focused component; the score function update, nuisance update, contexts, feasible set, and feedback model are held fixed. The two columns correspond to the full hybrid update \textsc{DFHPG} and the plug-in-only update \textsc{DFHPG-0}. Cell values are mean final cumulative regret, so lower values are better.

\begin{figure}[!ht]
\centering
\begin{subfigure}[t]{0.49\textwidth}
\centering
\includegraphics[width=\textwidth]{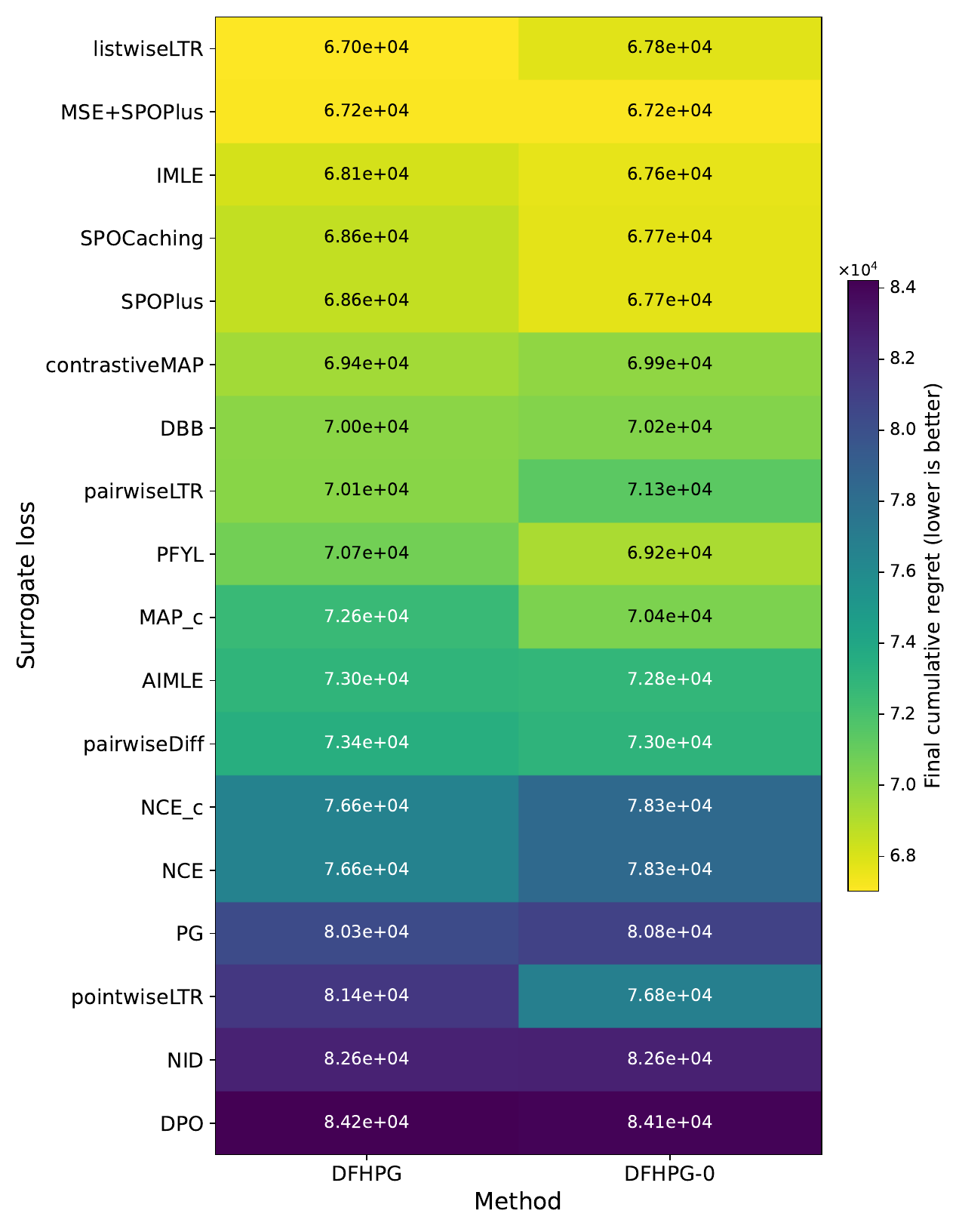}
\caption{Pricing}
\end{subfigure}\hfill
\begin{subfigure}[t]{0.49\textwidth}
\centering
\includegraphics[width=\textwidth]{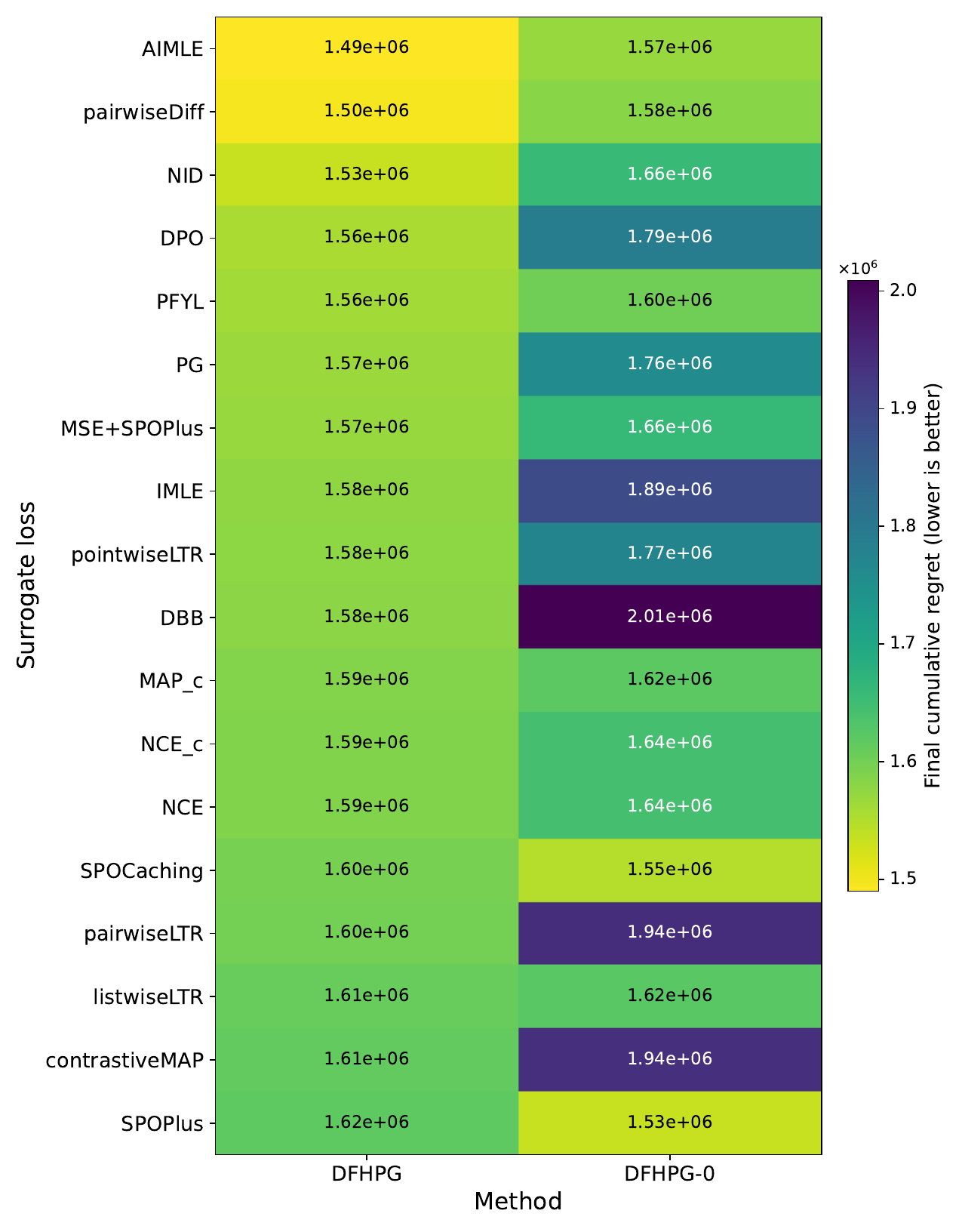}
\caption{Shortest path}
\end{subfigure}

\vspace{0.75em}
\begin{subfigure}[t]{0.49\textwidth}
\centering
\includegraphics[width=\textwidth]{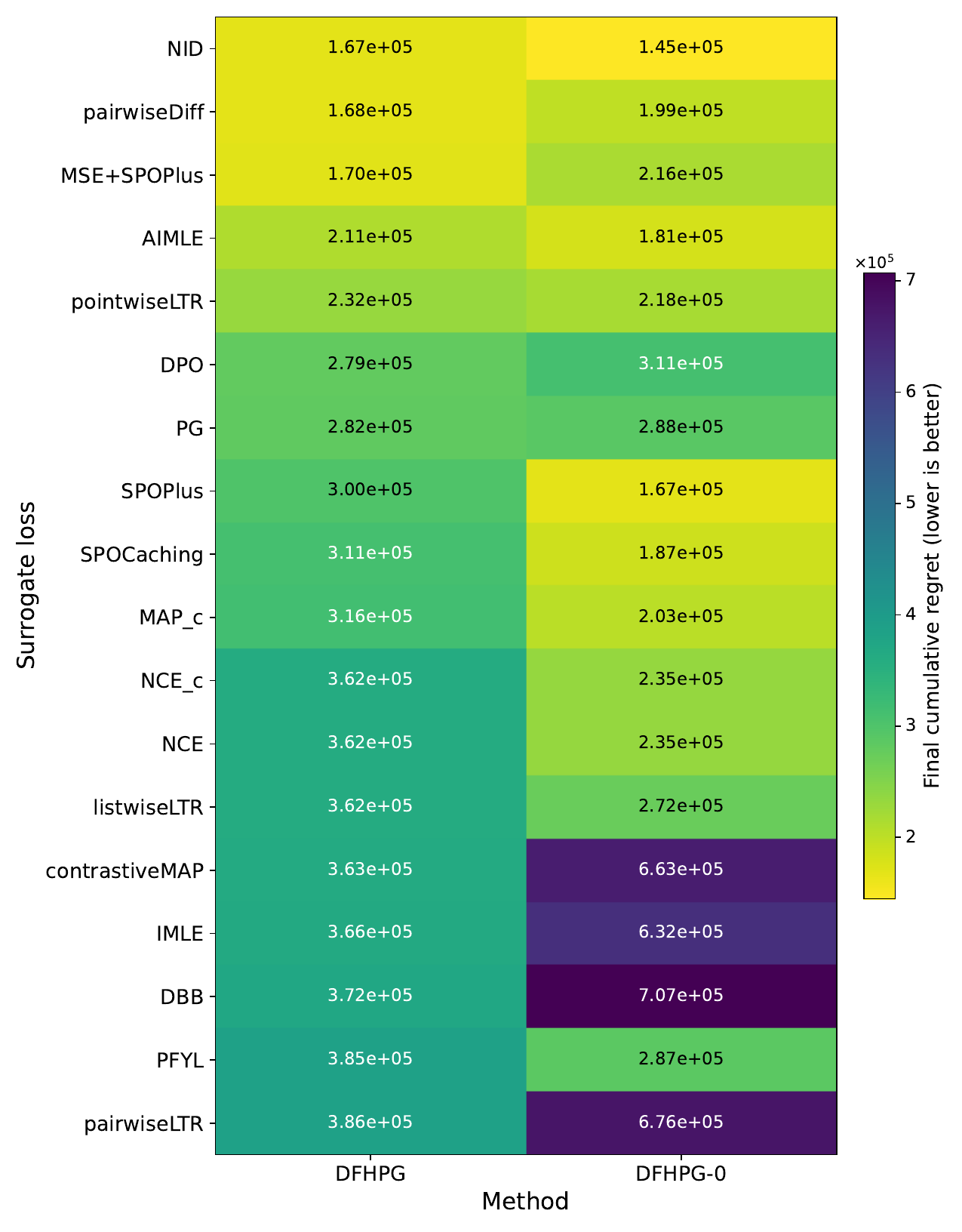}
\caption{Top-$k$ selection}
\end{subfigure}\hfill
\begin{subfigure}[t]{0.49\textwidth}
\centering
\includegraphics[width=\textwidth]{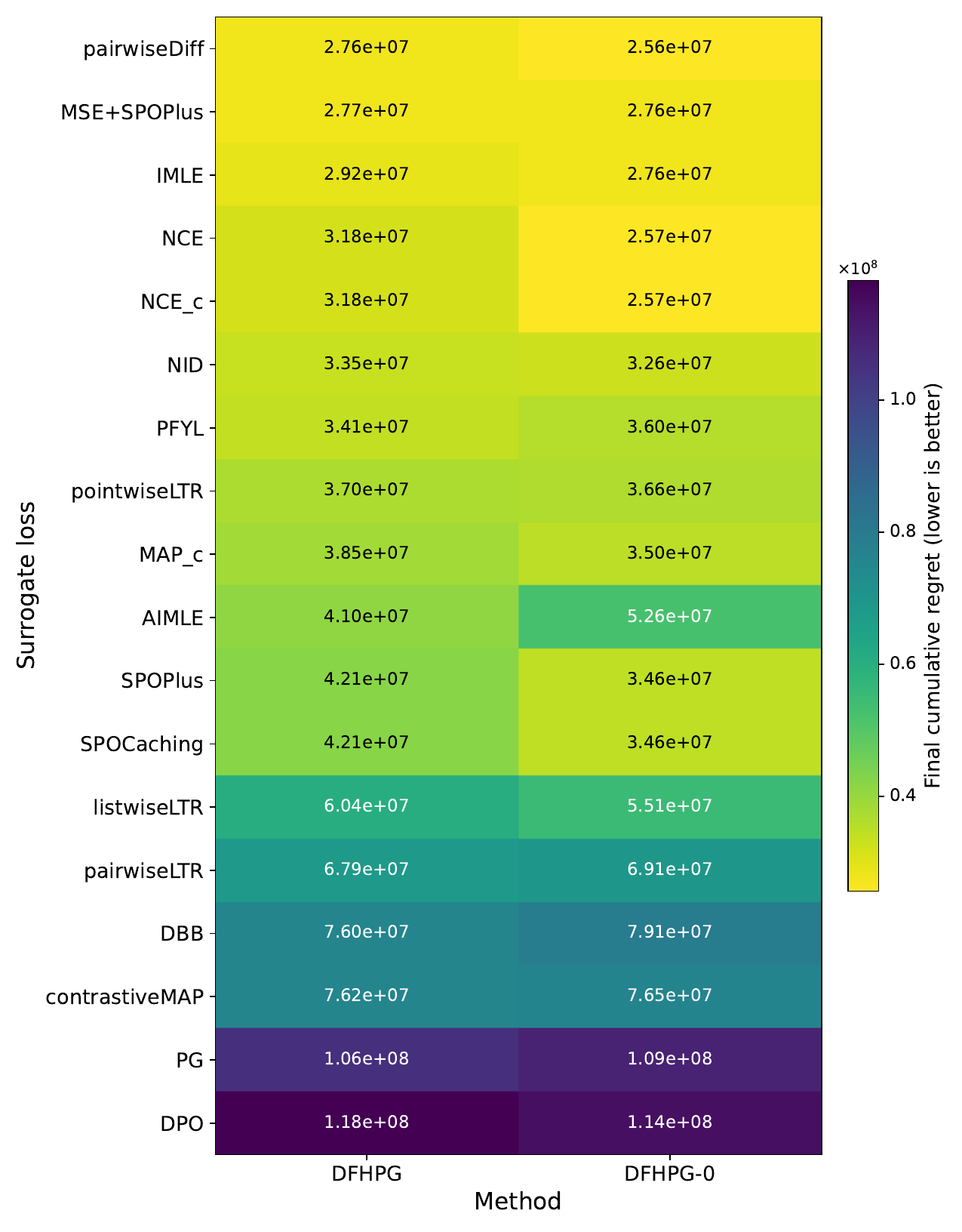}
\caption{Energy scheduling}
\end{subfigure}
\caption{Surrogate-loss comparison across the four benchmark problems. Each heatmap reports mean final cumulative regret for \textsc{DFHPG} and \textsc{DFHPG-0} when changing only the surrogate loss used by the decision-focused component; lower values are better.}
\label{fig:surrogate-loss-heatmaps}
\end{figure}

Figure~\ref{fig:surrogate-loss-heatmaps} shows that no single surrogate dominates across benchmarks, motivating the per-problem tuning protocol. On pricing, listwiseLTR, MSE+SPOPlus, IMLE, and SPOPlus (tied with SPOCaching) lead in both columns, while PG, NID, and DPO trail, and pointwiseLTR also trails for \textsc{DFHPG}. On shortest path, \textsc{DFHPG} is approximately flat across a wide best group (AIMLE, pairwiseDiff, NID, DPO, PFYL, PG), whereas \textsc{DFHPG-0} is more variable and peaks at SPOPlus, SPOCaching, and AIMLE. On top-$k$ selection, NID, pairwiseDiff, and MSE+SPOPlus lead for \textsc{DFHPG}, while \textsc{DFHPG-0} prefers NID and SPO-type losses. On energy scheduling the ranking inverts relative to pricing: pairwiseDiff, MSE+SPOPlus, IMLE, and NCE lead, while DBB and contrastiveMAP fall to the bottom half, and PG and DPO collapse by an order of magnitude.

\subsection{Mixing-weight sensitivity}
\label{app:alpha_ablation}

Figure~\ref{fig:app-alpha-ablation} ablates the mixing weight between the score function and plug-in components. We compare the adaptive schedule from Appendix~\ref{app:alpha_schedule} against fixed values of $\alpha$ on all four benchmarks. The adaptive schedule is a strong default but does not strictly dominate the best fixed $\alpha$: on top-$k$ selection, shortest path, and pricing it matches or modestly improves on the strongest fixed setting; on the energy-scheduling benchmark it trails the smaller fixed $\alpha$ values.

\begin{figure}[!ht]
\centering
\begin{subfigure}[t]{0.49\textwidth}
\centering
\includegraphics[width=\textwidth]{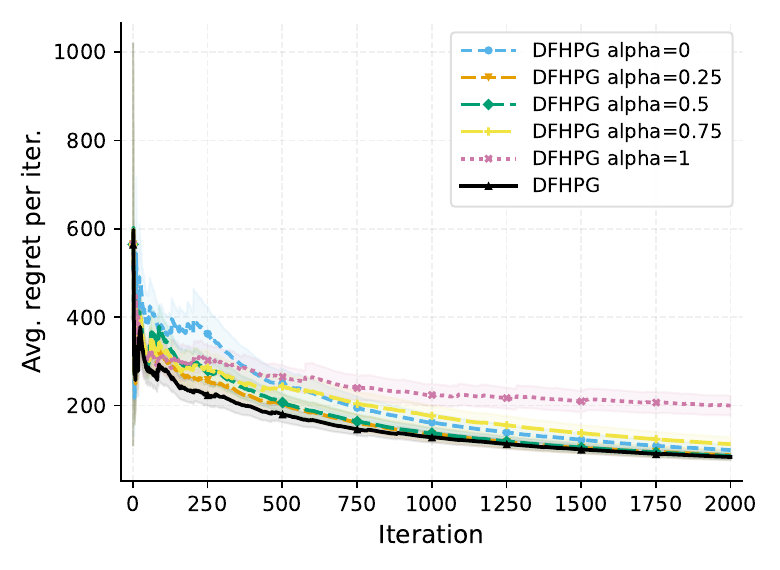}
\caption{Top-$k$ selection}
\end{subfigure}\hfill
\begin{subfigure}[t]{0.49\textwidth}
\centering
\includegraphics[width=\textwidth]{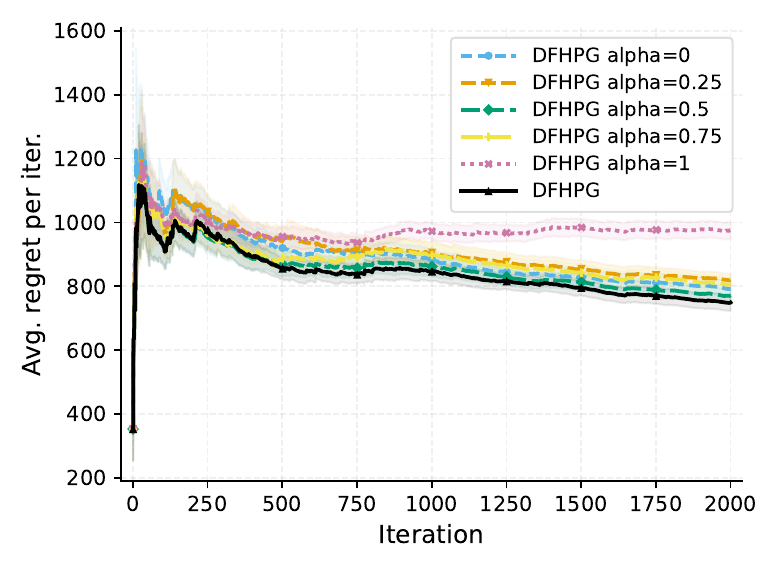}
\caption{Shortest path}
\end{subfigure}

\vspace{0.75em}
\begin{subfigure}[t]{0.49\textwidth}
\centering
\includegraphics[width=\textwidth]{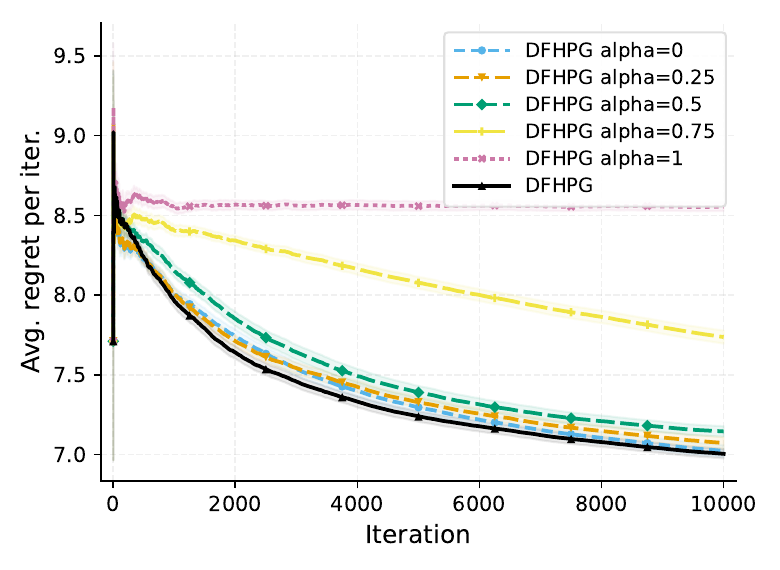}
\caption{Pricing}
\end{subfigure}\hfill
\begin{subfigure}[t]{0.49\textwidth}
\centering
\includegraphics[width=\textwidth]{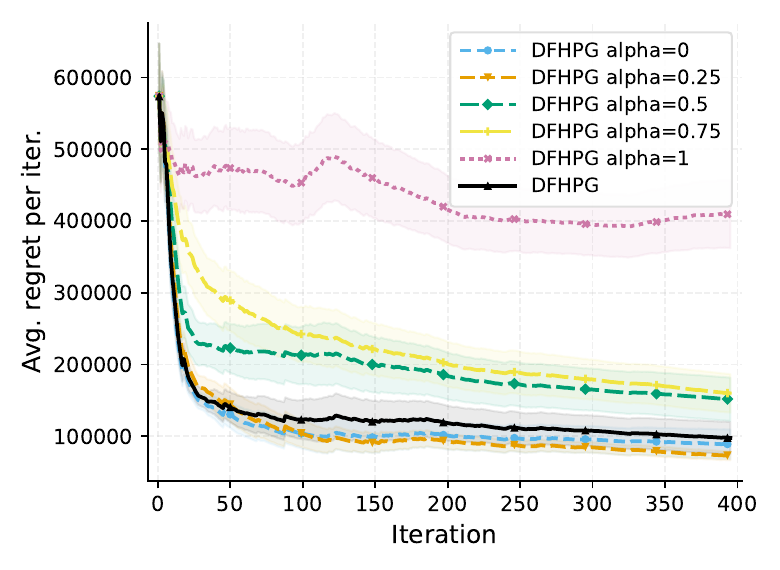}
\caption{Energy scheduling}
\end{subfigure}
\caption{Mixing-weight ablation comparing the adaptive $\alpha_t$ schedule against fixed $\alpha$ on the four benchmarks. Each curve reports average regret per iteration, with shaded $\pm 1$ standard-error bands across replications.}
\label{fig:app-alpha-ablation}
\end{figure}

\subsection{Sensitivity across problem difficulty}
\label{app:problem_difficulty_sensitivity}

Figure~\ref{fig:app-degree-robustness} reports degree-robustness on pricing, shortest path, and top-$k$ selection, where the polynomial degree controls the nonlinearity of the synthetic cost generator (top-$k$ and shortest path) or the demand generator (pricing). On top-$k$ and shortest path, higher degrees act as a difficulty knob: every method's regret rises with degree under the Gaussian linear cost model. On pricing, the degree changes the shape of the willingness to pay while its typical size stays fixed (Appendix~\ref{sec:synthetic_pricing}), so regret varies little with degree. Across degrees, \textsc{DFHPG} and \textsc{DFHPG-0} retain a regret advantage over the contextual-bandit baselines, but the magnitude of the gap is benchmark-specific. On pricing, where \textsc{DFHPG} is tuned separately at each degree, the better of \textsc{DFHPG} and \textsc{DFHPG-0} is below the best contextual-bandit baseline at every degree, by $1$--$10\%$. On top-$k$ the gap is more modest at low degrees (around $2\times$ at $\deg=2$) and widens to roughly $4$--$5\times$ at $\deg \geq 4$. Shortest path shows small gaps throughout. The score-only \textsc{DFHPG-1} ablation tracks or trails the baselines, indicating that the decision-focused plug-in component is what drives the advantage as the problem hardens.

\begin{figure}[!ht]
\centering
\begin{subfigure}[t]{0.49\textwidth}
\centering
\includegraphics[width=\textwidth]{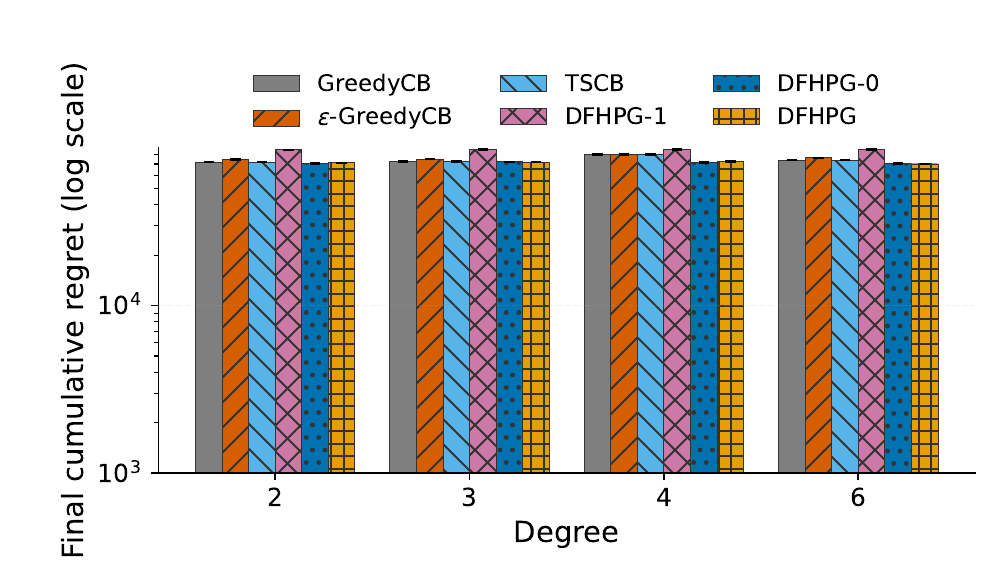}
\caption{Pricing}
\end{subfigure}\hfill
\begin{subfigure}[t]{0.49\textwidth}
\centering
\includegraphics[width=\textwidth]{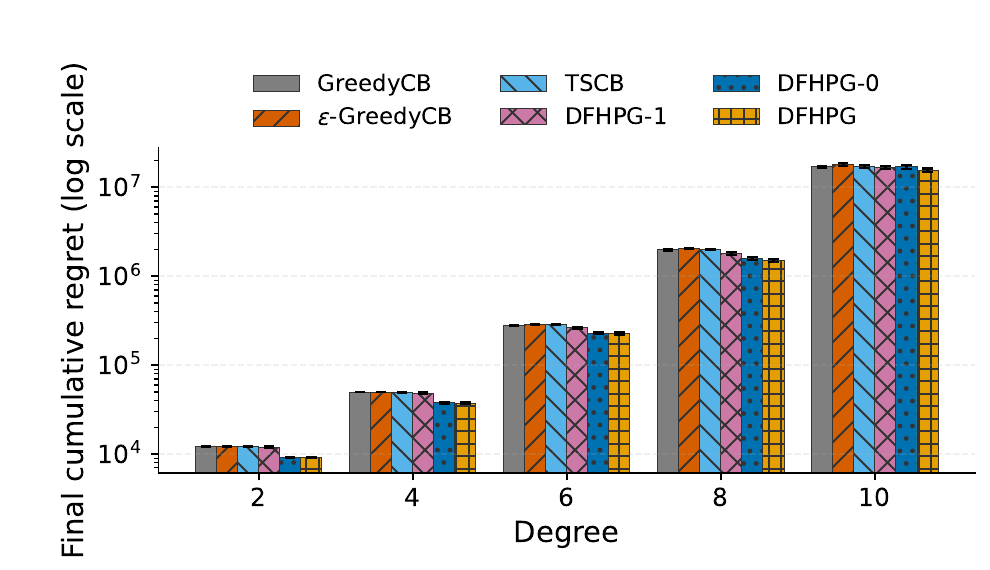}
\caption{Shortest path}
\end{subfigure}

\vspace{0.75em}
\begin{subfigure}[t]{0.49\textwidth}
\centering
\includegraphics[width=\textwidth]{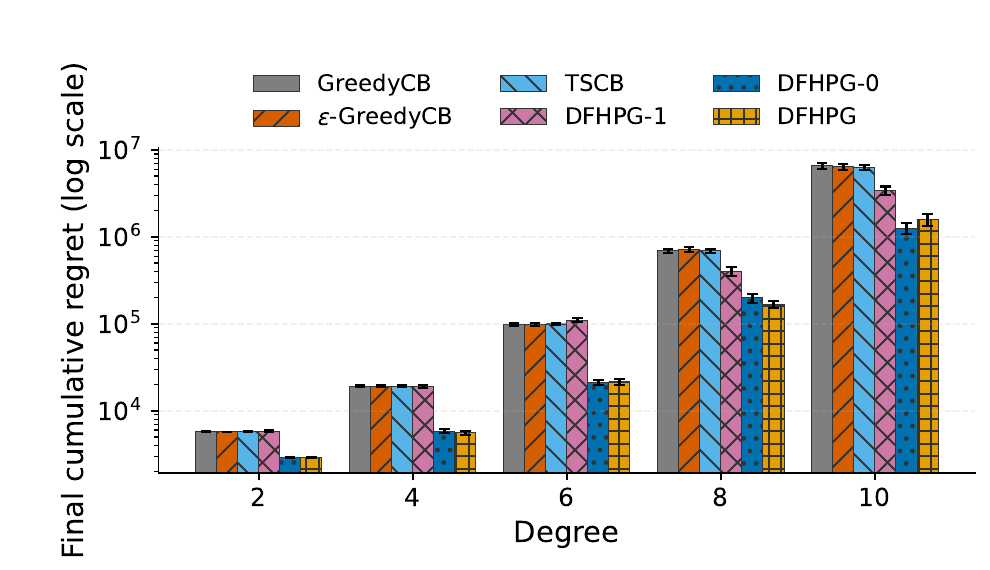}
\caption{Top-$k$ selection}
\end{subfigure}
\caption{Sensitivity across problem difficulty for pricing, shortest path, and top-$k$ selection. Each plot compares final cumulative regret across polynomial degrees for the Gaussian linear cost-model experiments. Lower is better.}
\label{fig:app-degree-robustness}
\end{figure}

\subsection{Feedback mode comparison}
\label{app:feedback_mode_comparison}

Figure~\ref{fig:app-feedback-mode} compares final cumulative regret across three feedback modes (pure bandit, semi-bandit, and full information) for both \textsc{DFHPG} and \textsc{DFHPG-0} on each benchmark. On all four problems, semi-bandit and full-information feedback lower regret relative to pure bandit feedback. Semi-bandit feedback sharpens the nuisance fit, whereas under full information \textsc{DFHPG} uses the single realized cost vector as its plug-in target (Appendix~\ref{app:synthetic_datagen}). On pricing, where Poisson demand makes that target noisy, full information is about $20\%$ above semi-bandit feedback, although still below pure bandit feedback. On the energy benchmark the variance contraction is especially large, with pure-bandit error bars roughly an order of magnitude wider than under semi-bandit or full feedback.

Table~\ref{tab:semi-bandit-full} reports final cumulative regret under semi-bandit feedback for the contextual-bandit baselines and \textsc{DFHPG} variants. The CB baselines use the semi-bandit-adapted regression update described at the end of Appendix~\ref{app:baseline_algorithms}; their action-selection rules are unchanged. \textsc{DFHPG-0} is best on top-$k$, shortest path, and energy. On pricing, the semi-bandit-adapted \textsc{TSCB} and \textsc{GreedyCB} are lowest, about $13\%$ below \textsc{DFHPG-0}. Under semi-bandit feedback, each round reveals the revenue of every product at its chosen price, so the regression these baselines run has a direct target for $20$ of the $80$ cost entries per round, and solving with the fitted costs already prices well. \textsc{DFHPG-0} learns from the same observations only through a separate nuisance model and a surrogate update of the policy, an indirect path designed for scarce, aggregate feedback.

\begin{table}[!ht]
\centering
\scriptsize
\caption{Final cumulative regret under semi-bandit feedback. The contextual-bandit baselines use the semi-bandit-adapted regression update of Appendix~\ref{app:baseline_algorithms}. Setup matches Table~\ref{tab:main-methods-selected-settings}; the best result in each column is bolded.}
\label{tab:semi-bandit-full}
\resizebox{\textwidth}{!}{%
\begin{tabular}{lcccc}
\toprule
Method & Top-$k$ (deg. 8) & Shortest path (deg. 8) & Pricing (deg. 6) & Energy \\
\midrule
\textsc{GreedyCB}
& $3.01{\times}10^5 \pm 1.61{\times}10^4$
& $9.37{\times}10^5 \pm 2.32{\times}10^4$
& $4.49{\times}10^4 \pm 1.5{\times}10^2$
& $2.66{\times}10^7 \pm 1.07{\times}10^5$ \\
$\epsilon$-\textsc{GreedyCB}
& $2.91{\times}10^5 \pm 2.01{\times}10^4$
& $9.87{\times}10^5 \pm 3.29{\times}10^4$
& $4.85{\times}10^4 \pm 1.2{\times}10^2$
& $4.16{\times}10^7 \pm 4.83{\times}10^5$ \\
\textsc{TSCB}
& $2.87{\times}10^5 \pm 1.86{\times}10^4$
& $9.67{\times}10^5 \pm 3.31{\times}10^4$
& $\boldsymbol{4.49{\times}10^4 \pm 1.5{\times}10^2}$
& $2.66{\times}10^7 \pm 1.08{\times}10^5$ \\
\midrule
\textsc{DFHPG}
& $7.10{\times}10^4 \pm 8.52{\times}10^3$
& $6.83{\times}10^5 \pm 3.11{\times}10^4$
& $5.19{\times}10^4 \pm 1.0{\times}10^2$
& $2.35{\times}10^7 \pm 4.21{\times}10^5$ \\
\textsc{DFHPG-0}
& $\boldsymbol{6.65{\times}10^4 \pm 5.30{\times}10^3}$
& $\boldsymbol{5.31{\times}10^5 \pm 2.28{\times}10^4}$
& $5.18{\times}10^4 \pm 9.8{\times}10^1$
& $\boldsymbol{2.33{\times}10^7 \pm 3.78{\times}10^5}$ \\
\bottomrule
\end{tabular}%
}
\end{table}

\begin{figure}[!ht]
\centering
\begin{subfigure}[t]{0.49\textwidth}
\centering
\includegraphics[width=\textwidth]{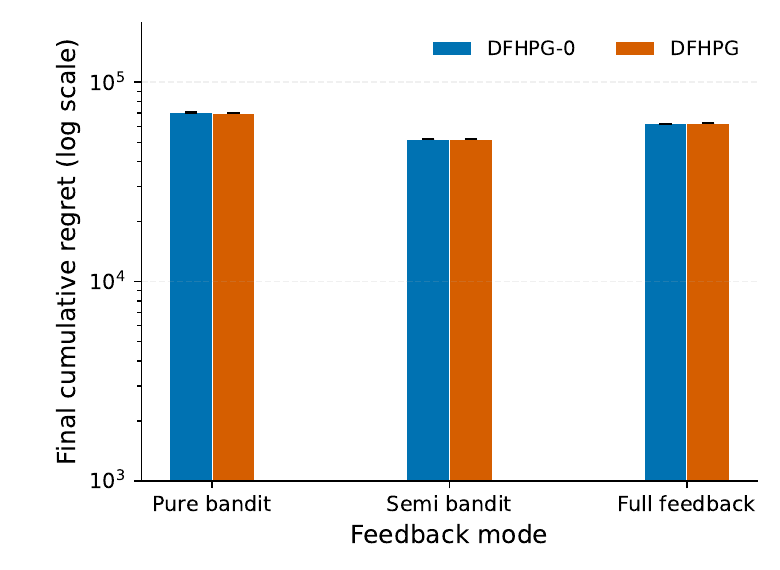}
\caption{Pricing}
\end{subfigure}\hfill
\begin{subfigure}[t]{0.49\textwidth}
\centering
\includegraphics[width=\textwidth]{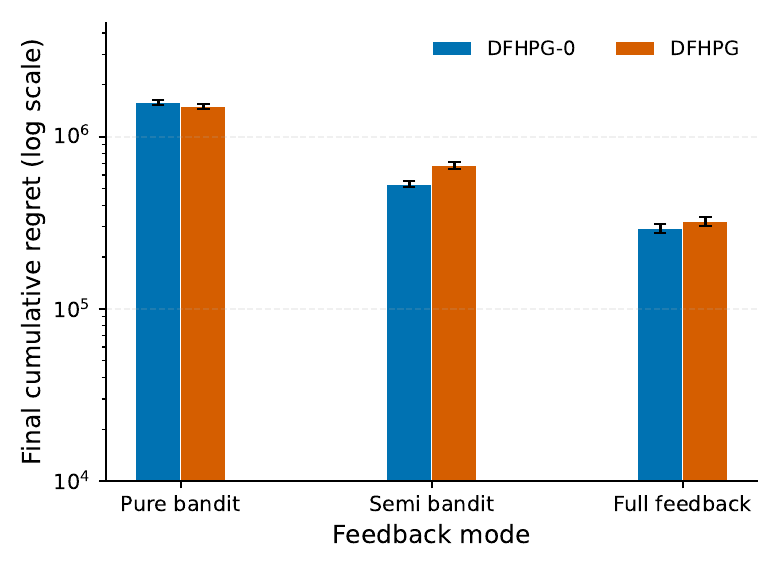}
\caption{Shortest path}
\end{subfigure}

\vspace{0.75em}
\begin{subfigure}[t]{0.49\textwidth}
\centering
\includegraphics[width=\textwidth]{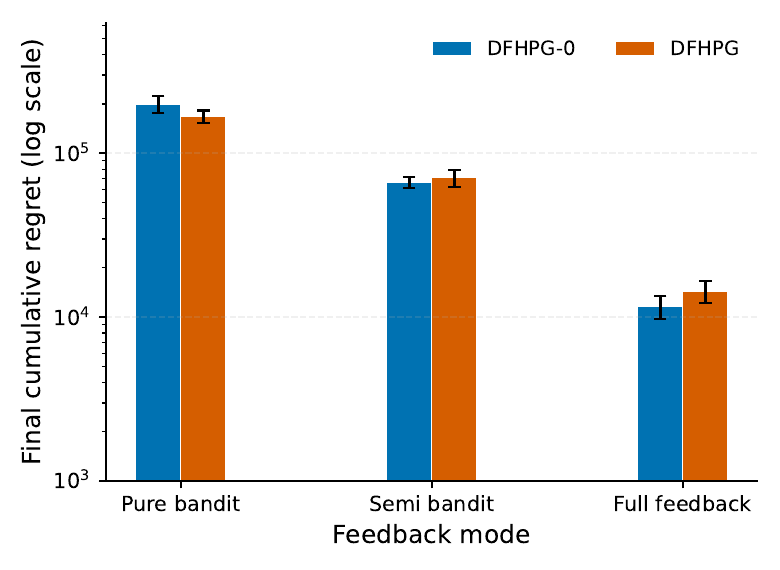}
\caption{Top-$k$ selection}
\end{subfigure}\hfill
\begin{subfigure}[t]{0.49\textwidth}
\centering
\includegraphics[width=\textwidth]{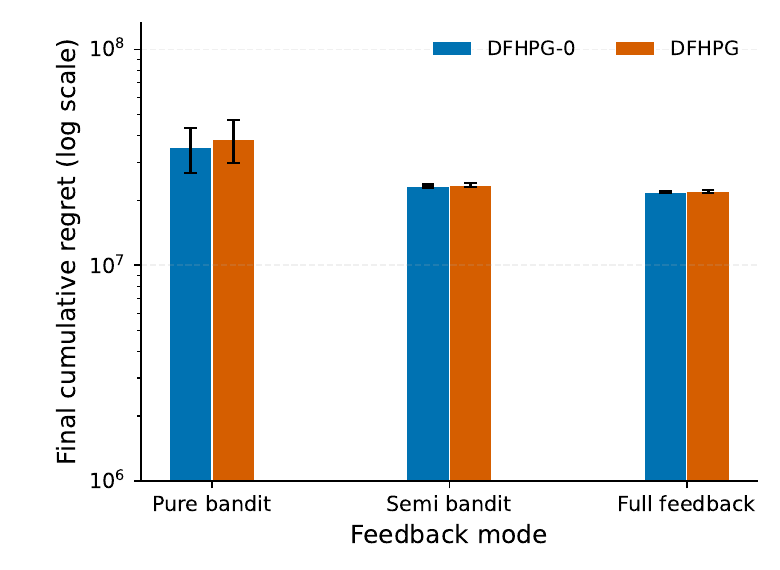}
\caption{Energy scheduling}
\end{subfigure}
\caption{Feedback mode comparison across the four benchmarks. Each plot compares final cumulative regret across feedback modes (bandit, semi-bandit, full information) on a log-scaled regret axis. Synthetic panels use the Gaussian linear cost-model experiments. Lower is better.}
\label{fig:app-feedback-mode}
\end{figure}

\FloatBarrier

\end{document}